\documentclass[twoside,11pt]{article}

\usepackage{blindtext}

\usepackage{jmlr2e}

\usepackage{mathtools,thmtools}                
\usepackage{bm}                       
\usepackage{hyperref}       
\hypersetup{
    colorlinks=true,
    linkcolor=blue,
    citecolor=blue,
    urlcolor=blue
}
\usepackage{xcolor}         
\usepackage{comment}        

\usepackage{graphicx}
\usepackage{caption}
\usepackage{subcaption}
\usepackage{booktabs}       
\usepackage{float}          

\newcommand{\M}{\mathcal{M}}
\newcommand{\R}{\mathbb{R}}

\newcommand{\var}{\text{Var}}

\usepackage{enumitem}         

\usepackage{lastpage}
\jmlrheading{23}{2022}{1-\pageref{LastPage}}{10/26; Revised -}{-}{21-0000}{Haoyu Chen and Chenghui Li}

\ShortHeadings{Minimax Rates and Improved Spectral Convergence for MMC}{Haoyu Chen and Chenghui Li}
\firstpageno{1}

\begin{document}

\title{Minimax Rates and Improved Spectral Convergence for Multi-Manifold Clustering}

\author{\name Haoyu Chen\thanks{Equal contribution.} 
        \email Haoyu.Chen@utah.edu \\
       \addr Department of Mathematics\\
       University of Utah\\
       Salt Lake City, UT 84112, USA
        \AND
       \name Chenghui Li\footnotemark[1]
       \email cli539@wisc.edu \\
       \addr Department of Statistics\\
       University of Wisconsin Madison\\
       Madison, WI 53706, USA
       }

\editor{My editor}

\maketitle

\begin{abstract}
We study eigenpair estimation for the tensorized, density-weighted Laplace--Beltrami operator from $n$ independent data points sampled from a multi-manifold mixture distribution where manifolds may intersect. Let \(d\) denote the largest intrinsic dimension of the manifold components and \(s\geq2\) the smoothness of the sampling densities. For simple, nontrivial eigenpairs, we establish the minimax rate \(n^{-s/(d+2s)}\) for approximating the eigenvalue and eigenfunction of the tensorized Laplace--Beltrami operator, and show that a continuum plug-in estimator based on kernel density estimation attains the same rate when the supporting manifolds are known. We further analyze a manifold-agnostic graph Laplacian estimator and prove that its eigenpairs converge to those of the tensorized Laplace--Beltrami operator at rate \(n^{-1/(2.5d-1)}\), up to logarithmic factors, improving upon previously known spectral convergence rates for multi-manifold learning. Numerical experiments on synthetic multi-manifold data corroborate the rates guaranteed by our theorems. These results characterize the information-theoretic difficulty of multi-manifold spectral estimation and quantify the additional cost of learning without oracle component information.
\end{abstract}


\section{Introduction}

Multi-manifold clustering (MMC) is the unsupervised learning problem of partitioning observations generated from a union of low-dimensional manifolds according to their latent manifold membership. A central challenge is that the manifold components need not be convex or well separated in the ambient space: individual components may exhibit nonlinear intrinsic geometries, and distinct manifolds may approach or intersect one another. This problem captures heterogeneous geometric structure that arises naturally in motion segmentation and optical flow \citep{ costeira1998, vidal2011, elhamifar2013}, face and image analysis \citep{ho2003, wright2009, liu2013}, dynamic pattern recognition \citep{tomasi1992, vidal2005}, and document classification \citep{belkin2001, ji2017}. Identifying the manifold components is often a prerequisite for downstream tasks such as dimension reduction, geometric representation, classification, and componentwise inference.

Most algorithms for MMC, such as that of \citet{Goh2007}, assume that manifold components do not intersect. In contrast, methods that address the setting considered here \citep{Wang2011, ariascastro2013, wang2015, babaeian2018, Li2022, chen2025} exploit local tangent information, constrained paths, or higher-order simplicial geometry to suppress spurious cross-manifold connections while preserving connectivity within each component. Nevertheless, designing computationally practical algorithms that provably remove spurious cross-manifold connections without distorting the within-manifold geometry therefore remains challenging.

Quantitative finite-sample guarantees are scarce among the existing methods tailored to intersecting manifolds. Despite their substantially different constructions, \citet{ariascastro2013} and \citet{chen2025} identified the same critical local scale of order $(\log n/n)^{1/d}$ for successfully discriminating intersecting manifold components with high probability under small noise perturbations; here \(d\) denotes the largest intrinsic dimension among the manifold components. By contrast, the probability bound in Theorem~3.1 of \citet{wang2015} leads to the slower-shrinking neighborhood scale $(\log n/n)^{1/(d+2)}$. The curvature-constrained path method of \citet{babaeian2018} provides a qualitative asymptotic guarantee for selecting landmarks away from manifold intersections, but no explicit sample-size rate for the final clustering error. \citet{bungert2026convergence} establishes $\Gamma$-convergence result for discrete normalized Dirichlet energy without providing an explicit convergence rate for spectral convergence, in which work the normalized graph Laplacian is considered. The most directly comparable result is \citet{Li2022}, which establishes spectral convergence of an angle-constrained graph Laplacian toward the tensorized Laplace--Beltrami operator, with eigenvalue and eigenfunction errors decaying at rates $n^{-1/(3d-1)}$ and $n^{-1/(6d-2)}$, respectively, up to logarithmic factors. Nevertheless, the information-theoretic limits of multi-manifold spectral estimation remain uncharacterized, and the existing spectral convergence rates remain substantial room for improvement. 

To formalize the problem, we assume that the data are sampled from a mixture distribution supported on $\mathcal{M}:=\cup_{j=1}^N\mathcal{M}_j$, where $\M_j$ is a $d_j$-dimensional smooth, compact, and connected Riemannian manifold without boundary embedded in $\R^D$; the distinct components \(\M_j\) may intersect. We set \(d:=\max_{j\in[N]} d_j\ll D\). The goal of MMC is to recover the latent assignment of each sample to its generating manifold component.

Following \citet{Li2022,bungert2026convergence}, we take the tensorized Laplace-Beltrami operator $\Delta_\M$, which acts independently on individual manifold components, as a natural spectral target for MMC. Since each $\M_j$ is connected, the nullspace of \(\Delta_{\mathcal M}\) is spanned by the component indicator functions and thus encodes the latent membership information of $\M_j$. We formally define this operator and describe its spectral properties in Section~\ref{sec:notation}. In this work, we characterize the convergence rates for the eigenvalues and eigenfunctions of $\Delta_\M$.

Our first contribution establishes, to our knowledge, the first minimax lower bound for spectral estimation in MMC; see Theorem~\ref{thm:spectral convergence lower bound}. For a simple, nontrivial eigenpair of $\Delta_{\M}$, this lower bound scales as $n^{-s/{(d+2s)}}$ where $s$ denotes the smoothness of the sampling densities. In particular, when $s=2$, this rate coincides with the minimax rate established by \citet{trillos2025} for a single $d$-dimensional flat torus. In a related setting, \citet{chaubet2025} consider a different partial differential operator on a given single manifold and establish the minimax rates for eigenvalue and eigenfunction estimation in the $L^2$ norm. 

Our second contribution establishes that this lower bound is sharp when the manifold is known; see Theorem~\ref{thm:spectral convergence (unrealistic) upper bound}. In the oracle setting, a continuum plug-in estimator based on kernel density estimation (KDE) attains the matching rate. Knowledge of the manifold components enables componentwise estimation even when the manifolds intersect. However, since the manifold components are unknown in the original MMC problem, the result serves as a statistical benchmark rather than a guarantee for a practically implementable estimator.

Our third contribution sharpens the spectral convergence guarantees of \citet{Li2022} for the annular proximity graph with angle constraints; see Theorem~\ref{thm:spectral convergence upper bound} and see Section \ref{ssec:annular_proximity_graph} for a description of the graph construction.
This practical graph-based estimator requires knowledge of the manifolds, and its construction preserves \textit{within-manifold connectivity} and suppresses \textit{spurious cross-manifold connections}---the two sufficient conditions for spectral convergence. Our key refinement is a direct residual comparison: using self-adjointness, we test the graph eigenvalue equation against the target population eigenfunction and decompose the resulting operator error into within- and between-manifold contributions. The analysis accommodates intersecting components of possibly different intrinsic dimensions, with the rate governed by the largest dimension $d$. For $d\geq2$, the eigenpair errors both decay at the rate $n^{-1/(2.5d-1)}$, up to logarithmic factors. Although slower than the oracle rate, this guarantee holds without oracle component information. 

Our fourth contribution is numerical: we corroborate the theoretical predictions in Section~\ref{sec:experiments}. We examine the minimax lower bound through the oracle KDE estimator and the practical upper bound through the graph-based estimator. Samples are drawn from intersecting flat tori with nonuniform sampling densities in dimensions \(d\in\{1,2,3\}\), and errors are reported across multiple eigenpairs. The observed oracle convergence rates closely match the minimax rates predicted by the lower bound, while the graph-based experiments corroborate the practical upper bound and suggest room for further improvement. Together, our results establish a sharp oracle benchmark for MMC estimation, characterize the intrinsic statistical difficulty of multi-manifold spectral estimation, and quantify the gap between this benchmark and what the state-of-the-art MMC algorithm achieves.

\subsection{Notation and Assumptions}\label{sec:notation}
For $N\in\mathbb N$, write $[N]:=\{1,\dots, N\}$. Throughout, \(\ell\) indexes eigenpairs and \(j\in[N]\) indexes the manifold components. Set $d_j:=\operatorname{dim}(\M_j)$ and $d:=\max_{j\in[N]}d_j$. While the components may intersect in the ambient space, functions and differential operators on $\M$ are understood componentwise, or equivalently on the abstract disjoint union $\sqcup_{j=1}^N\M_j$. Thus, a function \(f_\ell\) on $\M$ is identified with a tuple $f_\ell=(f_{\ell,1},\dots, f_{\ell,N})$, where $f_{\ell,j}:\M_j\rightarrow\mathbb R$ denotes its $j$-th component. The following geometric assumption is imposed throughout. 
\begin{assumption}[Manifold geometry] \label{assump:geometric assumptions} There exist uniform positive constants such that every component $\mathcal{M}_j$ has injectivity radius bounded below, sectional curvature bounded in absolute value, and Riemannian volume bounded both above and below. We denote by $\mathbf{M}$ the collection of all multi-manifold configurations satisfying these bounds.\end{assumption}
If the manifold components ever intersect, we assume the intersections are nontrivial. 
\begin{assumption}[Component intersections]\label{assump:component intersections}
For every $j\neq k$, the intersection $\M_{jk}:=\M_j\cap\M_k$ is either empty or a smooth submanifold of dimension $d_{jk}<\min\{d_j, d_k\}$. For $x\in\M_{jk}$, let $T_x\M_{jk}^{\perp,j}$ and $T_x\M_{jk}^{\perp,k}$ denote the orthogonal complements of $T_x\M_{jk}$ in $T_x\M_j$ and $T_x\M_k$. There exists a fixed $\beta<\pi/2$ such that for every $\M_{jk}$ and every $x\in\M_{jk}$
\[\sup_{\substack{v\in T_x\M_{jk}^{\perp,j}, \tilde v\in T_x\M_{jk}^{\perp,k}}, \|v\|_g=\|\tilde v\|_g=1}\left|\angle(v,\tilde v)-\frac{\pi}{2}\right|\leq\beta.\]
\end{assumption}

For each component, let $d\mathrm{vol}_g$ denote the associated Riemannian volume measure, and define the probability measure $d\mu_j:=\rho_j\,d\mathrm{vol}_g$. The global sampling measure is the mixture
\[\mu:=\sum_{j=1}^Nw_j\mu_j, \qquad w_j>0, \qquad \sum_{j=1}^Nw_j=1.\] We write $w_{\min}:=\min_{j}w_j$. The following regularity condition on $\rho_j$ is imposed throughout. 
\begin{assumption}[Density class]\label{assump:density class} For an integer $s\geq2$, let $\mathcal P_{\M_j}^{(s)}$ be the collection of strictly positive probability densities \(\rho_j\in C^s(\mathcal M_j)\) satisfying  \[0<\rho_{\min} \le \rho_j(x) \le \rho_{\max}, \qquad \int_{\mathcal{M}_j}\rho_j d\mathrm{vol}_g=1, \qquad \|\nabla_g^k\rho_j\|_{L^\infty(\mathcal M_j)}\leq C_{\rho}^{(k)},\]
for $k=0,\dots,s$. The corresponding multi-manifold density class is $\mathcal P^{(s)}_{\mathcal M}:=\prod_{j=1}^N \mathcal P_{\M_j}^{(s)}$. 
\end{assumption}
Let $\mathcal X_n:=\{X_1, \dots, X_n\}$ be an independent sample from $\mu$, and we write 
\[\mathcal X_{n,j}:=\mathcal X_n\cap\mathcal M_j, \qquad \widehat n_j:=|\mathcal X_{n,j}|, \qquad n_j:=nw_j.\]
Unless otherwise specified, probability $\mathbb P$ and expectation $\mathbb E$ are taken with respect to the joint distribution of $\mathcal X_n$. 

For $x,y\in\M_j$, let $d_g(x,y)$ denote their geodesic distance and let $\|x-y\|$ denote their ambient Euclidean distance. We write $T_x\M_j$ for the tangent space at $x$, and $\langle\cdot,\cdot\rangle_g$ and $\|\cdot\|_g$ for the Riemannian inner product and norm. The Euclidean and geodesic balls of radius $r$ centered at $x$ are denoted by $B(x,r)$ and $B_g(x,r)$, respectively. We use $\exp_x$ for the Riemannian exponential map and write $\log_x(y):=\exp_x^{-1}(y)$. The symbols $\nabla_g, \nabla_g^k$, and $\operatorname{div}_g$ denote the Riemannian gradient, the $k$-th covariant derivative, and the Riemannian divergence, respectively. 

For componentwise functions \(f_\ell, h_{\ell}\) with components \(f_{\ell,j}, h_{\ell,j}:\mathcal M_j\to\mathbb R\), define the standard Riemannian \(L^2\) and \(H^1\) norm by 
\[\|f_{\ell,j}\|_{L^2(\mathcal M_j)}^2:=\langle f_{\ell,j}, f_{\ell,j} \rangle_{L^2(\mathcal M_j)}, \qquad
\|f_{\ell,j}\|_{H^1(\M_j)}^2:=\|f_{\ell,j}\|_{L^2(\M_j)}^2+\|\nabla_g f_{\ell,j}\|_{L^2(\M_j)}^2\]
with \(\langle f_{\ell,j}, h_{\ell,j}\rangle_{L^2(\mathcal M_j)}:=\int_{\mathcal M_j}f_{\ell,j}h_{\ell,j}\,d\operatorname{vol}_g\). The unweighted global counterparts are defined by
\[\|f_{\ell,j}\|_{L^2(\mathcal M)}^2:=\sum_{j=1}^N\|f_{\ell,j}\|_{L^2(\mathcal M_j)}^2, \qquad \|f_\ell\|_{H^1(\mathcal M)}^2:=\sum_{j=1}^N\|f_{\ell,j}\|_{H^1(\mathcal M_j)}^2.\]
The density-weighted componentwise and global \(L^2\)-inner products are defined by \[\langle f_{\ell,j}, h_{\ell,j}\rangle_{L^2(\mu_j)}=\int_{\M_j}f_{\ell,j}h_{\ell,j}d\mu_j, \qquad \langle f_\ell, h_\ell\rangle_{L^2(\mu)}:=\sum_{j=1}^N w_j\langle f_{\ell,j}, h_{\ell,j}\rangle_{L^2(\mu_j)}.\] The corresponding norms are defined analogously. 
Unless a measure is displayed explicitly, norms default to the Riemannian volume measure $d\mathrm{vol}_g$. 

For functions $u,v:\mathcal X_{n,j}\rightarrow\mathbb R$, define the empirical inner product and norm by
\[\langle u, v\rangle_{L^2(\mathcal X_{n,j})}:=\frac{1}{\widehat n_j}\sum_{x\in\mathcal X_{n,j}}u(x)v(x), \qquad \|u\|_{L^2(\mathcal X_{n,j})}^2:=\langle u, u\rangle_{L^2(\mathcal X_{n,j})}.\]
The global empirical quantities on $\mathcal X_{n}$ are defined analogously, normalized by $n$ instead.

For each component, define the density-weighted Laplace--Beltrami operator
\begin{equation}\label{eq:laplace-beltrami dfn}\Delta_{\mathcal M_j} f_{\ell,j}\;:=\;-\frac{1}{\rho_j} \operatorname{div}_g\! \bigl(\rho_j^2 \nabla_g f_{\ell,j}\bigr),\qquad f_{\ell,j}:\mathcal M_j\to\mathbb R.\end{equation}
The population operator is the weighted direct sum $\Delta_\M:=\bigoplus w_j\Delta_{\mathcal M_j}$, or equivalently, 
\[\Delta_{\mathcal{M}}f_\ell:=(w_1\Delta_{\mathcal{M}_1}f_{\ell,1}, w_2\Delta_{\mathcal{M}_2}f_{\ell,2},\dots, w_N\Delta_{\mathcal{M}_N}f_{\ell,N}).\] 
Throughout, $(\lambda_\ell, f_\ell)$ denotes the $\ell$-th population eigenpair of $\Delta_\M$, so an equivalence between the global eigen-equation and the system of componentwise equations is characterized by 
\[\Delta_{\M}f_\ell=\lambda_\ell f_\ell, \qquad \Longleftrightarrow \qquad w_j\Delta_{\M_j}f_{\ell,j}
=\lambda_\ell f_{\ell, j},\quad j\in[N].\]
For each \(j\in[N]\), set
\(\lambda_{\ell,j}:=\lambda_\ell/w_j\) so that whenever
\(f_{\ell,j}\not\equiv0\), the pair \((\lambda_{\ell,j},f_{\ell,j})\) is a component eigenpair of \(\Delta_{\M_j}\). For the eigenvalues, the spectrum of $\Delta_\M$ is the multiset union of the weighted component spectra, counting multiplicity and arranged in nondecreasing order. Since each $\M_j$ is connected, each $\Delta_{\M_j}$ has a one-dimensional nullspace of constant functions associated with eigenvalue zero. Consequently
\[0=\lambda_1=\cdots=\lambda_N<\lambda_{N+1}\leq\lambda_{N+2}\leq\cdots.\]
For the eigenfunctions, we normalize \(f_\ell\) according to
\[\|f_\ell\|_{L^2(\mu)}^2=\sum_{j=1}^N w_j
\|f_{\ell,j}\|_{L^2(\mu_j)}^2=1.\]
Elliptic regularity for the componentwise uniformly elliptic equations on compact manifolds guarantees that $f_{\ell,j}\in C^2(\M_j)$ for all \(j\in[N]\). The energy identity
\begin{equation}\label{eq:energy-identity}
    \lambda_{\ell,j}\|f_{\ell,j}\|_{L^2(\mu_j)}^2=
\left\langle f_{\ell,j},\Delta_{\M_j}f_{\ell,j}
\right\rangle_{L^2(\mu_j)}=\int_{\M_j}
\rho_j^2\|\nabla_g f_{\ell,j}\|_g^2\,d\operatorname{vol}_g\end{equation}
together with the density bounds, yields
\begin{equation}\label{eq:first-order-elliptic-estimate}
\|\nabla_g f_{\ell,j}\|_{L^2(\mathcal M_j)}
\lesssim\lambda_{\ell,j}^{1/2}
\|f_{\ell,j}\|_{L^2(\mathcal M_j)}\lesssim
w_j^{-1/2}\lambda_{\ell,j}^{1/2},
\end{equation} while the corresponding second-order elliptic estimate gives
\begin{equation}\label{eq:second-order-elliptic-estimate}
\|\Delta_g f_{\ell,j}\|_{L^2(\mathcal M_j)}+\|\nabla_g^2 f_{\ell,j}\|_{L^2(\mathcal M_j)}\lesssim(1+\lambda_{\ell,j})
\|f_{\ell,j}\|_{L^2(\mathcal M_j)}\lesssim
w_j^{-1/2}(1+\lambda_{\ell,j}).
\end{equation}
Here, the implicit constants depend only on the uniform geometric and density-class parameters. Since \(w_{\min}>0\), these weight factors are uniformly bounded and may be absorbed into the implicit constants. By the norm equivalence stated above, the same estimates hold with the corresponding weighted \(L^2(\mu_j)\) norms. Finally, since the components are compact and \(N<\infty\), the preceding \(C^2\)-regularity implies
\begin{equation}\label{eq:L-infty-regularity}
C_f^{(\nu)}:=\max_{j\in[N]}\|\nabla_g^\nu f_{\ell,j}\|_{L^\infty(\mathcal M_j)}<\infty, \qquad \text{for } \nu=0,1,2.
\end{equation}
For quantitative \(L^\infty\) estimates for Laplace--Beltrami eigenfunctions, see \cite{sogge2001}.



Define the index set associated with
\(\lambda_\ell\) and the corresponding population eigenspace by 
\[[\ell]_{\lambda}:=\{k:\lambda_k=\lambda_\ell\},\qquad E_\ell:=\ker(\Delta_{\mathcal M}-\lambda_\ell I)
=\operatorname{span}\{f_k:k\in\mathcal I_\ell\}.\]
The eigengap separating \(E_\ell\) from the remainder of the spectrum is defined by
\begin{equation}\label{eq:eigengap}
\gamma_\ell:=\operatorname{dist}\!\left(\lambda_\ell,\operatorname{spec}(\Delta_{\mathcal M})\setminus\{\lambda_\ell\}
\right)=\min_{k\notin [\ell]_{\lambda}}|\lambda_k-\lambda_\ell|.\end{equation}
Simple eigenfunctions are compared modulo sign, whereas for a repeated eigenvalue, comparison with an estimated eigenfunction is understood through the optimal alignment
\[\operatorname{dist}(\widehat f_\ell,E_\ell):=\inf_{f\in E_\ell,\,\|f\|_{L^2(\mu)}=1}\|\widehat f_\ell-f\|.\]
Here and below, we use $(\widehat\lambda_\ell, \widehat f_\ell)$ to denote the empirical estimate of \((\lambda_\ell, f_\ell)\).

For a symmetric weighted graph on $\mathcal X_n$, let $W=(W_{ik})_{i,k=1}^n$ denote its weight matrix and let $D$ be the diagonal matrix with $D_{ii}:=\sum_{k=1}^n W_{ik}$. We write \(\mathcal L:=D-W\) for its graph Laplacian. The graph Laplacian operator $\mathcal L_{n,\varepsilon}:L^2(\mathcal X_n) \to L^2(\mathcal X_n)$ can be defined as 
\[
\mathcal L_{n,\varepsilon}f(x) := \frac{2}{C_d n \varepsilon^{d+2}} \sum_{y\in\mathcal X_n} W_{x\sim y}(f(x)-f(y)).
\]
In the practical upper bound, $W$ is constructed from the annular proximity graph with angle constraints; its precise definition is given in Appendix~\ref{ssec:annular_proximity_graph}.

For the graph construction, let $\eta:[0,\infty)\rightarrow[0,\infty)$ be a bounded, non-increasing, Lipschitz kernel profile supported on $[0,1]$, and define $\eta_{h}(t):=\eta\left(t/h\right)$ for bandwidth $h>0$. For an intrinsic dimension $d_j$, define 
\begin{equation}\label{eq:sigma eta}
\sigma_{\eta, d_j}:=\int_{\mathbb R^{d_j}}z_1^2\eta(\|z\|)dz=\frac{1}{d_j}\int_{\mathbb{R}^{d_j}}\|z\|^2\eta\left(\|z\|\right)dz. \end{equation}
For an eigenpair supported on $\M_j$, we abbreviate $\sigma_{\eta,d_j}$ as $\sigma_\eta$ when the dimension is clear. In this work, we consider $\eta(t)=1_{t\leq 1}.$

For two nonnegative quantities $a$ and $b$, we write $a\lesssim b$ if $a\leq Cb$ for a constant $C>0$ independent of the sample size and all asymptotically vanishing parameters. We write $a\gtrsim b$ if $b\lesssim a$, and $a\asymp b$ if both inequalities hold. The symbols $C,c>0$ denote generic constants whose values may change from line to line and may depend only on the fixed geometric and density-class parameters, the kernel profile, and the eigenvalue index under consideration.

\subsection{Main Results}
We first study the minimax estimation of a nontrivial eigenpair of $\Delta_\M$. The analysis consists of an information-theoretic lower bound and a matching oracle upper bound. To establish the lower bound over $\mathcal P_{\mathcal M}^{(s)}$ for an integer $s\geq2$, we restrict attention to a \emph{simple, nontrivial} global eigenpair
\((\lambda_\ell,f_\ell)\) with \(\ell\geq N+1\). Simplicity of \(\lambda_\ell\) ensures a strictly positive eigengap \(\gamma_\ell>0\). The direct-sum structure implies that there exists a unique \(j\in[N]\) such that 
\[f_{\ell, j}\not\equiv0, \qquad \text{and } \quad f_{\ell,k}=0 
\quad\text{for every }k\neq j.\]
In this case, we say \(f_\ell\) is componentwise supported on \(\M_j\). Our first theorem gives a lower bound over the full model class $\mathbf M$ and applies to every measurable estimator based on $\mathcal X_n$.


\begin{restatable}[Spectral convergence lower bound]{theorem}{first}\label{thm:spectral convergence lower bound}
Suppose Assumptions~\ref{assump:geometric assumptions} and \ref{assump:density class} hold with $\rho\in\mathcal P_{\mathcal M}^{(s)}$ for an integer $s\geq2$. For $\ell\geq N+1$, let $(\lambda_\ell,f_\ell)$ be a simple, nontrivial eigenpair of $\Delta_{\M}$, componentwise supported on $\M_j$. Then there exists $c_{\ell}>0$, depending on the geometric and density-class parameters, the index $\ell$, and the eigengap $\gamma_{\ell}$, such that 
\begin{equation}\label{eq:spectral-lower-bound}
\inf_{(\widehat\lambda_\ell,\widehat f_\ell)}\sup_{\mathcal M\in\mathbf M,\,\rho\in\mathcal P_{\mathcal M}^{(s)}}\mathbb E\left[|\lambda_\ell-\widehat\lambda_\ell|+\|f_\ell-\widehat f_\ell\|_{H^1(\mathcal M)}\right]\ge c_{\ell} \max_j \sqrt{\lambda_\ell\,w_j}\,n_j^{-s/(d_j+2s)},
\end{equation} where the infimum on the left is taken over all measurable estimators based on $\mathcal X_n$ and the maximum on the right is taken over all admissible supporting components. \end{restatable}

The \(H^1\)-based eigenfunction loss arises naturally from our lower-bound construction and is consistent with the single-manifold minimax formulation. Replacing it with the $L^2$-loss gives a slightly weaker eigenfunction-separation rate.

\begin{restatable}[\(L^2\) spectral lower bound]{corollary}{second}\label{cor:pure L2 spectral lower bound} Under the assumptions of Theorem~\ref{thm:spectral convergence lower bound}, one has \[\inf_{(\widehat \lambda_\ell,\widehat f_\ell)}\sup_{\mathcal M\in\mathbf M,\,\rho\in\mathcal P_{\mathcal M}^{(s)}}\mathbb E\left[|\lambda_\ell-\widehat\lambda_\ell|+\|f_\ell-\widehat f_\ell\|_{L^2(\mathcal M)}\right]\ge c_{\ell} \max_j\sqrt{\lambda_\ell\,w_j}\,n_j^{-(s+1)/(d_j+2s)}.\]  \end{restatable}

The lower-bound rate is attained by an oracle plug-in eigenpair estimator based on kernel density estimation (KDE), which is briefly reviewed in Section~\ref{sec:KDE on manifolds}. 

\begin{restatable}[Oracle plug-in upper bound]{theorem}{third}\label{thm:spectral convergence (unrealistic) upper bound} 
Suppose Assumptions~\ref{assump:geometric assumptions} and \ref{assump:density class} hold with $\rho\in\mathcal P_{\mathcal M}^{(s)}$ for an integer $s\geq2$.  
Let $(\lambda_\ell,f_\ell)$ be a simple, nontrivial global eigenpair of $\Delta_{\M}$ componentwise supported on $\M_j$. Assume that $\mathcal M_j$ is known and that the samples in $\mathcal X_{n,j}$ can be identified. Let $\widehat\rho_j$ be the KDE from the $\widehat n_j$ observed samples on $\mathcal M_j$, with bandwidth \[h_j\asymp ({n}_j\vee1)^{-1/(2s+d_j)}.\]  Let $(\widehat\lambda_\ell,\widehat f_\ell)$ be the plug-in eigenpair obtained by replacing $\rho_j$ with $\widehat\rho_j$ in the continuum eigenvalue problem. Then there exists $C_{\ell}>0$, depending only on the geometric and density-class parameters, the index $\ell$, and the eigengap $\gamma_\ell$, such that 
\[\mathbb E\left[|\lambda_\ell-\widehat\lambda_\ell| +\|f_\ell-\widehat f_\ell\|_{H^1(\mathcal M)}\right]\le C_{\ell}\max_{j}n_j^{-s/(d_j+2s)}.\]
\end{restatable} 

We next turn to the practical setting in which neither the component manifolds nor the latent component labels are known. The main difficulty arises near manifold intersections, where ambient proximity alone may create spurious edges between points sampled from different
components. To this end, we apply the annular proximity graph with angle constraints (Section~\ref{ssec:annular_proximity_graph}) for graph construction. We establish a sharper spectral convergence rate for the resulting graph Laplacian.

Unlike the population operator $\Delta_\M$, the finite-sample graph Laplacian need not have $N$ exact zero eigenvalues when between-manifold connections are present. Accordingly, rather than restricting to the nontrivial range $\ell\geq N+1$ considered in the lower-bound analysis, we state the practical upper-bound result for any fixed spectral index $\ell\geq1$. We impose neither a simplicity assumption on $\lambda_\ell$ nor the requirement that $f_\ell$ be componentwise supported on a single manifold. We only assume, in addition to Assumption~\ref{assump:density class}, that \(\rho\in C^{2,\alpha}\) for some Hölder exponent \(\alpha>0\) so that
\begin{equation}\label{eq:B_definition}
B_\ell:=\max_{j\in[N]}\left(\sum_{\nu=0}^3\|\nabla_g^\nu f_\ell\|_{L^\infty(\M_j)}^2\right)^{1/2}<\infty.
\end{equation}
\begin{restatable}[Graph-Laplacian spectral convergence]{theorem}{forth}\label{thm:spectral convergence upper bound} Suppose Assumptions~\ref{assump:geometric assumptions} and \ref{assump:density class} hold with $\rho\in C^{2,\alpha}$, and recall that $d:=\max_{j\in[N]}d_j\geq2$. Fix $\ell\geq1$ and let $(\widehat \lambda_{\ell}, \widehat f_{\ell})$ be the corresponding normalized eigenpair of the graph Laplacian $\mathcal L$ based on the sample $\mathcal X_n$. Then there exists a normalized population eigenpair $(\lambda_{\ell}, f_{\ell})$ such that, provided $B_\ell<\infty$, \[|\widehat\lambda_{\ell}-\sigma_{\eta}\lambda_{\ell}|\leq C_{\ell}B_\ell \Big(\frac{\log n}{n}\Big)^{\frac{1}{2.5d-1}}, \qquad \|\widehat f_{\ell}-f_{\ell}\|_{H^1(\mathcal X_n)}\leq C_{\ell} B_\ell {\Big(\frac{\log n}{n}\Big)}^{-\frac{1}{2.5d-1}}\log(n),\] with probability at least $1-Cn^{-2}$. Here the constant $C_{\ell}>0$ depends only on the geometry class \(\mathbf{M}\), the intersection parameters, and the index $\ell$. \end{restatable} \begin{proof} See Section~\ref{sec:spectral upper bound}.\end{proof}

The remainder of the paper is organized as follows. Section~\ref{sec:minimax lower bound} establishes the minimax lower bound via Fano’s method and concludes with the oracle plug-in KDE upper bound. Section~\ref{sec:spectral upper bound} analyzes the graph Laplacian by controlling its within- and between-manifold contributions. Section~\ref{sec:experiments} numerically tests both the minimax lower bound and the practical upper bound on synthetic data sampled from two intersecting flat tori with prescribed smooth densities.
Section~\ref{sec:conclusion} concludes with directions for future work.

\section{Minimax Lower Bound and an Oracle Upper Bound}
\label{sec:minimax lower bound}

This section establishes the minimax lower bound in Theorem~\ref{thm:spectral convergence lower bound} and the matching oracle upper bound in Theorem~\ref{thm:spectral convergence (unrealistic) upper bound}. We apply Fano's method to construct the lower bound where a brief review of Fano's method is given in Appendix~\ref{sec:fano's method}; see \cite{Wainwright2019} for a comprehensive treatment of minimax lower bounds. The oracle upper bound is achieved by the KDE estimator, which is presented in Section~\ref{sec:KDE on manifolds}. 

For each $\rho\in\mathcal P_{\mathcal M}^{(s)}$ with \(s\geq2\), our parameter of interest is the selected eigenpair 
\[\theta(\rho):=(\lambda_{\ell}(\rho), f_{\ell}(\rho))\in\mathbb{R}\times H^1(\mathcal{M}), \qquad \ell \ge N+1.\]  
Given $\mathcal X_n\sim\mu_\rho^{\otimes n}$, we consider an arbitrary measurable estimator
\[\widehat\theta(\mathcal{X}_n):=(\widehat{\lambda}_{\ell}(\mathcal{X}_n), \widehat f_{\ell}(\mathcal{X}_n))\in \mathbb{R}\times H^1(\mathcal{M}).\] We measure the estimation error using \[d_{\mathrm{eig}}\bigl(\theta(\rho),\widehat\theta(\mathcal X_n)\bigr):=|\lambda_\ell - \widehat\lambda_\ell|
+ \|f_\ell - \widehat f_\ell\|_{H^1(\mathcal M)},
\qquad \ell \ge N+1.\]
Dependence on $\rho$ and $\mathcal X_n$ is suppressed when clear.
Since the target eigenfunction is componentwise supported on $\M_j$,
setting the remaining components of any estimated eigenfunction to
zero cannot increase the global $H^1$-loss. It therefore suffices to
consider estimators $\widehat f_\ell$ satisfying
$\widehat f_{\ell,k}\equiv0$ for every $k\neq j$, in which case
\[d_{\mathrm{eig}}(\theta(\rho), \widehat \theta(\mathcal{X}_n)) = |\lambda_{\ell}-\widehat \lambda_{\ell}| + \|f_\ell - \widehat f_\ell\|_{H^1(\mathcal{M}_{j})}.\]
For any subclass $\mathcal Q\subseteq\mathcal P_{\mathcal M}^{(s)}$, define the minimax risk
\[\mathfrak M_n\bigl(\theta(\mathcal Q); d_{\mathrm{eig}}\bigr)
:=\inf_{\widehat\theta}\sup_{\rho\in\mathcal Q}
\mathbb E_{\mathcal X_n\sim\mu_\rho^{\otimes n}}
\bigl[d_{\mathrm{eig}}\bigl(\theta(\rho),\widehat\theta(\mathcal X_n)\bigr)\bigr],\] 
where the infimum is taken over all measurable estimators. This quantity is the smallest worst-case eigenpair estimation error achievable from $n$ independent observations. Fano's method lower bounds this risk by reducing estimation to a finite multi-hypothesis testing problem. Specifically, we construct, via perturbations of the density on $\M_j$, a finite family
\[\mathcal P_M:=\{\rho^{(1)},\dots,\rho^{(M)}\} \subset \mathcal P_{\mathcal M}^{(s)}\] while keeping the densities on all remaining components fixed. The construction is designed to satisfy the two conditions \eqref{eq:fano_sep} and \eqref{eq:fano_kl} required by Fano's method. Fano's inequality then gives a minimax lower bound. As $\mathcal P_M\subset \mathcal P_{\mathcal M}^{(s)}$, the resulting lower bound also applies to the full density class. The perturbation amplitudes are additionally chosen small enough that the target eigenpair remains simple and satisfies the uniform eigengap condition \eqref{eq:eigengap}
throughout the constructed family. The remainder of the lower-bound proof constructs these perturbations and verifies the required spectral-separation and KL-divergence bounds. The construction of the distribution family is presented in Section~\ref{sec:distribution construction}, and verification of Fano's conditions is carried out in Section~\ref{sec:fano-verification}. The proof of Theorem~\ref{thm:spectral convergence lower bound}, Corollary~\ref{cor:pure L2 spectral lower bound}, and a matching oracle upper bound specified in Theorem~\ref{thm:spectral convergence (unrealistic) upper bound} are completed in Section~\ref{sec:conclude minimax lower bound}.

\subsection{Distribution Construction}
\label{sec:distribution construction}

Throughout, $m \in \mathbb{N}$ is a sufficiently large scale parameter and $r:= m^{-1}$. 
We fix a reference density tuple $\rho\in\mathcal P_{\mathcal M}^{(s)}$ satisfying the defining bounds with uniform slack throughout. 

\subsubsection{Local Geometric Decomposition}
\label{sec:local geometric decomposition}
For $m$ sufficiently large, $r=m^{-1}$ lies below the injectivity radius of $\M_j$ and geodesic balls of radius $r$ are contained in normal coordinate neighborhoods. Choose a maximal $r$-separated set of points $\{p_i\}_{i=1}^K \subset \mathcal{M}_{j}$ so that \[d_g(p_i,p_{i'})\ge r \quad \text{for all } i\neq i', \qquad\text{and}\qquad \mathcal M_j=
\bigcup_{i=1}^K B_g(p_i,r).\]
Uniform volume comparison and packing estimates give \(K \asymp r^{-d_j} = m^{d_j}\),
where the implicit constants depend only on the geometric bounds of $\mathcal M_j$. Assigning equidistant points to one of the minimizing centers, define the associated Voronoi cells by
\[\Omega_i:=\bigl\{x\in\mathcal M_j:\ d_g(x,p_i)\le d_g(x,p_{i'})\ \text{for all } i'\neq i\bigr\}.\]
The cells $\{\Omega_i\}_{i=1}^K$ form a measurable partition of $\mathcal M_j$ and satisfy, uniformly in $i$,
\[B_g(p_i,r/2)\subset\Omega_i\subset B_g(p_i,r), \qquad \mathrm{vol}_g(\Omega_i)\asymp r^{d_j}, \qquad  \mathrm{diam}_g(\Omega_i)\lesssim r.\]
Perturbation templates constructed below have pairwise disjoint supports within these cells. 

For sufficiently large $m$, the exponential map defines a normal-coordinate chart
\begin{equation}\label{eq:exponential chart}
\exp_{p_i}:B(0,r)\subset T_{p_i}\M_j\simeq\mathbb R^{d_j}
\longrightarrow B_g(p_i,r).
\end{equation}
All local constructions are performed in these charts, allowing Euclidean constructions in the coordinate domain to be transferred to $\mathcal M_j$ with uniformly controlled geometric distortion.

\subsubsection{Construction of Perturbation Templates}\label{sec:construction of the perturbation templates}
We first construct a mean-zero Euclidean profile whose directional derivative has opposite signs on two separated regions. Fix $e\in\mathbb S^{d_j-1}$ and write $u=te+v$, where \(t\in\mathbb R\) and \(v\in e^{\perp}\). Choose an even function $q\in C_c^\infty\left((-a,a)\right)$ such that 
\[\int_{\mathbb R}q(t)dt=0, \qquad \text{and} \qquad q'(t_0)>0\]
for some $t_0\in(0,a)$. Such a function exists: take any nonzero even function \(\varphi\in C_c^\infty((-a,a))\), set \(q=\varphi''\), and, if
necessary, replace \(q\) by \(-q\). Let $\zeta\in C_c^\infty(e^\perp)$ be nonnegative and strictly positive near the origin, with the supports of $q$ and $\zeta$ chosen sufficiently small that
\[\phi(te+v):=q(t)\zeta(v)\in C_c^\infty\left(B(0,1/4)\right).\]
By Fubini's theorem, \(\phi\) has mean zero.
Since $q$ is even, $q'$ is odd and so \(q'(-t_0)<0\). The identity \(\partial_e\phi(te+v)=q'(t)\zeta(v)\), along with continuity, yields constants $b_0,\tau_0>0$ such that, with $u_0:=t_0e$ and $B(\pm u_0,\tau_0)\subset B(0,1/4)$,
\begin{equation}\label{eq:directional-derivative-witness}
\partial_e\phi(u)\ge b_0\quad\text{on }B(u_0,\tau_0),\qquad\partial_e\phi(u)\le-b_0\quad\text{on }B(-u_0,\tau_0).\end{equation}
Thus, the gradient of $\phi$ has uniformly opposite components along $e$ on two separated regions. After lifting to $\M_j$, these regions provide witness sets for the spectral separation argument.


For each $i\in[K]$, let $R_i\in O(d_j)$ and define $\phi_i(u):=\phi(R_i u)$. On the spectrally active patches, the rotations $R_i$ will be chosen in Lemma~\ref{lem:spectral visibility and alignment} to align the distinguished direction of $\phi_i$ with the eigenfunction's gradient $\nabla_g f_\ell$. Define the raw template $\tilde a_i:\mathcal M_j\to\mathbb R$ by 
\[\tilde a_i(x):=\frac{1}{\rho_j(x)}\,\phi_i\bigl(m\,\exp_{p_i}^{-1}(x)\bigr)\cdot \mathbf{1}_{B_g(p_i,r/4)}(x).\]
Let $a_i$ denote the weighted mean-zero correction of $\tilde a_i$ constructed in Appendix~\ref{sec:construction of perturbation templates}. The correction is supported away from the witness regions, and the resulting templates satisfy\begin{equation}\label{eq:a_i properties}
\int_{\mathcal M_j}a_i\,d\mu_j=0, \qquad
\operatorname{supp}(a_i)
\subset B_g(p_i,r/2)\subset\Omega_i,\qquad \|\nabla_g^{\nu}a_i\|_{L^\infty(\mathcal M_j)}\lesssim m^{\nu}, \quad \nu=0,\dots,s. \end{equation} In particular, the templates \(a_i\) have pairwise disjoint supports. The mean-zero property preserves the normalization of the perturbed densities, while the derivative bounds ensure that multiplication by $m^{-s}$ preserves the prescribed $s$-th order regularity. Moreover, 
\begin{equation}\label{eq:a_i norm}
\|a_i\|_{L^2(\M_j)}^2\le\|a_i\|_{L^\infty(\M_j)}^2\operatorname{vol}_g(\operatorname{supp}(a_i))\lesssim \operatorname{vol}_g(\Omega_i)\lesssim m^{-d_j}.\end{equation}
Note that the Euclidean profile $\phi$ and the estimates in \eqref{eq:a_i properties} are independent of $f_\ell$. Only the rotations $R_i$ on the active patches are selected using the gradient direction of the fixed reference eigenfunction. Lemma~\ref{lem:spectral visibility and alignment} shows that this selection yields a positive fraction of patches on which the perturbations interact with $f_\ell$ in a sign-definite manner. 

\subsubsection{The Perturbed Density Family}
\label{sec:the perturbed density family}

For each configuration
$\mathbf c=(c_1,\dots,c_K)\in\mathcal C:= \{-1, 1\}^K$, 
define the component densities
\begin{equation}\label{eq:perturbed densities}
\rho_j^{(\mathbf c)}(x):=\rho_j(x)\left( 1+m^{-s}\sum_{i=1}^K c_i a_i(x)
\right),\qquad\rho_k^{(\mathbf c)}:=\rho_k \quad(k\neq j).\end{equation}
Write $\rho^{(\mathbf c)}:=\left(\rho_1^{(\mathbf c)},\dots,\rho_N^{(\mathbf c)}\right)$ for the resulting density tuple, and let
\[\mu^{(\mathbf c)}:=\sum_{k=1}^Nw_k\mu_k^{(\mathbf c)}, \qquad d\mu_k^{(\mathbf c)}:=\rho_k^{(\mathbf c)}d\mathrm{vol}_g,\]
be its associated mixture distribution. Recall that we fix the mixture weights throughout.

One can verify that $\rho^{(\mathbf c)}$ remains valid in the sense that 
\(\{\rho^{(\mathbf c)}:\mathbf c\in\mathcal C\}\subset\mathcal P_{\mathcal M}^{(s)}\).
Indeed, the mean-zero condition in \eqref{eq:a_i properties} gives
$\int_{\M_j}\rho_j^{(\mathbf c)}d\mathrm{vol}_g=1$. Since the templates have pairwise disjoint supports, at most one $a_i$ contributes at any point. Thus for sufficiently large $m$, $\|\nabla_g^\nu (\rho_j^{(\mathbf c)} - \rho_j)\|_{L^\infty(\M_j)} \lesssim m^{\nu-s}$ for each $\nu=0,\dots,s$.
The perturbed densities are therefore strictly positive and normalized, and satisfy $\rho^{(\mathbf c)}\in\mathcal P_{\mathcal M}^{(s)}$. 
The scaling $m^{-s}$ simultaneously controls the $s$-th order regularity and keeps the corresponding sampling measures close in KL divergence, while the oscillations at scale $m^{-1}$ generate the spectral separation below. After identifying the spectrally active patches, Section~\ref{sec:fano-verification} selects a Hamming-separated codebook $\mathcal T\subset\mathcal C$. Fano's method is then applied to the finite family $\mathcal P_M
=\{\rho^{(\mathbf c)}:\mathbf c\in\mathcal T\}$.

\subsection{Verification of Fano Conditions}
\label{sec:fano-verification}

We verify the two conditions \eqref{eq:fano_sep} and \eqref{eq:fano_kl} for the perturbation family \eqref{eq:perturbed densities}. We first establish pairwise separation of the induced
eigenpairs under $d_{\mathrm{eig}}$. The argument proceeds by identifying a positive fraction of patches on which the reference eigenfunction has nonvanishing gradient,
aligning the local perturbations with that gradient, and then converting the
resulting local forcing into a global eigenpair separation.

\subsubsection{Separation in the eigenpair metric}\label{sec:2alpha separation} 
\begin{lemma}[Spectral visibility and alignment on active patches]
\label{lem:spectral visibility and alignment}
Let $(\lambda_\ell,f_\ell)$ be a simple, nontrivial global eigenpair componentwise supported
on $\M_j$. There exist constants
$m_0\in\mathbb N$ and $c_0,c_1,\epsilon_0>0$, independent of $m$, and an index
set $\mathcal I\subset[K]$ with
\begin{equation}\label{eq:active-patch-cardinality}|\mathcal I|\ge\epsilon_0m^{d_j},\end{equation}
such that, for every $m\ge m_0$ and $i\in\mathcal I$, there are measurable
witness sets $\mathcal N_i^\pm\subset\Omega_i$ satisfying
\begin{equation}\label{eq:raw-template-final-alignment}
\operatorname{vol}_g(\mathcal N_i^\pm)\ge c_0m^{-d_j},
\qquad\pm\bigl\langle\nabla_g(\rho_j^2a_i),\nabla_gf_{\ell,j}\bigr\rangle_g \ge c_1m\sqrt{\lambda_{\ell,j}} \quad\text{on }\mathcal N_i^\pm.
\end{equation}\end{lemma}
\begin{proof} 
We first identify a positive-volume region on which the reference eigenfunction has a quantitatively nonvanishing gradient. Given that the eigenpair is simple and nontrivial, the energy identity \eqref{eq:energy-identity} implies
\(\lambda_{\ell,j}>0.\)
Fix $\kappa>0$ and define the high-gradient set with associated active index set
\[S:=\left\{x\in\mathcal M_j:
\rho_j(x)\|\nabla_g f_{\ell,j}(x)\|_g
\ge\kappa\sqrt{\lambda_{\ell,j}}\right\}, \qquad \mathcal I:=
\bigl\{i\in[K]:\Omega_i\cap S\neq\varnothing\bigr\}.\]
Choose $\kappa$ small enough that $\kappa^2\operatorname{vol}_g(\M_j)\le 1/2$. 
The uniform $C^1$ bound 
\eqref{eq:L-infty-regularity} on \(f_{\ell,j}\) implies
\begin{equation}\label{eq:high-gradient-volume}
\operatorname{vol}_g(S)\ge\frac{\lambda_{\ell,j}}
{2\rho_{\max}^2\|\nabla_gf_{\ell,j}\|_{L^\infty(\M_j)}^2}>0.\end{equation}
Since the cells form a partition and
$S\subset\bigcup_{i\in\mathcal I}\Omega_i$, the uniform cell-volume bound gives
\[0<\operatorname{vol}_g(S)
\le \sum_{i\in\mathcal I}\operatorname{vol}_g(\Omega_i)
\le C_\Omega|\mathcal I|m^{-d_j}.\]
Rearranging terms gives \eqref{eq:active-patch-cardinality}. Thus, a positive fraction of the $K\asymp m^{d_j}$ patches intersect the region on which the eigenfunction gradient is nontrivial. Indeed, this eigenfunction gradient lower bound propagates across the entire patch. To see this, fix $i\in\mathcal I$ and choose a reference point $x_i\in\Omega_i\cap S$ where the eigenfunction gradient satisfies the lower bound. Set
\[G(x):=\rho_j(x)\|\nabla_gf_{\ell,j}(x)\|_g, \qquad x\in\Omega_i.\]
The $C^1$ bound on $\rho_j$ and the $C^2$ bound on $f_{\ell,j}$ \eqref{eq:L-infty-regularity} imply that $G$ is uniformly Lipschitz. Along a minimizing unit-speed geodesic, the fundamental theorem of calculus gives
\[\bigl|\|\nabla_gf_{\ell,j}(x)\|_g-\|\nabla_gf_{\ell,j}(y)\|_g
\bigr|\le\|\nabla_g^2f_{\ell,j}\|_{L^\infty(\M_j)}d_g(x,y),\]
and therefore
\[|G(x)-G(y)|\le L d_g(x,y),\qquad L:=
\rho_{\max}\|\nabla_g^2f_{\ell,j}\|_{L^\infty(\M_j)}
+\|\nabla_g\rho_j\|_{L^\infty(\M_j)}
 \|\nabla_gf_{\ell,j}\|_{L^\infty(\M_j)},\]
where $L$ is independent of $m$ and of the patch index. Since
$\Omega_i\subset B_g(p_i,r)$ and $x_i\in\Omega_i$,
\[d_g(x,x_i)\le d_g(x,p_i)+d_g(p_i,x_i)\le 2r=\frac{2}{m},
\qquad x\in\Omega_i.\]
Consequently,
\[G(x)\ge G(x_i)-\frac{2L}{m}
\ge \kappa\sqrt{\lambda_{\ell,j}}-\frac{2L}{m}.\]
After enlarging $m_0$ so that
$\frac{2L}{m_0}\le\frac{\kappa}{2}\sqrt{\lambda_{\ell,j}}$ for all \(m\geq m_0\), we obtain
\begin{equation}\label{eq:patchwise-gradient-lower-bound}
G(x)=\rho_j(x)\|\nabla_gf_{\ell,j}(x)\|_g
\ge\frac{\kappa}{2}\sqrt{\lambda_{\ell,j}},\qquad x\in\Omega_i.
\end{equation}
In particular, the eigenfunction at the patch center \(p_i\) has nontrivial gradient. We may therefore define the unit gradient direction at $p_i$ by
\[e_i(p_i):=\frac{\nabla_gf_{\ell,j}(p_i)}
{\|\nabla_gf_{\ell,j}(p_i)\|_g}\in T_{p_i}\M_j.\]

We rotate the Euclidean profile \(\phi\) so that its distinguished direction \(e\) agrees with $e_i(p_i)$, ensuring that the eigenfunction gradient interacts productively with the perturbation \(a_i\). Identify $T_{p_i}\M_j$ isometrically with $\mathbb R^{d_j}$ and choose $R_i\in O(d_j)$ such that $R_ie_i(p_i)=e$. Set
\[\phi_i(u):=\phi(R_iu),\qquad u_0^{(i)}:=R_i^\top u_0.\]
For notational simplicity, we relabel \(u_0\) for \(u_0^{(i)}\). By the chain rule,
\[\partial_{e_i(p_i)}\phi_i(u)=\bigl\langle\nabla\phi(R_iu),R_ie_i(p_i)\bigr\rangle=\partial_e\phi(R_iu).\]
Therefore, the directional-witness property of $\phi$ from Equation \eqref{eq:exponential chart} becomes
\[\partial_{e_i(p_i)}\phi_i(u)\ge b_0\quad\text{on }B(u_0,\tau_0),
\qquad\partial_{e_i(p_i)}\phi_i(u)\le-b_0\quad\text{on }B(-u_0,\tau_0).\]
To transfer these two regions to the manifold at scale $m^{-1}$, define
\[U_i^\pm:=B\!\left(\pm\frac{u_0}m,\frac{\tau_0}m\right)\subset B(0,r/4),
\qquad \mathcal N_i^\pm:=\exp_{p_i}(U_i^\pm)\subset B_g(p_i,r/4)\subset\Omega_i.\]
Also, the normal-coordinate volume density is uniformly bounded below on $B(0,r)$. Thus
\begin{equation}\label{eq:witness-volume}
\operatorname{vol}_g(\mathcal N_i^\pm)= \int_{U_i^\pm}\sqrt{\det g_i(u)}\,du\ge c\,|U_i^\pm|\ge c_0m^{-d_j}.\end{equation}
Consequently, every active patch indexed by \(i\in\mathcal I\) contains two witness regions \(\mathcal N_i^{\pm}\) of the natural patch-scale volume on which the rotated profile has uniformly opposite directional derivatives. It remains to transfer this directional property to $\nabla_g(\rho_j^2a_i)$.


Fix $x\in\mathcal N_i^\pm$, write
$u:=\exp_{p_i}^{-1}(x)\in U_i^\pm$. Then $mu\in B(\pm u_0,\tau_0)$, so
\begin{equation}\label{eq:phi-derivative-sign}
F_i(x):=\phi_i\bigl(m\exp_{p_i}^{-1}(x)\bigr) \qquad \text{satisfies } \qquad  \pm\,\partial_{e_i(p_i)}\phi_i(mu)\ge b_0.
\end{equation}
For each $x\in\Omega_i$, let \(\gamma:[0,1]\to\M_j\) be the unique constant-speed minimizing geodesic from \(p_i\) to \(x\). Parallel transport $e_i(p_i)$ along $\gamma$ and denote its endpoint value by $e_i(x)$. A standard Jacobi-field estimate in Proposition~\ref{eq:jacobi-field-comparison} along the radial geodesic from \(p_i\) to \(x\) gives
\[\left\|d(\exp_{p_i})_u\left(e_i(p_i)\right)-e_i(x)\right\|_g\leq C|u|^2 \implies \left\|\left(d(\exp_{p_i})_u\right)^{-1}e_i(x)-e_i(p_i)\right\|_g\leq C|u|^2.\]
Here we use the fact that \(d(\exp_{p_i})_u\) is uniformly invertible for \(|u|\leq r/4\). By the chain rule,
\[\left\langle\nabla_gF_i(x),e_i(x)\right\rangle_g =m\left\langle\nabla\phi_i(mu),\left(d(\exp_{p_i})_u\right)^{-1}e_i(x)\right\rangle.\]
This estimate, along with the smoothness of $\phi_i$, and $|u|\le r/4\lesssim m^{-1}$, gives
\[\left|\left\langle\nabla_gF_i(x),e_i(x)\right\rangle_g
-m\,\partial_{e_i(p_i)}\phi_i(mu)\right|\le Cm|u|^2\le Cm^{-1}.\]
Combining with \eqref{eq:phi-derivative-sign} and enlarging $m_0$ if necessary, we obtain
\begin{equation}\label{eq:profile-directional-lower-bound}
\pm\left\langle\nabla_gF_i(x),e_i(x)\right\rangle_g
\ge \frac{b_0}{2}m, \qquad x\in\mathcal N_i^\pm.\end{equation}
The same chain-rule argument and the uniform boundedness of
$\left(d(\exp_{p_i})_u\right)^{-1}$ give
\begin{equation}\label{eq:profile-gradient-upper-bound}
\|\nabla_gF_i(x)\|_g\le C_\phi m,
\qquad x\in\mathcal N_i^\pm.\end{equation}
Combining \eqref{eq:profile-directional-lower-bound} and
\eqref{eq:profile-gradient-upper-bound}, and controlling the discrepancy between the true normalized eigenfunction gradient and its parallel-transported approximation $e_i(x)$, we obtain 
\begin{align*}\pm\left\langle\nabla_gF_i(x),\frac{\nabla_gf_{\ell,j}(x)}{\|\nabla_gf_{\ell,j}(x)\|_g}\right\rangle_g&=\pm\left \langle \nabla_gF_i(x),e_i(x)\right\rangle_g\pm\left\langle
\nabla_gF_i(x),\frac{\nabla_gf_{\ell,j}(x)}{\|\nabla_gf_{\ell,j}(x)\|_g}-e_i(x)\right\rangle_g\\&\ge \frac{b_0}{2}m-C_\phi\,m\,\left\|\frac{\nabla_gf_{\ell,j}(x)}{\|\nabla_gf_{\ell,j}(x)\|_g}-e_i(x)\right\|_g. \end{align*}
Indeed, the parallel-transported approximation \(e_i(x)\) satisfies
\[\begin{aligned}\bigl\|\nabla_gf_{\ell,j}(x)-\|\nabla_gf_{\ell,j}(p_i)\|_g e_i(x)\bigr\|_g&\leq\int_0^1\bigl\|\nabla_{\dot\gamma(t)} \nabla_gf_{\ell,j}\bigr\|_g\,dt\leq \|\nabla_g^2f_{\ell,j}\|_{L^\infty(\M_j)}\int_0^1\|\dot\gamma(t)\|_g\,dt\\&=d_g(p_i,x)\|\nabla_g^2f_{\ell,j}\|_{L^\infty(\M_j)}\leq C_{\ell,j}\,m^{-1}.\end{aligned}\]
Applying the inequality
\[\left\|\frac{a}{\|a\|}-\frac{b}{\|b\|}\right\|\le
\frac{2\|a-b\|}{\min\{\|a\|,\|b\|\}},\]
in \(T_x\M_j\) with \(a=\nabla_gf_{\ell,j}(x)\) and \(b=\|\nabla_gf_{\ell,j}(p_i)\|_g\, e_i(x)\), together with \eqref{eq:patchwise-gradient-lower-bound}, yields
\[\left\|\frac{\nabla_gf_{\ell,j}(x)}{\|\nabla_gf_{\ell,j}(x)\|_g}
-e_i(x)\right\|_g\le C_{\ell,j} m^{-1}\qquad x\in\Omega_i.\]
Consequently, enlarging $m_0$ once more, we obtain, for all \(x\in\mathcal N_i^\pm\),
\[\pm\left\langle\nabla_gF_i(x),\frac{\nabla_gf_{\ell,j}(x)}{\|\nabla_gf_{\ell,j}(x)\|_g}\right\rangle_g\geq\frac{b_0}{2}m-C_\phi C_{\ell,j}\ge \frac{b_0}{4}m.\]
By construction, \(\rho_j^2a_i=\rho_jF_i\), and so
\(\nabla_g(\rho_j^2a_i)= \rho_j\nabla_gF_i+F_i\nabla_g\rho_j\).
Taking the inner product with the normalized eigenfunction gradient and applying the uniform bounds on \(\phi\) and \(\nabla_g\rho_j\) yields, for large enough $m$, there exists \(c_a>0\) independent of \(m\) and \(i\) such that
\[\pm\left\langle\nabla_g(\rho_j^2a_i)(x),\frac{\nabla_gf_{\ell,j}(x)}{\|\nabla_gf_{\ell,j}(x)\|_g}\right\rangle_g\ge\rho_{\min} \frac{b_0}{4}m-\|\phi_i\|_{L^\infty}\|\nabla_g\rho_j \|_{L^\infty(\M_j)}\ge c_a m \]
for all \(x\in\mathcal N_i^\pm\). Multiplying by $\|\nabla_gf_{\ell,j}(x)\|_g$ and applying \eqref{eq:patchwise-gradient-lower-bound} completes the proof.
\end{proof}

Lemma~\ref{lem:spectral visibility and alignment} identifies a  collection of active patches, indexed by $\mathcal I$, on which the local perturbations produce a sign-definite spectral effect. We next compare the eigenpairs induced by any pair of distinct configurations on their differing active set \(\mathcal D:=\{i\in \mathcal I: c_i^{(1)}\neq c_i^{(2)}\}.\) By the Varshamov--Gilbert bound \cite[Lemma~2.9]{Tsybakov2009}, there exists a subset of sign configurations \(\mathcal T\subset\{-1, 1\}^{|\mathcal I|}\) such that \(\log|\mathcal T|\geq c_T|\mathcal I|\gtrsim m^{d_j}\) and for distinct \(\mathbf c^{(1)}, \mathbf c^{(2)} \in \mathcal T\)
\[|\mathcal D|=\sum_{i\in\mathcal I}\mathbf 1_{\{ c_i^{(1)}\neq c_i^{(2)}\}}\geq c_{\mathrm H}|\mathcal I| \gtrsim m^{d_j},\]
for some constants $c_{\mathrm T}, c_{\mathrm H} > 0$ independent of \(m\). 
For $b=1,2$, write the perturbed density
\[\rho_j^{(b)}:=\rho_j^{(\mathbf c^{(b)})}=\rho_j\left(1+m^{-s}\sum_{i=1}^Kc_i^{(b)}a_i\right).\]
Since the supports of $a_i$ are pairwise disjoint, \eqref{eq:a_i norm} gives
\[\left\|\sum_{i=1}^Kc_i^{(b)}a_i\right\|_{L^2(\M_j)}^2=
\sum_{i=1}^K\|a_i\|_{L^2(\M_j)}^2
\lesssim K\,m^{-d_j}\lesssim1.\]
It follows that the perturbation size of densities satisfies
\begin{equation}
\label{eq:aggregate-density-deviation}
\|\rho_j^{(b)}-\rho_j\|_{L^2(\M_j)}
=m^{-s}\left\|\rho_j\sum_{i=1}^K
c_i^{(b)}a_i\right\|_{L^2(\M_j)}\lesssim m^{-s},\qquad b=1,2. \end{equation}
Define the corresponding perturbed measure and component operator by 
\[d\mu_j^{(b)}:=\rho_j^{(b)}\,d\mathrm{vol}_g,\qquad \Delta_j^{(b)}f_{\ell,j}:=-\frac{1}{\rho_j^{(b)}}\operatorname{div}_g
\left(\bigl(\rho_j^{(b)}\bigr)^2\nabla_gf_{\ell,j}
\right).\]
Consider the linear density path
\begin{equation}\label{eq:linear-density-path}
\rho_{j,t}^{(b)}:=\rho_j+t\bigl(\rho_j^{(b)}-\rho_j\bigr),\qquad 0\le t\le1. \end{equation} By the eigengap condition \eqref{eq:eigengap},
the perturbation bound \eqref{eq:aggregate-density-deviation}, and standard perturbation theory for simple isolated eigenvalues \cite{Kato1995}, the reference eigenpair admits a unique simple continuation along this path for all sufficiently large $m$. The
quantitative eigenvalue and eigenfunction estimates along this branch are given by Proposition~\ref{prop:eigenpair under perturbation}.

Let $(\lambda_{\ell,j}^{(b)},f_{\ell,j}^{(b)})$ denote the resulting component eigenpair at $t=1$, normalized by
\[\Delta_j^{(b)}f_{\ell,j}^{(b)}=\lambda_{\ell,j}^{(b)}f_{\ell,j}^{(b)},\qquad\|f_{\ell,j}^{(b)}\|_{L^2(\mu_j^{(b)})}=w_j^{-1},\qquad\lambda_\ell^{(b)}:=w_j\lambda_{\ell,j}^{(b)}.\]
The feasibility of the perturbed family implies the uniform measure equivalence, so
\begin{equation}\label{eq:perturbed-eigenfunction-L2-bound}
\|f_{\ell,j}^{(b)}\|_{L^2(\mu_j)}\le C\|f_{\ell,j}^{(b)}\|_{L^2(\mu_j^{(b)})}=C,
\qquad b=1,2.\end{equation}
We choose the signs of $f_{\ell,j}^{(1)}$ and $f_{\ell,j}^{(2)}$ according to the convention fixed throughout, so that both are aligned with the reference eigenfunction $f_{\ell,j}$. The next lemma derives the governing differential inequality for the difference $h:=f_{\ell,j}^{(1)}-f_{\ell,j}^{(2)}$, establishing a strict lower bound for the forcing term induced by the local density perturbations on the witness regions. 



\begin{lemma}\label{lem:two_codeword_difference_active_patch}There exist $c, C_{\ell,j}>0$ independent of $m$ such that for any distinct \(\mathbf c^{(1)}\), \(\mathbf c^{(2)}\in \mathcal T\) and each differing active patch $i\in\mathcal D$, there exists a witness region $\mathcal N_i^\star \in \{\mathcal N_i^+, \mathcal N_i^-\}$ on which the difference $h = f_{\ell,j}^{(1)} - f_{\ell,j}^{(2)}$ satisfies the pointwise inequality
\begin{equation}\label{eq:spectral separation prep}
\operatorname{div}_g(\rho_j^2\nabla_gh)(x)+
(\lambda_{\ell,j}^{(1)}-\lambda_{\ell,j}^{(2)})
\rho_j(x)f_{\ell,j}^{(1)}(x)+\lambda_{\ell,j}^{(2)}\rho_j(x)h(x)\geq c\sqrt{\lambda_{\ell,j}}\,m^{1-s}-|\mathcal R_m(x)|,
\end{equation} where the remainder satisfies \(\|\mathcal R_m\|_{L^2(\mathcal M_j)} \le C_{\ell,j}\lambda_{\ell,j} m^{-s}.\)
\end{lemma}
\begin{proof}
For brevity, for \(b=1,2\), set
\[
A_b:=\sum_{i=1}^Kc_i^{(b)}a_i, \qquad
A_\Delta:=\frac{A_1-A_2}{2} =\sum_{i\in\mathcal D}\frac{c_i^{(1)}-c_i^{(2)}}{2} a_i=:\sum_{i\in\mathcal D}\varsigma_i a_i.\] Note that \(\varsigma_i\in\{\pm1\}\) for each \(i\in\mathcal D\) and vanishes for \(i\not\in\mathcal D\). Moreover, \(A_\Delta=\varsigma_i a_i\) on \(\Omega_i\). Since the support of \(a_i\) are pairwise disjoint, we write $B:=A_1^2=A_2^2=\sum_i a_i^2$. Thus 
\[\rho_j^{(b)}=\rho_j(1+m^{-s}A_b),\qquad \bigl( \rho_j^{(b)}\bigr)^2=\rho_j^2\bigl(1+2m^{-s}A_b+m^{-2s}B\bigr).\]
For \(b=1,2\), the perturbed eigen-equation is \[ \operatorname{div}_g\!\bigl( (\rho_j^{(b)})^2 \nabla_g f_{\ell,j}^{(b)}\bigr)+\lambda_{\ell,j}^{(b)}\,\rho_j^{(b)} f_{\ell,j}^{(b)}=0.\] 
Subtracting the equation for $b=2$ from that for $b=1$ and separating the terms gives
\begin{align*}
\operatorname{div}_g(\rho_j^2\nabla_gh)
&+(\lambda_{\ell,j}^{(1)}-\lambda_{\ell,j}^{(2)})
\rho_jf_{\ell,j}^{(1)}+\lambda_{\ell,j}^{(2)}\rho_jh \notag\\
&=-2m^{-s}\operatorname{div}_g\left[\rho_j^2\left(
A_1\nabla_gf_{\ell,j}^{(1)}-A_2\nabla_gf_{\ell,j}^{(2)}
\right)\right]-m^{-2s} \operatorname{div}_g (\rho_j^2B\nabla_gh)\notag\\
&\quad-m^{-s}\rho_j
\left(\lambda_{\ell,j}^{(1)}A_1f_{\ell,j}^{(1)}
-\lambda_{\ell,j}^{(2)}A_2f_{\ell,j}^{(2)}\right).
\end{align*}
The first term on the right-hand side contains the principal forcing term. Using
\[\operatorname{div}_g(\rho_j^2A_b\nabla_gf_{\ell,j}^{(b)})=
\left\langle\nabla_g(\rho_j^2A_b),\nabla_gf_{\ell,j}^{(b)}\right\rangle_g+\rho_j^2A_b\Delta_g f_{\ell,j}^{(b)}\] and 
\begin{align*}&\left\langle\nabla_g(\rho_j^2A_1),\nabla_gf_{\ell,j}^{(1)}\right\rangle_g-\left\langle \nabla_g(\rho_j^2A_2), \nabla_gf_{\ell,j}^{(2)}\right\rangle_g\\&\;=2\left\langle
\nabla_g(\rho_j^2A_\Delta),\nabla_gf_{\ell,j}\right\rangle_g+
\left\langle\nabla_g(\rho_j^2A_1),\nabla_g(f_{\ell,j}^{(1)}-f_{\ell,j})\right\rangle_g-\left\langle\nabla_g(\rho_j^2A_2), \nabla_g(f_{\ell,j}^{(2)}-f_{\ell,j})\right\rangle_g,\end{align*}
we obtain 
\[\operatorname{div}_g(\rho_j^2\nabla_gh)+
(\lambda_{\ell,j}^{(1)}-\lambda_{\ell,j}^{(2)})
\rho_jf_{\ell,j}^{(1)}+\lambda_{\ell,j}^{(2)}\rho_jh \notag=
-4m^{-s}\left\langle\nabla_g\left(
\rho_j^2A_{\Delta}\right),
\nabla_gf_{\ell,j}\right\rangle_g+\mathcal R_m,\]
where \(\mathcal R_m=\mathcal R_m^{(1)}+\mathcal R_m^{(2)}
+\mathcal R_m^{(3)}+\mathcal R_m^{(4)}+\mathcal R_m^{(5)}\) with


\begin{align*}
\mathcal R_m^{(1)}
&:=-m^{-s}\rho_j\left(\lambda_{\ell,j}^{(1)}A_1f_{\ell,j}^{(1)}-\lambda_{\ell,j}^{(2)}A_2f_{\ell,j}^{(2)}\right),\qquad \mathcal R_m^{(2)}:=-2m^{-s}\rho_j^2\left(A_1\Delta_gf_{\ell,j}^{(1)}-A_2\Delta_gf_{\ell,j}^{(2)}\right),\\\mathcal R_m^{(3)}
&:=-2m^{-s}\left[\left\langle\nabla_g(\rho_j^2A_1),
\nabla_g(f_{\ell,j}^{(1)}-f_{\ell,j}) \right\rangle_g-\left\langle\nabla_g(\rho_j^2A_2),\nabla_g(f_{\ell,j}^{(2)}-f_{\ell,j})\right\rangle_g\right],\\\mathcal R_m^{(4)}&:=-m^{-2s}\left\langle
\nabla_g(\rho_j^2B),\nabla_gh
\right\rangle_g,\qquad \mathcal R_m^{(5)}:=-m^{-2s}\rho_j^2B\Delta_gh.\end{align*}
We derive uniform $L^1$ bounds for these terms. Disjoint supports of \(a_i\) and \eqref{eq:a_i properties} imply
\begin{equation}
\label{eq:aggregate-template-bounds-I}
\|A_b\|_{L^\infty(\M_j)}
+\|A_\Delta\|_{L^\infty(\M_j)}
+\|B\|_{L^\infty(\M_j)}
\le C
\end{equation}
and 
\begin{equation}
\label{eq:aggregate-template-bounds-II}
\|\nabla_g(\rho_j^2A_b)\|_{L^\infty(\M_j)}
+\|\nabla_g(\rho_j^2B)\|_{L^\infty(\M_j)}
\le Cm.
\end{equation}

Integrating the derivative bounds
\eqref{eq:lambda_dot_bound} and \eqref{eq:fdot_H1_bound} along the linear density paths~\eqref{eq:linear-density-path} gives
\begin{equation}\label{eq:finite-eigenpair-stability}
|\lambda_{\ell,j}^{(b)}-\lambda_{\ell,j}|
\le C_{\ell,j}\lambda_{\ell,j}m^{-s},
\qquad \|f_{\ell,j}^{(b)}-f_{\ell,j}\|_{H^1(\M_j)}
\le C_{\ell,j}\frac{\lambda_{\ell,j}}{\gamma_\ell}m^{-s},\qquad b=1,2.\end{equation}
In particular, $\lambda_{\ell,j}^{(b)}\lesssim\lambda_{\ell,j}$ and the uniform elliptic estimates from~\eqref{eq:second-order-elliptic-estimate} give
\begin{equation}\label{eq:perturbed-eigenfunction-regularity}
\|f_{\ell,j}^{(b)}\|_{L^2(\M_j)}
+\|\Delta_gf_{\ell,j}^{(b)}\|_{L^2(\M_j)}
\le C_{\ell,j},\qquad b=1,2.\end{equation}
Using \eqref{eq:aggregate-template-bounds-I} and
\eqref{eq:perturbed-eigenfunction-regularity}, Hölder's inequality gives
\begin{equation}\label{eq:R12-bound}
\|\mathcal R_m^{(1)}\|_{L^2(\M_j)}\le
C_{\ell,j}\lambda_{\ell,j}m^{-s}, 
\qquad \|\mathcal R_m^{(2)}\|_{L^2(\M_j)}\le C_{\ell,j}\lambda_{\ell,j}m^{-s}.\end{equation}
For the third term, the finite perturbation estimate
\eqref{eq:finite-eigenpair-stability} gives
\begin{align}\label{eq:R3-bound}
\|\mathcal R_m^{(3)}\|_{L^2(\M_j)}
&\le C_{\ell,j} m^{-s} m\sum_{b=1}^2
\|\nabla_g(f_{\ell,j}^{(b)}-f_{\ell,j})\|_{L^2(\M_j)}\le
C_{\ell,j}\lambda_{\ell,j}\gamma_\ell^{-1}
m^{1-2s}.\end{align} Similarly,
\begin{equation}\label{eq:R4-bound}
\|\nabla_gh\|_{L^2(\M_j)}\le
C_{\ell,j}\lambda_{\ell,j}\gamma_\ell^{-1}m^{-s} \implies \|\mathcal R_m^{(4)}\|_{L^2(\M_j)}\le
C_{\ell,j}\lambda_{\ell,j}\gamma_\ell^{-1}
m^{1-3s}.\end{equation}
Finally, by \eqref{eq:perturbed-eigenfunction-regularity},
\begin{equation}\label{eq:R5-bound}
\|\mathcal R_m^{(5)}\|_{L^2(\M_j)}\le
Cm^{-2s}\left(
\|\Delta_gf_{\ell,j}^{(1)}\|_{L^2(\M_j)}+
\|\Delta_gf_{\ell,j}^{(2)}\|_{L^2(\M_j)}\right)\le
C_\ell\lambda_{\ell,j}m^{-2s}.\end{equation}
Combining \eqref{eq:R12-bound}--\eqref{eq:R5-bound} and absorbing the fixed eigengap into $C_{\ell,j}$ gives the remainder's scale. To control the principle forcing term, we fix $i\in\mathcal D$. Since \(A_\Delta=\varsigma_i a_i\) on \(\Omega_i\), Lemma~\ref{lem:spectral visibility and alignment} allows one to choose \(\mathcal N_i^\star\in\{\mathcal N_i^+,\mathcal N_i^-\}\) so that 
\[-4m^{-s}\left\langle\nabla_g(\rho_j^2A_\Delta),\nabla_gf_{\ell,j}\right\rangle_g=-4m^{-s}\varsigma_i\left\langle\nabla_g(\rho_j^2a_i),\nabla_gf_{\ell,j}\right\rangle_g\geq c\sqrt{\lambda_{\ell,j}}\,m^{1-s}\qquad\text{on }\mathcal N_i^\star.\] This proves \eqref{eq:spectral separation prep}. Since \(|\mathcal D|\gtrsim m^{d_j}\), the selected witness regions satisfy
\(\operatorname{vol}_g\left(\bigcup_{i\in\mathcal D}\mathcal N_i^\star\right)\gtrsim|\mathcal D|m^{-d_j}\gtrsim1.\) Thus, the local principal forcing accumulates globally at scale \(m^{1-s}\), dominating the remainder \(\mathcal R_m= \mathcal O(m^{-s})\). We therefore complete the proof.

\end{proof}

\begin{lemma}[Sufficient spectral separation]
\label{lem:spectral separation}
Fix a differing active patch $i\in\mathcal D$. Then either
\begin{equation}\label{eq:eigenvalue separation}
|\lambda_\ell^{(1)}-\lambda_\ell^{(2)}|
\ge c\,\sqrt{\lambda_\ell w_j}\,m^{1-s}, \end{equation} or
\begin{equation}\label{eq:eigenfunc separation}
\frac{\lambda_\ell\rho_{\max}}{m}
\|f_\ell^{(1)}-f_\ell^{(2)}\|_{L^2(\M_j)}+
w_j\|\nabla_g(f_\ell^{(1)}-f_\ell^{(2)})\|_{L^2(\M_j)}
\ge c\,\sqrt{\lambda_\ell w_j}\,m^{-s}.
\end{equation}\end{lemma}
\begin{proof}
Let $c_1>0$. If $|\lambda_{\ell,j}^{(1)}-\lambda_{\ell,j}^{(2)}|\ge
c_1\sqrt{\lambda_{\ell,j}}\,m^{1-s}$,
then multiplication by $w_j$ gives the desired eigenvalue separation \eqref{eq:eigenvalue separation}. It therefore remains to consider \begin{equation}\label{eq:small-eigenvalue-gap-assumption}
|\lambda_{\ell,j}^{(1)}-\lambda_{\ell,j}^{(2)}|<c_1\,\,\sqrt{\lambda_{\ell,j}}\,m^{1-s}.\end{equation} 
For $i\in\mathcal D$, choose \(\mathcal N_i^\star
\in\{\mathcal N_i^+,\mathcal N_i^-\}\)
so that Equation~\eqref{eq:spectral separation prep} in  Lemma~\ref{lem:two_codeword_difference_active_patch} is satisfied. Define 
\[\mathfrak r:=|\mathcal R_m|+|\lambda_{\ell,j}^{(1)}-\lambda_{\ell,j}^{(2)}|\,\rho_j|f_{\ell,j}^{(1)}|.\]
Then Equation~\eqref{eq:spectral separation prep} implies
\begin{equation}
\label{eq:local-difference-inequality}
\operatorname{div}_g(\rho_j^2\nabla_gh)+\lambda_{\ell,j}^{(2)}\rho_jh+\mathfrak r\ge c_0\,\sqrt{\lambda_{\ell,j}}\,m^{1-s}
\qquad\text{on }\mathcal N_i^\star.\end{equation}
Applying Lemma~\ref{lem:local L1 coercivity} with
\(V(x):=\lambda_{\ell,j}^{(2)}\rho_j(x)\), \(
\Lambda:=\lambda_{\ell,j}^{(2)}\rho_{\max}\), and \(A:=c_0\sqrt{\lambda_{\ell,j}}\) yields
\begin{align*}
\frac{\lambda_{\ell,j}^{(2)}\rho_{\max}}{m}\int_{\mathcal N_i^\star}|h|\,d\mathrm{vol}_g+
\int_{\mathcal N_i^\star}\|\nabla_gh\|_g\,d\mathrm{vol}_g
&\ge c\sqrt{\lambda_{\ell,j}}\,m^{-(d_j+s)}
\nonumber\\ &\quad-\frac{C}{m}\int_{\mathcal N_i^\star}
\left(|\mathcal R_m|+|\lambda_{\ell,j}^{(1)}-\lambda_{\ell,j}^{(2)}|\rho_j|f_{\ell,j}^{(1)}|
\right)d\mathrm{vol}_g.
\end{align*}
Since \(\mathcal N_i^\star\) are pairwise disjoint, summing over $i\in\mathcal D$ then yields
\begin{align}
\frac{\lambda_{\ell,j}^{(2)}\rho_{\max}}{m}
&\int_{\M_j}|h|\,d\mathrm{vol}_g
+\int_{\M_j}\|\nabla_gh\|_g\,d\mathrm{vol}_g
\nonumber\\&\ge c\sqrt{\lambda_{\ell,j}}\,m^{-s}
-\frac{C}{m}\int_{\M_j}|\mathcal R_m|\,d\mathrm{vol}_g
-\frac{C}{m}|\lambda_{\ell,j}^{(1)}-\lambda_{\ell,j}^{(2)}|\int_{\M_j}|f_{\ell,j}^{(1)}|\,d\mu_j.
\label{eq:summed-local-separation}\end{align}
The remainder's scale from Lemma~\ref{lem:two_codeword_difference_active_patch} and Hölder's inequality gives
\(\|\mathcal R_m\|_{L^1(\M_j)}\leq C_{\ell,j}\lambda_{\ell,j}m^{-s}\)
which allows us to drop the second term on the right that is of strictly lower order. Meanwhile, the integral in the third term is of constant order according to  \eqref{eq:perturbed-eigenfunction-L2-bound} and so \eqref{eq:small-eigenvalue-gap-assumption} implies that it is at most of the same order as the leading term. Choosing the constants sufficiently small allows this term to be absorbed, yielding
\begin{equation}\label{eq:L1-spectral-separation}
\frac{\lambda_{\ell,j}^{(2)}\rho_{\max}}{m}
\|h\|_{L^1(\M_j)}+\|\nabla_gh\|_{L^1(\M_j)}
\ge c\sqrt{\lambda_{\ell,j}}\,m^{-s}.\end{equation}
The eigenvalue perturbation estimate \eqref{eq:lambda_dot_bound} in Proposition~\ref{prop:eigenpair under perturbation} ensures that
$\lambda_{\ell,j}^{(2)}\le C\lambda_{\ell,j}$ for all sufficiently large
$m$. Moreover, since \(\mathcal M_j\) has finite volume, Hölder's inequality yields 
\[\|h\|_{L^1(\M_j)}\le C\|h\|_{L^2(\M_j)},\qquad
\|\nabla_gh\|_{L^1(\M_j)}\le C\|\nabla_gh\|_{L^2(\M_j)}. \]
By definition, \(f_{\ell,j}^{(b)}\) and \(f_{\ell}^{(b)}\) coincide on \(\mathcal M_j\), while \(f_{\ell}^{(b)}\) vanishes on all the other components. Therefore,
multiplying by $w_j$ and applying the relation $w_j\lambda_{\ell,j}=\lambda_\ell$ proves \eqref{eq:eigenfunc separation}.\end{proof} 

\subsubsection{Upper bound on the KL divergence}\label{sec:kl upper bound}
We verify the KL divergence requirement \eqref{eq:fano_kl} of Fano's method. We show that although the perturbation templates oscillate at spatial scale $m^{-1}$, their amplitude is only of order $m^{-s}$. As a consequence, the resulting densities remain close in KL divergence.
\begin{lemma}[KL divergence upper bound]
\label{lem:kl_bound}
For all sufficiently large $m$ and \(n_j=nw_j\)
\begin{equation}\label{eq:n-sample-kl-bound}
\operatorname{KL}\!\left((\mu^{(1)})^{\otimes n}\,\middle\|\,(\mu^{(2)})^{\otimes n}\right)\le Cn_j\,m^{-2s}.\end{equation}\end{lemma}
\begin{proof}
The feasibility of the perturbed family ensures that \(\rho_j^{(2)}\geq(1/2)\cdot \rho_j\). Using the standard inequality \(\operatorname{KL}(p\|q)\le\int(p-q)^2/q\,d\mathrm{vol}_g,\) we obtain
\[\operatorname{KL}\!\left(\mu_j^{(1)}\,\middle\|\,\mu_j^{(2)}\right)\le\int_{\M_j}\frac{ \bigl(\rho_j^{(1)}-\rho_j^{(2)}\bigr)^2
}{\rho_j^{(2)}}\,d\mathrm{vol}_g\le C_1m^{-2s}\int_{\M_j}\left(\sum_{i=1}^K(c_i^{(1)}-c_i^{(2)})a_i\right)^2d\mu_j.\]
Since the templates \(a_i\) have pairwise disjoint supports,
\begin{align*}
\int_{\mathcal M_j}\left(\sum_{i=1}^K(c_i^{(1)}-c_i^{(2)})a_i\right)^2\,d\mu_j&=\int_{\mathcal M_j}\sum_{i=1}^K(c_i^{(1)}-c_i^{(2)})^2a_i^2\,d\mu_j\leq 4\sum_{i=1}^K
\int_{\operatorname{supp}(a_i)}a_i^2\,d\mu_j\\&\le C_2\sum_{i=1}^K\mu_j(\Omega_i)\le C_2\mu_j(\mathcal M_j)=C_2.
\end{align*} 
Combining this estimate with the additivity and joint convexity of KL divergence implies 
\[\operatorname{KL}\!\left((\mu^{(1)})^{\otimes n}\,\middle\|\,(\mu^{(2)})^{\otimes n}\right)=n\operatorname{KL}\!\left((\mu^{(1)})\,\middle\|\,\mu^{(2)}\right)\le n\sum_{k=1}^Nw_k\operatorname{KL}\!\left(\mu_k^{(1)}\,\middle\|\,\mu_k^{(2)}\right)=Cnw_jm^{-2s}.\]
\end{proof}

\subsection{Minimax Lower Bound via Fano's Method}
\label{sec:conclude minimax lower bound}

We are now in a position to establish Theorem~\ref{thm:spectral convergence lower bound} and its accompanying Corollary~\ref{cor:pure L2 spectral lower bound}. 

\subsubsection{Proof of Theorem~\ref{thm:spectral convergence lower bound}}
It suffices to establish the lower bound over the restricted family constructed above, in which the manifold geometry, mixture weights, and component densities outside $\mathcal M_j$ are held fixed. Since this family is contained in the full parameter space $\{(\mathcal M, \rho) :\M\in\mathbf M,\rho\in\mathcal P^{(s)}\}$, any minimax lower bound for the restricted problem applies directly to the full parameter space.

Recall that the codebook \(\mathcal T\) satisfies \(\log|\mathcal T|\gtrsim m^{d_j}\). For two distinct sign configurations \(\mathbf c^{(1)}, \mathbf c^{(2)} \in \mathcal T\) and for sufficiently large \(n_j\), Lemma~\ref{lem:kl_bound} gives the KL divergence upper bound
\[\frac{\displaystyle\max_{\mathbf c\neq\mathbf c'}
\operatorname{KL}\!\left(
\bigl(\mu^{(1)}\bigr)^{\otimes n}\,\middle\|\,\bigl(\mu^{(2)}\bigr)^{\otimes n}\right)+\log2
}{\log|\mathcal T|}\le C_1n_j\,m^{-(d_j+2s)}+
C_2m^{-d_j}\leq\frac12.\]
Lemma~\ref{lem:spectral separation} implies that corresponding perturbed eigenpairs are separated at scale \(m^{-s}\). In particular, the left side of \eqref{eq:eigenfunc separation} is bounded above by a constant multiple of the $H^1(\M_j)$ distance for $m\ge1$. Consequently
\[d_{\mathrm{eig}}\!\left(
\theta(\rho^{(1)}),\theta(\rho^{(2)})\right)\ge c_{\ell,j}\sqrt{\lambda_\ell w_j}\,m^{-s}.\]
Choose $m=\left\lceil A n_j^{1/(d_j+2s)}\right\rceil$ with $A>0$ sufficiently large. Then Fano's method yields 
\[\inf_{\widehat\theta}\max_{\mathbf c\in\mathcal T}\mathbb E_{\bigl(\mu^{(\mathbf c)}\bigr)^{\otimes n}}
\left[d_{\mathrm{eig}}\!\left(\theta(\rho^{(\mathbf c)}),\widehat\theta(\mathcal{X}_n) \right)\right]\ge c_{\ell,j} \sqrt{\lambda_\ell w_j}\,n_j^{-s/(d_j+2s)}.\]
Because the restricted family is contained in the full parameter space, this proves
\[\inf_{(\widehat\lambda_{\ell},\widehat f_{\ell})}\sup_{\M\in\mathbf M,\,\rho\in\mathcal P^{(s)}}
\mathbb E_{\mathcal X_n\sim\mu_\rho^{\otimes n}}\left[|\lambda_{\ell}-\widehat\lambda_{\ell}|+\|f_{\ell}-\widehat f_{\ell}\|_{H^1(\M_j)}\right]\ge c_{\ell,j}
\sqrt{\lambda_{\ell} w_j}\,n_j^{-s/(d_j+2s)}.\]
Repeating the argument for every admissible supporting component establishes \eqref{eq:spectral-lower-bound}.

\subsubsection{Proof of Corollary~\ref{cor:pure L2 spectral lower bound}}
We use the same finite family, Hamming-separated codebook \(\mathcal T\), and KL bound as
in the proof of Theorem~\ref{thm:spectral convergence lower bound}. It
remains only to strengthen the pairwise eigenfunction separation from the
$H^1$ scale $m^{-s}$ to the $L^2$ scale $m^{-(s+1)}$. Since $h=f_{\ell,j}^{(1)}-f_{\ell,j}^{(2)}$ captures the full difference \(f_{\ell}^{(1)}-f_{\ell}^{(2)}\) by the assumption, we use \(h\) to denote this difference interchangeably.  

If the eigenvalue alternative~\eqref{eq:eigenvalue separation} from Lemma~\ref{lem:spectral separation} holds, then
\[|\lambda_\ell^{(1)}-\lambda_\ell^{(2)}|
\ge c_{\ell,j}\sqrt{\lambda_\ell w_j}\,m^{1-s}
\ge c_{\ell,j}\sqrt{\lambda_\ell w_j}\,m^{-(s+1)}.\]
Otherwise, if the first term on the left side of \eqref{eq:eigenfunc separation} is at least half of its right side, then
\[\|h\|_{L^2(\M_j)}\ge c_{\ell,j} \sqrt{\lambda_\ell w_j}\,m^{1-s}\ge
c_{\ell,j}\sqrt{\lambda_\ell w_j}\,m^{-(s+1)},\]
with all the fixed factors absorbed into
$c_{\ell,j}$. Otherwise, \(\|\nabla_gh\|_{L^2(\M_j)}
\ge c_{\ell,j}\sqrt{\lambda_{\ell}w_j}\,m^{-s}.\)
Since \(\mathcal M_j\) is closed, integration by parts and Cauchy--Schwarz yield
\[\|\nabla_gh\|_{L^2(\M_j)}^2=-\int_{\M_j}h\, \Delta_gh\,d\mathrm{vol}_g\le\|h\|_{L^2(\M_j)}
\|\Delta_gh\|_{L^2(\M_j)}.\]
It remains to upper bound \(\|\Delta_gh\|_{L^2}\). 
Expanding the component eigen-equations yields
\[\Delta_g f_{\ell,j}^{(b)}= \frac{\lambda_{\ell,j}^{(b)}}{\rho_j^{(b)}}f_{\ell,j}^{(b)}-2\left\langle
\nabla_g\log\rho_j^{(b)},\nabla_gf_{\ell,j}^{(b)}\right\rangle_g, \qquad b=1,2.\] 
Subtracting \(b=1\) from \(b=2\) and by the triangle inequality, \(\|\Delta_gh\|_{L^2(\mathcal M_j)}\) is controlled by
\begin{align*}&|\lambda_{\ell,j}^{(1)}-\lambda_{\ell,j}^{(2)}|\|f_{\ell,j}^{(1)}\|_{L^2(\M_j)}
+\lambda_{\ell,j}^{(2)}\|h\|_{L^2(\M_j)}+
\lambda_{\ell,j}^{(2)}\left\|
\frac{1}{\rho_j^{(1)}}-\frac{1}{\rho_j^{(2)}}
\right\|_{L^\infty(\M_j)}\|f_{\ell,j}^{(2)}\|_{L^2(\M_j)}\\&+\|\nabla_g\log\rho_j^{(1)}\|_{L^\infty(\M_j)}
\|\nabla_gh\|_{L^2(\M_j)}+\|\nabla_g\log\rho_j^{(1)}
-\nabla_g\log\rho_j^{(2)}\|_{L^\infty(\M_j)}
\|\nabla_gf_{\ell,j}^{(2)}\|_{L^2(\M_j)}.\end{align*}
The perturbed densities \eqref{eq:perturbed densities}, with the templates bounds \eqref{eq:a_i properties} and disjoint supports, imply
\[\|\rho_j^{(1)}-\rho_j^{(2)}\|_{L^\infty(\M_j)}\le
Cm^{-s},\qquad\|\nabla_g\rho_j^{(1)}-\nabla_g\rho_j^{(2)}\|_{L^\infty(\M_j)}\le
Cm^{1-s}.\]
Since the perturbed densities are uniformly bounded away from zero,
\begin{equation}
\label{eq:log-density-pairwise-difference}
\left\|\frac{1}{\rho_j^{(1)}}-\frac{1}{\rho_j^{(2)}}
\right\|_{L^\infty(\M_j)}\le C m^{-s},
\qquad \|\nabla_g\log\rho_j^{(1)}-\nabla_g \log\rho_j^{(2)}\|_{L^\infty(\M_j)}\le
Cm^{1-s}.\end{equation}
Moreover, integrating the eigenpair perturbation bounds \eqref{eq:lambda_dot_bound} and \eqref{eq:fdot_H1_bound} along the
corresponding density paths gives
\begin{equation} \label{eq:eigenpair-pairwise-stability}
|\lambda_{\ell,j}^{(1)}-\lambda_{\ell,j}^{(2)}|\le C_{\ell,j}m^{-s}, \qquad \|h\|_{H^1(\M_j)}\le C_{\ell,j}m^{-s}.
\end{equation}
Equations~\eqref{eq:log-density-pairwise-difference},
\eqref{eq:eigenpair-pairwise-stability}, and the uniform $H^1$ bounds on
$f_{\ell,j}^{(1)}, f_{\ell,j}^{(2)}$ give
\(\|\Delta_gh\|_{L^2(\M_j)}\le\; C_{\ell,j}m^{1-s}.\) Consequently, 
\[\|h\|_{L^2(\M_j)}\ge\frac{ \|\nabla_gh\|_{L^2(\M_j)}^2}{\|\Delta_gh\|_{L^2(\M_j)}}\ge c_{\ell,j}
\sqrt{\lambda_\ell w_j}\, m^{-(s+1)},\]
where any fixed factors are absorbed into
$c_{\ell,j}$. We therefore have established
\begin{equation}\label{eq:fano-L2-pairwise-separation}|\lambda_\ell^{(1)}-\lambda_\ell^{(2)}|
+\|f_\ell^{(1)}-f_\ell^{(2)}\|_{L^2(\M_j)}\ge c_{\ell,j}\sqrt{\lambda_\ell w_j}\,m^{-(s+1)}.\end{equation}
Choosing \(m\) as in the proof of Theorem~\ref{thm:spectral convergence lower bound}, then Fano's method gives
\[\inf_{(\widehat\lambda_\ell,\widehat f_\ell)}
\sup_{\M\in\mathbf M,\,\rho\in\mathcal P^{(s)}}
\mathbb E_{\mathcal X_n\sim\mu_\rho^{\otimes n}}
\left[|\lambda_\ell-\widehat\lambda_\ell|
+\|f_\ell-\widehat f_\ell\|_{L^2(\M_j)}\right]
\ge c_{\ell,j}\sqrt{\lambda_\ell w_j}\,
n_j^{-(s+1)/(d_j+2s)}.\]
Taking the maximum over all admissible supporting components completes the proof.
\subsubsection{Proof of Theorem~\ref{thm:spectral convergence (unrealistic) upper bound}}
It suffices to estimate the density on $\mathcal M_j$. Choosing $h_j\asymp n_j^{-1/(2s+d_j)}$ in Lemma~\ref{lem:KDE manifold} gives \[\mathbb E\|\widehat\rho_j-\rho_j\|_{L^2(\mathcal M_j)}\leq Cn_j^{-s/(d_j+2s)}.\] By Proposition~\ref{prop:eigenpair under perturbation}, the simple eigenpair is stable under $L^2$-perturbations of the density: \[|\lambda_{\ell}-\widehat\lambda_{\ell}|+\|f_{\ell}-\widehat f_{\ell}\|_{H^1(\mathcal M)}\leq C_{\ell,j}\|\widehat\rho_j-\rho_j\|_{L^2(\mathcal M_j)}.\] Taking expectations and combining the two estimates yields \[\mathbb E_{\mathcal X_n\sim\rho}\left[|\lambda_{\ell}-\widehat\lambda_{\ell}|+\|f_{\ell}-\widehat f_{\ell}\|_{H^1(\mathcal M)}\right]\leq C_{\ell,j}n_j^{-s/(d_j+2s)}.\] Taking the worst case over all admissible supporting components proves the claim.


\section{Spectral Convergence of the Graph Laplacian}
\label{sec:spectral upper bound}
We establish an upper bound for the convergence of the graph Laplacian $\mathcal{L}$ to the Laplace-Beltrami operator $\Delta_{\mathcal{M}}$, thereby proving Theorem~\ref{thm:spectral convergence upper bound}. In contrast to the oracle setting in the construction of Theorem~\ref{thm:spectral convergence (unrealistic) upper bound}, manifold components are not known a priori nor required to be separated in the ambient space. The graph Laplacian then must recover the intrinsic geometry of each manifold component while suppressing interactions across distinct components. To address this challenge, we employ the annular proximity graph with angle constraints introduced in \cite{Li2022}; see Appendix~\ref{ssec:annular_proximity_graph} for a high-level summary of this construction. Our analysis improves upon the best known spectral convergence result in that work \[|\hat\lambda_{\ell}-\sigma_{\eta}\lambda_{\ell}|\lesssim \mathcal O(n^{-\frac{1}{3d-1}}), \qquad \|\hat f_{\ell}-f_{\ell}\|_{L^2(\mathcal M)}\lesssim\mathcal O(n^{-\frac{1}{6d-2}})\] and achieves the sharper convergence rate stated in Theorem \ref{thm:spectral convergence upper bound}. The analytical framework developed here is conditionally based on the high-probability events in Theorems~\ref{thm:intense_within_manifold} and \ref{thm:sparse_between_manifolds}, which simultaneously promote strong within-manifold connectivity while maintaining sparse between-manifold connections. These underlying probabilities are then made explicit and rigorously aggregated in the final proof in Section~\ref{ssec:proof_theorem_7}. 

We focus on the manifold components with maximal intrinsic dimension \(d:=\max_j d_j\) which governs the bottleneck of the overall convergence rate. For some inner and outer annulus radii \(0<\epsilon_-<\epsilon_+<1\) to be chosen appropriately, define the Euclidean and geodesic annulus kernels respectively by
\begin{align*}K_{\epsilon}^E(x,y):=\eta_{\epsilon_+}(\|x-y\|)-\eta_{\epsilon_-}(\|x-y\|),\qquad K_{\epsilon}^g(x,y):=\eta_{\epsilon_+}(d_g(x,y))-\eta_{\epsilon_-}(d_g(x,y)).\end{align*} When working in normal coordinates via the change of variables $y=\exp_x(\epsilon_+v)$, with $\|v\|\leq1$, we further define the kernel in normal coordinates by \[K_{\epsilon}^{\mathcal N}(v):=\eta_1(\|v\|)-\eta_{\epsilon_-/\epsilon_+}(\|v\|).\] Setting the normalization factor $\vartheta_\varepsilon:=(\epsilon_+^{d+2}-\epsilon_-^{d+2})^{-1}$, the graph Laplacian is given by
\[\mathcal Lf_\ell(x):=\frac{\vartheta_\varepsilon}{n}\sum_{y\in\mathcal X_n}K_{\epsilon}^E(x,y)\left(f_{\ell}(x)-f_{\ell}(y)\right).\]
For \(x\in\mathcal X_{n,j}\), the graph Laplacian decomposes into the within-manifold operator 
\begin{equation}\label{eq:within-laplacian}\begin{aligned}
\mathcal{L}^{\mathrm{wth}}f_{\ell}(x)&= \sum_{y\in\mathcal{X}_{n,j}} \frac{w_j}{n_j (\epsilon_+^{d+2}-\epsilon_-^{d+2})}\left[\eta_{\epsilon_+}\left(\|x-y\|\right)-\eta_{\epsilon_-}\left(\|x-y\|\right)\right]\left(f_{\ell}(x)-f_{\ell}(y)\right)\\&= \frac{\vartheta_\varepsilon}{n} \sum_{y\in\mathcal{X}_{n,j}}K_{\epsilon}^E(x,y)\left(f_{\ell}(x)-f_{\ell}(y)\right),\end{aligned}\end{equation}
and the between-manifold operator
\begin{equation}\label{eq:between-laplacian}\begin{aligned}
\mathcal{L}^{\mathrm{btw}}f_{\ell}(x)= \frac{\vartheta_\varepsilon}{n}\sum_{k\neq j,\,y\in\mathcal{X}_{n,k}}K_{\epsilon}^E(x,y)\left(f_{\ell}(x)-f_{\ell}(y)\right).\end{aligned} \end{equation}
Symmetry of \(K_{\epsilon}^E\) identifies 
the quadratic energy \(U:=\langle f_{\ell}, \mathcal L f_{\ell} \rangle_{L^2(\mathcal X_n)} =U_{\mathrm{wth}}+U_{\mathrm{btw}}\) with
\begin{equation}\label{eq:within dirichlet energy}
\begin{aligned}U_{\mathrm{wth}}:=\langle f_{\ell}, \mathcal{L}^{\textrm{wth}} f_{\ell}\rangle_{L^2(\mathcal{X}_n)} =\frac{\vartheta_\varepsilon}{2n^2}\sum_{j=1}^N \sum_{x,y\in\mathcal{X}_{n,j}}K_{\epsilon}^{E}(x,y)(f_{\ell}(x)-f_{\ell}(y))^2,\end{aligned}\end{equation} and
\begin{equation}\label{eq:between dirichlet energy}
\begin{aligned}U_{\mathrm{btw}}:=\langle
f_{\ell},\mathcal{L}^{\textrm{btw}}f_{\ell} \rangle_{L^2(\mathcal{X}_n)}= \frac{\vartheta_\varepsilon}{2n^2}\sum_{j=1}^N\sum_{x\in\mathcal{X}_{n,j}}\sum_{k\neq j,\,y\in\mathcal{X}_{n,k}}K_{\epsilon}^E(x,y)(f_{\ell}(x)-f_{\ell}(y))^2.\end{aligned}\end{equation}
We use \((\widehat \lambda_\ell, \widehat f_\ell)\) to denote an eigenpair of \(\mathcal L\) with normalization \(\|\widehat f_\ell\|_{L^2(\mathcal X_n)}^2=\frac{1}{n}\sum\widehat f_\ell(x)^2=1\).
The graph-induced discrete $H^1$-norm, along with the corresponding seminorm, is given by
\[\|g\|_{H^1(\mathcal X_n)}^2:=\|g\|_{L^2(\mathcal X_n)}^2+|g|_{H^1(\mathcal X_n)}^2,\qquad |g|_{H^1(\mathcal X_n)}^2:=
\frac{\vartheta_\varepsilon}{n^2}\sum_{x,y\in\mathcal X_n}K_\epsilon^E(x,y)\bigl(g(x)-g(y)\bigr)^2.\]




\subsection{Preliminary Eigenpair Convergence}\label{ssec:preliminary_eigenpari_convergence}

Our first step is to derive an error decomposition that separates the within-manifold approximation error from the between-manifold interaction. By the self-adjointness of \(\mathcal L\), \begin{align*}\widehat{\lambda}_\ell\langle \widehat f_{\ell}, f_{\ell}\rangle_{L^2(\mathcal{X}_n)}&=\langle\widehat f_{\ell}, \mathcal{L}f_{\ell}\rangle_{\mathcal{L}^2(\mathcal{X}_n)}=\langle\widehat f_{\ell}, \sigma_{\eta}\Delta_{\mathcal{M}}f_{\ell}\rangle_{L^2(\mathcal{X}_n)}+ \langle\widehat f_{\ell}, (\mathcal{L}-\sigma_{\eta}\Delta_{\mathcal{M}})f_{\ell}\rangle_{L^2(\mathcal{X}_n)} \\ & = \sigma_{\eta}\lambda_{\ell}\langle\widehat f_{\ell}, f_{\ell}\rangle_{L^2(\mathcal{X}_n)}+ \langle\widehat f_{\ell}, (\mathcal{L}^{\mathrm{wth}}-\sigma_{\eta}\Delta_{\mathcal{M}})f_{\ell}\rangle_{L^2(\mathcal{X}_n)}+\langle\widehat f_{\ell}, \mathcal{L}^{\mathrm{btw}}f_{\ell}\rangle_{L^2(\mathcal{X}_n)}.\end{align*} Adding and subtracting $f_{\ell}$ and moving the first term on the second line to the left gives
{\small\[(\widehat\lambda_{\ell}-\sigma_{\eta}\lambda_{\ell})\langle\widehat f_{\ell}, f_{\ell}\rangle_{L^2(\mathcal{X}_n)}=\langle\widehat f_{\ell}, (\mathcal{L}^{\mathrm{wth}}-\sigma_{\eta}\Delta_{\mathcal{M}})f_{\ell}\rangle_{L^2(\mathcal{X}_n)}+\langle\widehat f_{\ell}-f_{\ell},\mathcal{L}^{\mathrm{btw}} f_{\ell}\rangle_{L^2(\mathcal{X}_n)}+\langle f_{\ell}, \mathcal{L}^{\mathrm{btw}}f_{\ell} \rangle_{L^2(\mathcal{X}_n)}.\]}
By Theorem~\ref{thm:early_eigenpair_convergence}, \(\langle\widehat{f_{\ell}},f_{\ell}\rangle_{L^2(\mathcal X_n)}\geq c_0\) with probability at least $1-Cn\exp(-cn\epsilon_+^d)$ for some $0<c_0<1$ independent of $n$. We can then isolate the eigenvalue discrepancy 
{\small\begin{align*}|\widehat\lambda_{\ell}-\sigma_{\eta}\lambda_{\ell}|&\leq C_\ell\big(\big|\langle\widehat f_{\ell}, (\mathcal{L}^{\mathrm{wth}}-\sigma_{\eta}\Delta_{\mathcal{M}})f_{\ell}\rangle_{L^2(\mathcal{X}_n)}\big|+\big|\langle\widehat f_{\ell}-f_{\ell}, \mathcal{L}^{\mathrm{btw}}f_{\ell}\rangle_{L^2(\mathcal{X}_n)}\big|+\big|\langle f_{\ell}, \mathcal{L}^{\mathrm{btw}}f_{\ell} \rangle_{L^2(\mathcal{X}_n)}\big|\big)\end{align*}} for some constant \(C_\ell>1\). The following Lemma~\ref{lem:preliminary_eigenvalue_convergence} controls each term on the right-hand side and gives a preliminary high probability upper bound on the eigenvalue convergence rate. We will improve this rate in Section~\ref{ssec:imporved_convergence}.   
\begin{lemma}[Preliminary eigenvalue convergence]
\label{lem:preliminary_eigenvalue_convergence}
There exist some constant \(C_\ell>0\), depending only on the geometric, the density-class, the intersection parameters and the index \(\ell\) such that with probability at least \(1-CNn\epsilon_+^{-d}\exp(-cn\epsilon_+^{d+4})\), one has
\begin{equation}\label{eq:prelimnimary-eigenvalue-convergence}
|\widehat \lambda_\ell-\sigma_\eta\lambda_\ell|\leq C_\ell \sqrt{\lambda_\ell} B_\ell\epsilon_+^2\log\epsilon_+^{-1}.
\end{equation}
\end{lemma}
\begin{proof}
The proof simply follows by combining Propositions~\ref{prop:consistency of the within-manifold operator} through~\ref{prop:between-manifold-energy-bound} in Appendix~\ref{sec:spectral upper bound}. 
Given $h:\mathcal X_n\to\mathbb R$, write
$h_j:=h|_{\mathcal X_{n,j}}$ for its restriction to the sample from \(\mathcal M_j\). For
$\varphi:\mathcal X_{n,j}\to\mathbb R$, define the componentwise
discrete $H^1$-seminorm by
\[|\varphi|_{H^1(\mathcal X_{n,j})}^2:= \frac{\vartheta_\varepsilon}{n_j^2}\sum_{x,y\in\mathcal X_{n,j}}
K_\epsilon^g(x,y)\bigl(\varphi(x)-\varphi(y)\bigr)^2,\]
and its dual $H^{-1}$-seminorm by
\[\|h_j\|_{H^{-1}(\mathcal X_{n,j})}:=
\sup\Bigg\{\frac{1}{n_j}\sum_{x\in\mathcal X_{n,j}}h_j(x)\varphi(x):\frac{1}{n_j}\sum_{x\in\mathcal X_{n,j}}\varphi(x)=0,\;|\varphi|_{H^1(\mathcal X_{n,j})}\leq1\Bigg\}.\]
It follows that the first contribution can be controlled as 
\[\big|\langle\widehat f_{\ell}, (\mathcal{L}^{\mathrm{wth}}-\sigma_{\eta}\Delta_{\mathcal{M}})f_{\ell}\rangle_{L^2(\mathcal{X}_n)}\big|\leq\|\widehat f_\ell\|_{H^1(\mathcal X_n)}\|(\mathcal{L}^{\mathrm{wth}}-\sigma_{\eta}\Delta_{\mathcal{M}})f_{\ell}\|_{{H^{-1}(\mathcal{X}_n)}}.\] Note that the right side of the quadratic energy is exactly \(\frac12
|g|_{H^1(\mathcal X_n)}^2\) for any \(g:\mathcal X_n\rightarrow\mathbb R\):
\[\langle\mathcal Lg,g\rangle_{L^2(\mathcal X_n)}=\frac{\vartheta_\varepsilon}{2n^2}\sum_{x,y\in\mathcal X_n}K_\epsilon^E(x,y)\bigl(g(x)-g(y)\bigr)^2=\frac{1}{2}|g|_{H^1(\mathcal X_n)}^2.\]
Applying to \(\widehat f_\ell\) and using \(\mathcal L\widehat f_\ell
=\widehat\lambda_\ell\widehat f_\ell\) yields \(|\widehat f_\ell|_{H^1(\mathcal X_n)}^2=2\widehat\lambda_\ell
\|\widehat f_\ell\|_{L^2(\mathcal X_n)}^2=2\widehat\lambda_\ell\). Therefore
\[\|\widehat f_\ell\|_{H^1(\mathcal X_n)}^2=1+2\widehat\lambda_\ell\quad \implies \quad \|\widehat f_\ell\|_{H^1(\mathcal X_n)}\leq C_\ell\Big(1+\sqrt{\widehat \lambda_\ell}\Big).\]
Combining with Proposition~\ref{prop:consistency of the within-manifold operator} establishes, with probability at least $1-CNn\epsilon_+^{-d}\exp(-cn\epsilon_+^{d+4})$,
\[\big|\langle\widehat f_{\ell}, (\mathcal{L}^{\mathrm{wth}}-\sigma_{\eta}\Delta_{\mathcal{M}})f_{\ell}\rangle_{L^2(\mathcal{X}_n)}\big|\leq C_\ell\sqrt{\lambda_\ell}B_\ell\epsilon_+^2\log\epsilon_+^{-1}.\]
For the second contribution, assume for the moment that \(\|\widehat f_\ell-f_\ell\|_{H^{1}(\mathcal X_n)}<\infty\). We establish the exact convergence rate in Lemma~\ref{lem:eigenfunction_convergence} below. Combining with Proposition~\ref{prop:between connections counts} gives
\[\|\widehat f_\ell-f_\ell\|_{H^1(\mathcal X_n)}\|\mathcal L^{\mathrm{btw}}f_\ell\|_{H^{-1}(\mathcal X_n)}\leq C_\ell B_\ell\epsilon_+^2,\] with probability at least $1-Cn\epsilon_+^{-d}\exp(-cn\epsilon_+^{d+4})$. For the last term, Proposition~\ref{prop:between-manifold-energy-bound} gives 
\[\big|\langle f_l,\mathcal{L}^\mathrm{btw}f_l\rangle_{L^2(\mathcal{X}_n)}\big|\leq C_\ell B_{\ell}^2 \epsilon_+^4.\]  
with probability at least \(1-Cn\exp\bigl(-cn\epsilon_+^{d+4}\bigr).\) All pieces together complete the proof.\end{proof}

The high-probability upper bound on the eigenvalues' convergence allows us to bound the eigenfunctions convergence. This is established in the next Lemma~\ref{lem:eigenfunction_convergence}. 
\begin{lemma}[Eigenfunction convergence]\label{lem:eigenfunction_convergence}  Let $(\widehat\lambda_\ell,\widehat f_\ell)$ be an \(L^2(\mathcal X_n)\)-normalized eigenpair of the graph Laplacian $\mathcal L$. Given the preliminary eigenvalue convergence established in Lemma~\ref{lem:preliminary_eigenvalue_convergence}, there exists some constant \(C_\ell>0\), depending only on the geometric, the density-class, the intersection parameters, the index \(\ell\), and the eigengap \(\gamma_\ell\), such that with probability at least \(1-CNn\epsilon_+^{-d}\exp(-cn\epsilon_+^{d+4})\), one has
\[\|\widehat f_\ell-f_\ell\|_{H^1(\mathcal X_n)}\leq C_\ell B_\ell
\epsilon_+^2\log\epsilon_+^{-1}.\]
Here \(B_\ell\) denotes the regularity constant defined in \eqref{eq:B_definition}.
\end{lemma}
\begin{proof}
For notational convenience, set
\(h_\ell:=\widehat f_\ell-f_\ell.\) Since \(\mathcal L1=0\), self-adjointness gives
\[\widehat\lambda_\ell
\langle\widehat f_\ell,1\rangle_{L^2(\mathcal X_n)}
=\langle\mathcal L\widehat f_\ell,1\rangle_{L^2(\mathcal X_n)}
=\langle\widehat f_\ell,\mathcal L1\rangle_{L^2(\mathcal X_n)}=0.\]
Thus \(\widehat f_\ell\) has empirical mean zero. For the population eigenfunction, define
\[\bar f_\ell:=\langle f_\ell,1\rangle_{L^2(\mathcal X_n)},\qquad
\delta_{\ell,n}:=\left|\|f_\ell\|_{L^2(\mathcal X_n)}^2-1\right|,\qquad f_\ell^\circ:=f_\ell-\bar f_\ell,\qquad
h_\ell^\circ:=\widehat f_\ell-f_\ell^\circ.\]
Then \(h_\ell^\circ\) has empirical mean zero and
\(h_\ell^\circ=h_\ell+\bar f_\ell\). Since the discrete \(H^1\)-seminorm is invariant under the addition of constants,
\(|h_\ell^\circ|_{H^1(\mathcal X_n)}=
|h_\ell|_{H^1(\mathcal X_n)}\). Moreover, by orthogonality to the constants,
\begin{equation}
\label{eq:demeaned-L2-decomposition}
\|h_\ell\|_{L^2(\mathcal X_n)}^2
=\|h_\ell^\circ\|_{L^2(\mathcal X_n)}^2+|\bar f_\ell|^2.
\end{equation}
Thus, the discrete \(H^1\)-norm is given by 
\[\|h_\ell\|_{H^1(\mathcal X_n)}^2=\|h_\ell\|_{L^2(\mathcal X_n)}^2+|h_\ell|_{H^1(\mathcal X_n)}^2=\|h_\ell^\circ\|_{L^2(\mathcal X_n)}^2+|\bar f_\ell|^2+|h_\ell|_{H^1(\mathcal X_n)}^2.\]
Since \(\Delta_{\mathcal M}f_\ell^\circ=\Delta_{\mathcal M}f_\ell=\lambda_\ell f_\ell
=\lambda_\ell f_\ell^\circ+
\lambda_\ell\bar f_\ell,\) the centered error satisfies
\begin{equation}
\label{eq:centered-discrete-eigenfunction-error-equation}
(\mathcal L-\widehat\lambda_\ell I)h_\ell^\circ=g_\ell^\circ-\sigma_\eta\lambda_\ell\bar f_\ell,
\end{equation}
where
\[g_\ell^\circ:=-(\mathcal L-\sigma_\eta\Delta_{\mathcal M})f_\ell^\circ
-(\sigma_\eta\lambda_\ell-\widehat\lambda_\ell)f_\ell^\circ.\]
The final term in
\eqref{eq:centered-discrete-eigenfunction-error-equation} is constant
and hence vanishes when paired with \(h_\ell^\circ\), as well as with every positive-eigenvalue eigenvector of \(\mathcal L\). Therefore,
\begin{equation}
\label{eq:centered-energy-identity}
\langle\mathcal Lh_\ell^\circ,h_\ell^\circ
\rangle_{L^2(\mathcal X_n)}-
\widehat\lambda_\ell\|h_\ell^\circ\|_{L^2(\mathcal X_n)}^2
=\langle g_\ell^\circ,h_\ell^\circ
\rangle_{L^2(\mathcal X_n)}. \end{equation}
By the symmetry of \(K_\varepsilon^E\),
\[\langle\mathcal Lh_\ell^\circ,h_\ell^\circ
\rangle_{L^2(\mathcal X_n)}
=\frac{\vartheta_\varepsilon}{2n^2}
\sum_{x,y\in\mathcal X_n}
K_\varepsilon^E(x,y)
\bigl(h_\ell^\circ(x)-h_\ell^\circ(y)\bigr)^2
=\frac12|h_\ell^\circ|_{H^1(\mathcal X_n)}^2.\]
Let
\(S:=\operatorname{span}\{\widehat f_\ell\}\) and let \(P_S^\perp\) denote the orthogonal projection onto \(S^\perp\). Let
\(\{(\widehat\lambda_k,\widehat f_k)\}_{k=1}^n\) be an
\(L^2(\mathcal X_n)\)-orthonormal eigenbasis of \(\mathcal L\), and
write
\[P_S^\perp h_\ell^\circ=
\sum_{k\neq\ell \mathrm{\ s.t.\ } \widehat\lambda_k>0}a_k\widehat f_k,\qquad a_k:=
\langle h_\ell^\circ,\widehat f_k
\rangle_{L^2(\mathcal X_n)},
\qquad b_k:=
\langle g_\ell^\circ,\widehat f_k
\rangle_{L^2(\mathcal X_n)}.\]
The zero mode does not appear because \(h_\ell^\circ\) is
empirically mean zero. Define the empirical eigengap by \(\widehat\gamma_\ell:=\min_{k\neq\ell}|\widehat\lambda_k-\widehat\lambda_\ell|\).
Taking the inner product of
\eqref{eq:centered-discrete-eigenfunction-error-equation} with
\(\widehat f_k\) gives
\[(\widehat\lambda_k-\widehat\lambda_\ell)a_k=
\langle g_\ell^\circ,\widehat f_k\rangle_{L^2(\mathcal X_n)}=b_k,
\qquad k\neq\ell,\quad \widehat\lambda_k>0.\]
The Parseval identity gives \(\|P_S^{\perp}h_\ell^{\circ}\|_{L^2(\mathcal X_n)}^2=\sum_{k\neq\ell}|a_k|^2\). By the Cauchy-Schwarz inequality, 
\begin{equation}\label{eq:projected-gap-duality}
\begin{aligned}
\widehat \gamma_\ell\|P_S^{\perp}h_\ell^\circ\|_{L^2(\mathcal X_n)}^2&=\sum_{k\neq\ell,\, \widehat\lambda_k>0}\widehat\gamma_\ell|a_k|^2\leq\sum_{k\neq\ell,\, \widehat\lambda_k>0}|\widehat \lambda_k-\widehat\lambda_\ell||a_k|^2=\sum_{k\neq\ell,\, \widehat\lambda_k>0}|b_k||a_k|\\&\leq\left(\sum_{k\neq\ell,\, \widehat\lambda_k>0}\frac{|b_k|^2}{\widehat\lambda_k}\right)^\frac{1}{2}\left(\sum_{k\neq\ell,\, \widehat\lambda_k>0}\widehat\lambda_k|a_k|^2\right)^{\frac{1}{2}}=C\|g_\ell^\circ\|_{H^{-1}(\mathcal X_n)}|h_\ell^\circ|_{H^1(\mathcal X_n)}.\end{aligned}\end{equation}
On the other hand, choose the sign of \(\widehat f_\ell\) such that \(\zeta_{\ell,n}:=\langle f_\ell,\widehat f_\ell\rangle_{L^2(\mathcal X_n)}=
\langle f_\ell^\circ,\widehat f_\ell
\rangle_{L^2(\mathcal X_n)}\geq0\). Since \(\|\widehat f_\ell\|_{L^2(\mathcal X_n)}=1\),
\[\|P_S^\perp h_\ell^\circ\|_{L^2(\mathcal X_n)}^2
=\|f_\ell^\circ\|_{L^2(\mathcal X_n)}^2-\zeta_{\ell,n}^2,\qquad \|h_\ell^\circ\|_{L^2(\mathcal X_n)}^2
=(1-\zeta_{\ell,n})^2+
\|P_S^\perp h_\ell^\circ\|_{L^2(\mathcal X_n)}^2.\]
Moreover, \(0\leq\zeta_{\ell,n}\leq\|f_\ell^\circ\|_{L^2(\mathcal X_n)}\), and hence
\[(1-\zeta_{\ell,n})^2
\leq2(1-\|f_\ell^\circ\|_{L^2(\mathcal X_n)})^2
+2(\|f_\ell^\circ\|_{L^2(\mathcal X_n)}-\zeta_{\ell,n})^2
\leq2(1-\|f_\ell^\circ\|_{L^2(\mathcal X_n)})^2
+2\|P_S^\perp h_\ell^\circ\|_{L^2(\mathcal X_n)}^2.\]
Since
\[|1-\|f_\ell^\circ\|_{L^2(\mathcal X_n)}^2|
\leq\delta_{\ell,n}+|\bar f_\ell|^2,
\qquad
|1-\|f_\ell^\circ\|_{L^2(\mathcal X_n)}|
\leq|1-\|f_\ell^\circ\|_{L^2(\mathcal X_n)}^2|,\]
it follows that
\begin{equation}
\label{eq:centered-L2-projection-bound}
\|h_\ell^\circ\|_{L^2(\mathcal X_n)}^2
\leq3\|P_S^\perp h_\ell^\circ\|_{L^2(\mathcal X_n)}^2
+C\delta_{\ell,n}^2+C|\bar f_\ell|^4.
\end{equation}
On the preliminary spectral-convergence event from Lemma~\ref{lem:preliminary_eigenvalue_convergence}, for \(n\) sufficiently large, the upper bound in \eqref{eq:prelimnimary-eigenvalue-convergence} converges to zero. Since $\gamma_\ell>0$ is fixed, we therefore have
\[|\widehat\lambda_k-\sigma_\eta\lambda_k|
\leq\frac{\sigma_\eta\gamma_\ell}{4}\] for every index $k$ entering the definition of
$\widehat\gamma_\ell$. Consequently,
\[|\widehat\lambda_k-\widehat\lambda_\ell|\geq\sigma_\eta|\lambda_k-\lambda_\ell|-|\widehat\lambda_k-\sigma_\eta\lambda_k|-|\widehat\lambda_\ell-\sigma_\eta\lambda_\ell|\geq\sigma_\eta\gamma_\ell-\frac{\sigma_\eta\gamma_\ell}{2}=\frac{\sigma_\eta\gamma_\ell}{2}.\]
Taking the minimum over $k\neq\ell$ and increasing $n$ further if necessary,
\[\widehat\gamma_\ell
\geq\frac{\sigma_\eta\gamma_\ell}{2}, \qquad \frac{\widehat\lambda_\ell}{\widehat\gamma_\ell}\leq\frac{\sigma_\eta\lambda_\ell
+
|\widehat\lambda_\ell-\sigma_\eta\lambda_\ell|}{\widehat\gamma_\ell}\leq\frac{2\sigma_\eta\lambda_\ell}{\frac{\sigma_\eta\gamma_\ell}{2}}
\leq4\frac{\lambda_\ell}{\gamma_\ell}.\]
Combining \eqref{eq:projected-gap-duality} and \eqref{eq:centered-L2-projection-bound} yields
\begin{align}
\label{eq:centered-shift-bound}
\widehat\lambda_\ell
\|h_\ell^\circ\|_{L^2(\mathcal X_n)}^2\leq
C\frac{\lambda_\ell}{\gamma_\ell}
\|g_\ell^\circ\|_{H^{-1}(\mathcal X_n)}|h_\ell^\circ|_{H^1(\mathcal X_n)}+C_\ell(\delta_{\ell,n}^2+|\bar f_\ell|^4).
\end{align}
Substituting \eqref{eq:centered-shift-bound} into
\eqref{eq:centered-energy-identity} gives
\[ \frac12|h_\ell^\circ|_{H^1(\mathcal X_n)}^2\leq
C\frac{\lambda_\ell}{\gamma_\ell}
\|g_\ell^\circ\|_{H^{-1}(\mathcal X_n)}
|h_\ell^\circ|_{H^1(\mathcal X_n)}+
C_\ell\bigl(\delta_{\ell,n}^2+|\bar f_\ell|^4\bigr).\] The elementary implication \(x^2\leq ax+b^2\Rightarrow x\leq a+b\)
therefore yields
\begin{equation}
\label{eq:centered-H1-seminorm-bound}
|h_\ell^\circ|_{H^1(\mathcal X_n)}
\leq C\frac{\lambda_\ell}{\gamma_\ell}
\|g_\ell^\circ\|_{H^{-1}(\mathcal X_n)}
+C_\ell\bigl(\delta_{\ell,n}+|\bar f_\ell|^2\bigr).\end{equation}
Equations \eqref{eq:projected-gap-duality},
\eqref{eq:centered-L2-projection-bound}, and
\eqref{eq:centered-H1-seminorm-bound} similarly imply
\[\|h_\ell^\circ\|_{L^2(\mathcal X_n)}
\leq C_{\ell,\gamma}
\|g_\ell^\circ\|_{H^{-1}(\mathcal X_n)}
+C_\ell\bigl(\delta_{\ell,n}+|\bar f_\ell|^2\bigr).\]
Together with \eqref{eq:demeaned-L2-decomposition} and
\(|h_\ell|_{H^1}=|h_\ell^\circ|_{H^1}\), this gives
\begin{equation}
\label{eq:full-H1-with-mean-remainder}
\|h_\ell\|_{H^1(\mathcal X_n)}
\leq C\frac{\lambda_\ell}{\gamma_\ell}
\|g_\ell^\circ\|_{H^{-1}(\mathcal X_n)}
+C_\ell\bigl(\delta_{\ell,n}+|\bar f_\ell|\bigr).\end{equation}
Since both \(\mathcal L\) and \(\Delta_{\mathcal M}\) annihilate
constants,
\((\mathcal L-\sigma_\eta \Delta_{\mathcal M})f_\ell^\circ
=(\mathcal L -\sigma_\eta\Delta_{\mathcal M})f_\ell.
\) Therefore,
\[\|g_\ell^\circ\|_{H^{-1}(\mathcal X_n)}\leq\|(\mathcal L-\sigma_\eta\Delta_{\mathcal M})f_\ell
\|_{H^{-1}(\mathcal X_n)}+
|\widehat\lambda_\ell-\sigma_\eta\lambda_\ell|
\|f_\ell^\circ\|_{H^{-1}(\mathcal X_n)}.\]
By the discrete Poincaré inequality and
\(\|f_\ell^\circ\|_{L^2(\mathcal X_n)}\leq1+\delta_{\ell,n}\),
\(\|f_\ell^\circ\|_{H^{-1}(\mathcal X_n)}\leq C.\) Hence \eqref{eq:prelimnimary-eigenvalue-convergence}
\[\|g_\ell^\circ\|_{H^{-1}(\mathcal X_n)}\leq C\left(\|(\mathcal L-\sigma_\eta\Delta_{\mathcal M})f_\ell
\|_{H^{-1}(\mathcal X_n)}+
|\widehat\lambda_\ell-\sigma_\eta\lambda_\ell|\right).\]
It remains to bound \(|\bar f_\ell|\) and \(\delta_{\ell,n}\).
Since population eigenfunction satisfies \(\int_{\mathcal M}f_\ell d\mu=0\), Bernstein's inequality gives, for any \(t>0\), with probability at least \(1-2\exp(-ct^2)\),
\[|\bar f_\ell|\leq C\left(\frac{t}{\sqrt{n}}+\frac{\|f_\ell\|_{L^\infty(\mathcal M)}t^2}{n}\right).\] Setting \(t=\sqrt{n}\epsilon_+^2\) gives, with probability at least \(1-2\exp(-cn\epsilon_+^4)\), 
\[|\bar f_\ell|\leq C_\ell(\epsilon_+^2+\epsilon_+^4)\leq C_\ell\epsilon_+^2.\]
Applying Bernstein's inequality once more gives, with probability at least \(1-2\exp(-ct^2)\), \[\delta_{\ell,n}=\left|\|f_\ell\|_{L^2(\mathcal X_n)}^2-1\right|\leq C\left(\frac{\|f_\ell\|_{L^\infty(\mathcal M)}t}{\sqrt{n}}+\frac{\|f_\ell\|_{L^\infty(\mathcal M)}^2t^2}{n}\right).\] Setting \(t=\sqrt{n}\epsilon_+^2\) gives, with probability at least \(1-2\exp(-cn\epsilon_+^4)\), \(\delta_{\ell,n}\leq C_\ell\epsilon_+^2\).
Therefore, both remainder terms can be absorbed into the eigenvalue convergence \eqref{eq:prelimnimary-eigenvalue-convergence}. 

To complete the proof, it suffices to control \(\|(\mathcal L
-\sigma_\eta\Delta_{\mathcal M})f_\ell
\|_{H^{-1}(\mathcal X_n)}\). Decompose it by 
\[\|(\mathcal L
-\sigma_\eta\Delta_{\mathcal M})f_\ell\|_{H^{-1}(\mathcal X_n)}\leq{}
\|(\mathcal L^{\mathrm{wth}}
-\sigma_\eta\Delta_{\mathcal M})f_\ell
\|_{H^{-1}(\mathcal X_n)}+
\|\mathcal L^{\mathrm{btw}}f_\ell
\|_{H^{-1}(\mathcal X_n)}.\] Proposition~\ref{prop:consistency of the within-manifold operator} and \ref{prop:between connections counts} together yield, with probability at least \(1-CNn\epsilon_+^{-d}\exp(-cn\epsilon_+^{d+4})\), 
\[\|(\mathcal L
-\sigma_\eta\Delta_{\mathcal M})f_\ell\|_{H^{-1}(\mathcal X_n)}\leq{}C_\ell B_\ell\epsilon_+^2\log\epsilon_+^{-1}.\]
Combining the eigenvalue estimates from Lemma~\ref{lem:preliminary_eigenvalue_convergence} completes the proof.
\end{proof}

\begin{remark}
We can now revisit the proof of Lemma~\ref{lem:preliminary_eigenvalue_convergence}. The second term is now well controlled without imposing the assumption of the eigenfunction convergence. By Lemma~\ref{lem:eigenfunction_convergence}, with probability at least \(1-CNn\epsilon_+^{-d}\exp(-cn\epsilon_+^{d+4})\),
\[\big|\langle\widehat f_{\ell}-f_{\ell}, \mathcal{L}^{\mathrm{btw}}f_{\ell}\rangle_{L^2(\mathcal{X}_n)}\big|\leq\|\widehat f_\ell-f_\ell\|_{H^1(\mathcal X_n)}\|\mathcal L^{\mathrm{btw}}f_\ell\|_{H^{-1}(\mathcal X_n)}\leq C_\ell B_\ell\epsilon_+^4\log\epsilon_+^{-1},\] which is strictly smaller in \(\epsilon_+\).
The rest of the proof remains unchanged. \end{remark}

\subsection{Improved Eigenpair Convergence}\label{ssec:imporved_convergence} In this section, we improve the eigenpair estimate from Section~\ref{ssec:preliminary_eigenpari_convergence}. The analysis takes the same procedure but with one more decomposition
\begin{equation}\label{eq:improved_decomposition}
\begin{aligned}|\hat\lambda_{\ell}-\sigma_{\eta}\lambda_{\ell}|&\leq C_\ell\big(\big|\langle\hat f_{\ell}-f_{\ell}, (\mathcal{L}^{\mathrm{wth}}-\sigma_{\eta}\Delta_{\mathcal{M}})f_{\ell}\rangle_{L^2(\mathcal{X}_n)}\big|+\big|\langle f_{\ell}, (\mathcal{L}^{\mathrm{wth}}-\sigma_{\eta}\Delta_{\mathcal{M}})f_{\ell}\rangle_{L^2(\mathcal{X}_n)}\big|\\&\quad+\big|\langle\hat f_{\ell}-f_{\ell}, \mathcal{L}^{\mathrm{btw}}f_{\ell}\rangle_{L^2(\mathcal{X}_n)}\big|+\big|\langle f_{\ell}, \mathcal{L}^{\mathrm{btw}}f_{\ell} \rangle_{L^2(\mathcal{X}_n)}\big|\big).\end{aligned} \end{equation}
All terms on the right except the second has been controlled in Section~\ref{ssec:preliminary_eigenpari_convergence}. We begin with bounding the discrepancy between the expected within-manifold graph energy and its continuum counterpart: 
\[\text{Bias}_{\mathrm{wth}}\coloneqq\left|\mathbb{E}\langle f_{\ell}, (\sigma_{\eta}\Delta_{\mathcal{M}}-\mathcal{L}^{\mathrm{wth}})f_{\ell}\rangle_{L^2(\mathcal{X}_n)}\right|=\left|\sigma_{\eta}\mathbb{E}\langle f_{\ell},\Delta_{\mathcal{M}}f_{\ell}\rangle_{L^2(\mathcal{X}_n)}-\mathbb{E}\langle f_{\ell},\mathcal{L}^{\mathrm{wth}}f_{\ell}\rangle_{L^2(\mathcal{X}_n)}\right|.\]
By the definition of $\Delta_{\mathcal{M}}$, the continuum term admits the decomposition:
{\small\begin{equation}
\begin{aligned}
\label{eq:laplace-beltrami decomposition}
\sigma_{\eta}\mathbb{E}\langle f_{\ell},\Delta_{\mathcal{M}}f_{\ell}\rangle_{L^2(\mathcal{X}_n)}&=-\sigma_{\eta}\sum_{j=1}^Nw_j\int_{\mathcal{M}_j}f_{\ell}(w_j\Delta_{\mathcal{M}_j}f_{\ell})\rho_j\,d\mathrm{vol}_g\\&\hspace{-1.35cm}=-\sigma_{\eta}\sum_{j=1}^Nw_j^2\int_{\mathcal{M}_j}f_{\ell}\langle\nabla_g\rho_j,\nabla_gf_{\ell}\rangle\rho_j\,d\mathrm{vol}_g-\frac{\sigma_{\eta}}{2}\sum_{j=1}^Nw_j^2\int_{\mathcal{M}_j}f_{\ell}\rho_j^2\Delta_{\mathcal{M}_j}f_{\ell}\,d\mathrm{vol}_g.\end{aligned}\end{equation}} 
We replace the ambient Euclidean kernel $K_{\epsilon}^E$ by its geodesic counterpart $K_{\epsilon}^g$. The discrepancy introduced by this replacement will be collected into a geometric remainder and estimated separately. For $x,y\in\mathcal X_{n,j}$, expand $f_{\ell}(y)$ along the geodesic from $x$ to $y$, \begin{equation}\label{eq:f_ell expansion}
f_{\ell}(y)-f_{\ell}(x)=\langle\nabla_g f_{\ell}(x),\log_x(y)\rangle_g+\frac{1}{2}\langle\nabla_g^2 f_{\ell}(x)\log_x(y),\log_x(y)\rangle_g+\mathcal R_f(x,y),
\end{equation} where 
\[\mathcal R_f(x,y)=\int_{0}^1\frac{(1-t)^2}{2}\nabla_g^3f_{\ell}(\gamma_{xy}(t))[\log_x(y), \log_x(y), \log_x(y)]\,dt\] denotes the exact third-order Taylor remainder satisfying $|\mathcal R_f(x,y)|\le C\,d_g(x,y)^3$. Here $\gamma_{xy}$ denotes the geodesic from $x$ to $y$. The quadratic form then admits the decomposition 
\begin{equation}\label{eq:within decomposition}
U_{\mathrm{wth}}=\langle f_{\ell}, \mathcal L^{\mathrm{wth}}f_{\ell}\rangle_{L^2(\mathcal X_n)}=U_{\mathrm{drift}}+U_{\mathrm{diff}}+U_{\mathrm{rem}}+U_{\mathrm{geo}}.\end{equation} The two leading contribution, and the two remainders are defined respectively as
\begin{equation}\label{eq:linear term}
U_{\mathrm{drift}}:=\big\langle f_\ell, \mathcal{L}^{\mathrm{drift}} f_\ell \big\rangle_{L^2(\mathcal X_n)}=\frac{\vartheta_\varepsilon}{n^2}\sum_{j=1}^N\sum_{x,y\in \mathcal{X}_{n,j}}K_{\epsilon}^g(x,y)f_\ell(x)\langle\nabla_g f_{\ell}(x),\log_x(y)\rangle_g.\end{equation}
\begin{equation}\label{eq:quadratic term}
U_{\mathrm{diff}}:=\big\langle f_\ell, \mathcal{L}^{\mathrm{diff}} f_\ell \big\rangle_{L^2(\mathcal X_n)}=\frac{\vartheta_\varepsilon}{2n^2}\sum_{j=1}^N\sum_{x,y\in \mathcal{X}_{n,j}}K_{\epsilon}^g(x,y)f_\ell(x)\langle \nabla_g^2 f_{\ell}(x)\log_x(y),\log_x(y)\rangle_g.\end{equation}
\begin{equation}\label{eq:Taylor remainder}
U_{\mathrm{rem}}:=\big\langle f_\ell, \mathcal{L}^{\mathrm{rem}} f_\ell \big\rangle_{L^2(\mathcal X_n)}=\frac{\vartheta_\varepsilon}{n^2}\sum_{j=1}^N\sum_{x,y\in \mathcal{X}_{n,j}}K_{\epsilon}^g(x,y)f_\ell(x)\mathcal{R}_f(x,y).\end{equation}
\begin{equation}\label{eq:geometric correction}
U_{\mathrm{geo}}:=\langle f_{\ell}, \mathcal L^{\mathrm{geo}}f_{\ell}\rangle_{L^2(\mathcal X_n)}=\frac{\vartheta_\varepsilon}{n^2}\sum_{j=1}^N\sum_{x,y\in\mathcal X_{n,j}}\big(K_{\epsilon}^E(x,y)-K_{\epsilon}^g(x,y)\big)f_{\ell}(x)(f_{\ell}(y)-f_{\ell}(x)).\end{equation}
The expected drift contribution with respect to the sampling measure is
\[\mathbb EU_{\mathrm{wth}}=\vartheta_\varepsilon\sum_{j=1}^Nw_j^2\int_{\mathcal M_j}\int_{\mathcal M_j}K_{\epsilon}^g(x,y)f_{\ell}(x)\langle\nabla_gf_{\ell}(x),\log_x(y)\rangle_g\rho_j(x)\rho_j(y)d\mathrm{vol}_g(y)d\mathrm{vol}_g(x).\] 
Passing to normal coordinates with $v\in B(0,1)\subset T_x\mathcal M_j$ and  $y=\exp_x(\epsilon_+v)$, one has 
{\small\begin{equation}\label{eq:expected drift}
\mathbb E\,U_{\mathrm{drift}}=\vartheta_\varepsilon\sum_{j=1}^Nw_j^2\int_{\mathcal M_j}f_{\ell}(x)\rho_j(x)\int_{B(0,1)}K_{\epsilon}^{\mathcal N}(v)\langle\nabla_g f_{\ell}(x),\epsilon_+v\rangle_g\rho_j(\exp_x(\epsilon_+v))J_x(\epsilon_+v)\epsilon_+^{d}dvd\mathrm{vol}_g.\end{equation}} 
The remaining terms have analogous normal-coordinate representations
{\small\begin{equation}\label{eq:expected diffusion}
\mathbb{E}\,U_{\mathrm{diff}}=\frac{\vartheta_\varepsilon}{2}\sum_{j=1}^Nw_j^2\int_{\mathcal M_j}f_{\ell}(x)\rho_j(x)\int_{B(0,1)}K_\epsilon^{\mathcal N}(v)\langle\nabla_g^2f_{\ell}(x)\,\epsilon_+v,\epsilon_+v\rangle_g\,\rho_j(\exp_x(\epsilon_+v))J_x(\epsilon_+v)\epsilon_+^{d}dvd\mathrm{vol}_g, \end{equation}}
{\small\begin{equation}\label{eq:expected Taylor remainder}\mathbb{E}\,U_{\mathrm{rem}}=\vartheta_\varepsilon \sum_{j=1}^Nw_j^2\int_{\mathcal M_j}f_{\ell}(x)\rho_j(x)\int_{B(0,1)}K_\epsilon^{\mathcal N}(v)\mathcal R_f(x,\exp_x(\epsilon_+v))\, \rho_j(\exp_x(\epsilon_+v))J_x(\epsilon_+v)\epsilon_+^{d}dvd\mathrm{vol}_g,\end{equation}} {\small\begin{equation}\begin{aligned}\label{eq:expected geometric correction} \mathbb{E}\,U_{\mathrm{rem}}=\vartheta_\varepsilon \sum_{j=1}^Nw_j^2\int_{\mathcal M_j}f_{\ell}(x)\rho_j(x)\times&\int_{B(0,1)}[K_\epsilon^{E}(x,\exp_x(\epsilon_+v))-K_\epsilon^{g}(x,\exp_x(\epsilon_+v))]\\&(f_{\ell}(\exp_x(\epsilon_+v))-f_{\ell}(x))\, \rho_j(\exp_x(\epsilon_+v))J_x(\epsilon_+v)\epsilon_+^dd\mathrm{vol}_g.\end{aligned}\end{equation}} Throughout these expressions, the density is expanded in normal coordinates according to \[\rho_j(\exp_x(\epsilon_+v))=\rho_j(x)+\epsilon_+\langle\nabla_g \rho_j(x),v\rangle_g+\mathcal R_{\rho}(x,v),\qquad |\mathcal R_{\rho}(x,v)|\le C\epsilon_+^2\|v\|^2.\] The Jacobian determinant of the exponential map is denoted by $J_x(\epsilon_+v)$ which satisfies  
\[
J_x(\epsilon_+v)=1+q_x(\epsilon_+v)+\mathcal R_{J}(x,v), \qquad |\mathcal R_{J}(x,v)|\le C\epsilon_+^3\|v\|^3,\] where $q_x(\cdot)$ is a quadratic polynomial whose coefficients depend smoothly on the local curvature. Multiplying the two expansions yields \begin{equation}\label{eq:rho_J_combined} \rho_j(\exp_x(\epsilon_+v))J_x(\epsilon_+v)=\rho_j(x)+\epsilon_+\langle\nabla_g\rho_j(x),v\rangle_g+Q_{\rho,J}(x,v;\epsilon_+)+T_{\rho,J}(x,v;\epsilon_+)+\mathcal O(\epsilon_+^4\|v\|^4),\end{equation}
 where \begin{equation}\label{eq:quadratic contribution}
Q_{\rho,J}(x,v;\epsilon_+):=\rho_j(x)q_x(\epsilon_+v)+ \mathcal R_{\rho}(x,v)\end{equation} collects quadratic-order contribution, and 
\begin{equation}\label{eq:cubic contribution}T_{\rho,J}(x,v;\epsilon_+)=\rho_j(x)\mathcal R_J(x,v)+\epsilon_+ \langle\nabla_g\rho_j(x),v\rangle_gq_x(\epsilon_+v)\end{equation} collects cubic-order contribution. Uniformly over admissible manifolds and density classes \begin{equation}\label{eq:Q_T_bound}
|Q_{\rho,J}(x,v;\epsilon_+)|\leq C\epsilon_+^2\|v\|^2, \qquad |T_{\rho,J}(x,v;\epsilon_+)|\leq C\epsilon_+^3\|v\|^3.
\end{equation}
The remainder of the proof identifies the leading-order contributions and controls all higher-order errors. The drift term \eqref{eq:expected drift} recovers the first continuum contribution in the Laplace-Beltrami decomposition~\eqref{eq:laplace-beltrami decomposition}, which is established in Lemma~\ref{lem:linear Taylor term}. Similarly, the diffusion term~\eqref{eq:expected diffusion} recovers the second continuum contribution in \eqref{eq:laplace-beltrami decomposition}, as shown in Lemma~\ref{lem:quadratic Taylor term}. The Taylor remainder \eqref{eq:expected Taylor remainder} and the geometric correction \eqref{eq:expected geometric correction} are shown to be of higher-order errors $\mathcal{O}(\epsilon_+^2)$ in Lemmas~\ref{lem:taylor remainder R} and~\ref{lem:geometric correction R_1}, respectively. Combining these estimates yield Proposition~\ref{prop:bias of the graph laplacian}, establishing the bias bound for the within-manifold graph Laplacian and its convergence to the tensorized Laplace--Beltrami operator.


\begin{lemma}[Drift contribution] \label{lem:linear Taylor term} The expected drift contribution in \eqref{eq:expected drift} satisfies \begin{align*}
\left|\mathbb{E}\big\langle f_\ell, \mathcal{L}^{\mathrm{drift}} f_\ell \big\rangle_{L^2(\mathcal{X}_n)}-\sigma_{\eta}\sum_{j=1}^Nw_j^2\int_{\mathcal{M}_j}f_{\ell}\rho_j\langle\nabla_g\rho_j,\nabla_g f_{\ell}\rangle d\mathrm{vol}_g\right|\leq C_\ell B_{\ell}^2\epsilon_+^2
\end{align*} where $B_{\ell}$ is defined in \eqref{eq:B_definition} and 
$C_\ell$ depends only on the geometric and density regularity, the index \(\ell\), the kernel profile $\eta$, the fixed annulus ratio $\epsilon_-/\epsilon_+$, and the mixture weights. \end{lemma}
\begin{proof} Substituting the expansion \eqref{eq:rho_J_combined} into \eqref{eq:expected drift}, we decompose the drift contribution according to the order in $\epsilon_+$. The zeroth-order term is \[\vartheta_\varepsilon\sum_{j=1}^Nw_j^2\int_{\mathcal M_j}f_{\ell}(x)\rho_j^2(x)\int_{B(0,1)} K_{\epsilon}^{\mathcal N}(v)\langle\nabla_g f_{\ell}(x),\epsilon_+v\rangle_g\,\epsilon_+^{d}dvd\mathrm{vol}_g(x).\]
Since $K_{\epsilon}^{\mathcal N}$ is radial, it is even in $v$. On the other hand, $\langle\nabla_g f_{\ell}(x),\epsilon_+v\rangle_g$ is odd in $v$. Therefore, the inner integral is odd and so vanishes over the symmetric ball $B(0,1)$. 

For the first-order term, the definition of $K_{\epsilon}^{\mathcal N}(v)$, $\sigma_{\eta}$, and the change of variable yield \begin{align*}
&\int_{B_0(1)}K_{\epsilon}^{\mathcal N}(v)\langle\nabla_g f_{\ell}(x),\epsilon_+v\rangle_g\langle\nabla_g\rho_j(x),\epsilon_+v\rangle_g\epsilon_+^{d}dv\\=\;\;&\epsilon_+^{d+2}\,\frac{\langle\nabla_gf_{\ell}(x),\nabla_g\rho_j(x)\rangle_g}{d}\left[\int_{B(0,1)}\eta_1(\|v\|)\|v\|^2dv-\int_{B(0,\frac{\epsilon_-}{\epsilon_+})}\eta_1\left(\frac{\epsilon_+}{\epsilon_-}\|v\|\right)\|v\|^2dv\right]\\=\;\;&\epsilon_+^{d+2}\,\frac{\langle\nabla_gf_{\ell}(x),\nabla_g\rho_j(x)\rangle_g}{d}\left[1-\left(\frac{\epsilon_-}{\epsilon_+}\right)^{d+2}\right]\int_{B(0,1)}\eta_1(\|v\|)\|v\|^2dv\\=\;\;&\sigma_{\eta}\langle\nabla_gf_{\ell}(x),\nabla_g\rho_j(x)\rangle_g(\epsilon_+^{d+2}-\epsilon_-^{d+2}). \end{align*} Therefore, using the definition of $\vartheta_\varepsilon$
\begin{align*}
\vartheta_\varepsilon\sum_{j=1}^N w_j^2\int_{\mathcal M_j}f_{\ell}(x)\rho_j(x)\int_{B(0,1)}&K_{\epsilon}^{\mathcal N}(v)\langle\nabla_g f_{\ell}(x),\epsilon_+v\rangle_g\langle\nabla_g\rho_j(x),\epsilon_+v\rangle_g\epsilon_+^ddvd\mathrm{vol}_g(x)\\&=\sigma_{\eta}\sum_{j=1}^Kw_j^2\int_{\mathcal M_j}f_{\ell}(x)\langle\nabla_g f_{\ell}(x), \nabla_g\rho_j(x)\rangle_g\rho_j(x)d\mathrm{vol}_g(x).\end{align*}
Hence, the first-order drift contribution exactly matches the first continuum term in \eqref{eq:laplace-beltrami decomposition}. 

It remains to control the higher-order contributions. We begin with the quadratic term $Q_{\rho,J}$. The part involving $q_x$ vanishes by symmetry: $q_x(\epsilon_+v)$ and $K_{\epsilon}^{\mathcal N}$ are even in $v$, while $\langle\nabla_gf_{\ell}(x), \epsilon_+v\rangle_g$ is odd. Hence the corresponding inner integral is zero. The remaining part of $Q_{\rho,J}$ involves the density Taylor remainder $\mathcal R_\rho$. Since we only assume \(\rho_j\in C^2\), we use the integral remainder formula
\[\mathcal R_{\rho}(x,v)=\epsilon_+^2
\int_0^1(1-t)\nabla_g^2\rho_j(\exp_x(t\epsilon_+v))
\left[d(\exp_x)_{t\epsilon_+v}(v),d(\exp_x)_{t\epsilon_+v}(v)\right]\,dt.\] For each $t\in[0,1]$, define the Hessian factor \[\mathcal D_t(x,v):=\nabla_g^2\rho_j (\exp_x(t\epsilon_+v))\left[d(\exp_x)_ {t\epsilon_+v}(v),d(\exp_x)_{t\epsilon_+v}(v)\right].\] The corresponding contribution is \begin{align*}
\mathcal I_t:=&\vartheta_\varepsilon
\sum_{j=1}^Nw_j^2\epsilon_+^{d+2}\int_{\mathcal M_j}f_\ell(x)\rho_j(x)
\int_{B(0,1)}K_\epsilon^{\mathcal N}(v)\langle \nabla_g f_\ell(x),\epsilon_+v\rangle_g\, \mathcal D_t(x,v)
\,dv\,d\mathrm{vol}_g(x).\end{align*}
A direct estimate would only give a $\mathcal O(\epsilon_+)$ bound. To recover the desired \(\mathcal O(\epsilon_+^2)\) rate, we use the geodesic symmetrization argument from Section \ref{sec:symmetrization argument}, which yields $|\mathcal I_t|\le CB_{\ell}^2\epsilon_+^2$ uniformly for all $t\in[0,1]$. Consequently, {\small\[
\left|\vartheta_\varepsilon \sum_{j=1}^Nw_j^2\int_{\mathcal M_j}f_{\ell}\rho_j\int_{B(0,1)}K_{\epsilon}^{\mathcal N}(v)\langle\nabla_g f_{\ell},\epsilon_+v\rangle_g\mathcal R_\rho(x,v)\epsilon_+^ddvd\mathrm{vol}_g\right|=\left|\int_{0}^1(1-t)\mathcal I_tdt\right|\leq C_\ell B_{\ell}^2\epsilon_+^2.\]} Combining this with the symmetry cancellation of the \(q_x\)-term shows that the full quadratic contribution \(Q_{\rho,J}\) contributes only an \(\mathcal O(B_{\ell}^2\epsilon_+^2)\) error.

Finally, we control the cubic contribution $T_{\rho,J}$. Since $|\mathcal R_J(x,v)|\leq C\epsilon_+^3\|v\|^3$, we have \begin{align*}
&\left|\vartheta_\varepsilon\sum_{j=1}^Nw_j^2\int_{\mathcal M_j}f_{\ell}(x)\rho_j(x)\int_{B(0,1)}K_{\epsilon}^{\mathcal N}(v)\langle\nabla_g f_{\ell}(x),\epsilon_+v\rangle_g\rho_j\mathcal R_J(x,v)\epsilon_+^{d}\,dv\,d\mathrm{vol}_g(x)\right|\\\;\;\le& C_\ell\vartheta_\varepsilon\epsilon_+^{d+4}\sum_{j=1}^Nw_j^2\int_{\mathcal M_j}|f_{\ell}(x)||\rho_j(x)|^2|\nabla_gf_{\ell}(x)|d\mathrm{vol}_g(x)\int_{B(0,1)}|K_{\epsilon}^{\mathcal N}(v)|\|v\|^4dv\\\;\;\le& C_\ell\vartheta_\varepsilon\epsilon_+^{d+4}\sum_{j=1}^Nw_j^2\|f_{\ell}\|_{L^2(\mathcal M_j)}\|\nabla_gf_{\ell}\|_{L^2(\mathcal M_j)}\leq C_\ell B_{\ell}^2\epsilon_+^2.\end{align*} 
Similarly, since $q_x(\epsilon_+v)=\mathcal O(\epsilon_+^2\|v\|^2)$, the mixed cubic term satisfies
\begin{align*}&\left|\vartheta_\varepsilon\sum_{j=1}^Nw_j^2\int_{\mathcal M_j}f_{\ell}(x)\rho_j(x)\int_{B(0,1)} K_{\epsilon}^{\mathcal N}(v)\langle\nabla_g f_{\ell}(x),\epsilon_+v\rangle_g\langle\nabla_g\rho_j,\epsilon_+v\rangle_gq_x(\epsilon_+v)\epsilon_+^{d}dvd\mathrm{vol}_g(x)\right|\\\le\;\;&C_\ell \vartheta_\varepsilon\epsilon_+^{d+4}\sum_{j=1}^Nw_j^2\int_{\mathcal M_j}|f_{\ell}(x)||\rho_j(x)||\nabla_gf_{\ell}(x)||\nabla_g\rho_j(x)|d\mathrm{vol}_g(x)\int_{B(0,1)}|K_{\epsilon}^\mathcal N(v)|\|v\|^4dv\\\le\;\;&C_\ell\vartheta_\varepsilon\epsilon_+^{d+4}\sum_{j=1}^Nw_j^2\|f_{\ell}\|_{L^2(\mathcal M_j)}\|\nabla_gf_{\ell}\|_{L^2(\mathcal M_j)}\leq C_\ell B_{\ell}^2\epsilon_+^2.\end{align*} In both estimates, the moment $\int_{B(0,1)}|K_\epsilon^{\mathcal N}(v)|\|v\|^4\,dv$
is bounded by a constant depending only on the kernel profile \(\eta\) and the fixed ratio $\epsilon_-/\epsilon_+$. The fourth-order product remainder in \eqref{eq:rho_J_combined} is even smaller. Hence all contributions beyond the first-order drift term are \(\mathcal O(B_{\ell}^2\epsilon_+^2)\). Combining the zeroth-order cancellation, the exact identification of the first-order contribution, and the quadratic and cubic-remainder estimate completes the proof.\end{proof}

\begin{lemma}[Diffusion contribution]
\label{lem:quadratic Taylor term}
The expected diffusion contribution  \eqref{eq:expected diffusion} satisfies \begin{align*}\left|\mathbb{E}\big \langle f_\ell,\mathcal{L}^{\mathrm{diff}} f_\ell \big\rangle_{L^2(\mathcal{X}_n)}-\frac{\sigma_{\eta}}{2}\sum_{j=1}^Nw_j^2\int_{\mathcal{M}_j}f_{\ell}(x)\rho_j^2(x)\Delta_{\mathcal{M}_j}f_{\ell}(x)\mathrm{d}\mathrm{vol}_g(x)\right|\leq C_\ell B^2_{\ell}\epsilon_+^2, 
\end{align*} where $B_{\ell}$ is defined in \eqref{eq:B_definition} and 
$C_\ell$ depends only on the geometric and density regularity, the index \(\ell\), the kernel profile $\eta$, the fixed annulus ratio $\epsilon_-/\epsilon_+$, and the mixture weights.
\end{lemma} \begin{proof}
The proof follows the same decomposition strategy as in Lemma~\ref{lem:linear Taylor term}, but is substantially simpler. Unlike the drift contribution, the diffusion integrand already contains the quadratic factor $\langle\nabla_g^2f_{\ell}(x)\epsilon_+v,\epsilon_+v\rangle_g$. Thus, it suffices to track the terms through $Q_{\rho,J}$ in the expansion~\eqref{eq:rho_J_combined}; the cubic term $T_{\rho,J}$ and the remainders are of higher order $\mathcal O(\epsilon_+^3)$.

First, consider the zeroth-order term arising from $\rho_j(x)$. By isotropy of the radial kernel, 
\begin{align*}\frac{\vartheta_\varepsilon}{2}\sum_{j=1}^Nw_j^2\int_{\mathcal M_j}f_{\ell}(x)\rho_j(x)\int_{B(0,1)}K_\epsilon^{\mathcal N}(v)\langle\nabla_g^2&f_{\ell}(x)\,\epsilon_+v,\epsilon_+v\rangle_g\,\rho_j(x)\epsilon_+^{d}dvd\mathrm{vol}_g(x)\\&=\frac{\sigma_{\eta}}{2}\sum_{j=1}^Nw_j^2\int_{\mathcal{M}_j}f_{\ell}(x)\rho_j^2(x)\Delta_{\mathcal{M}_j}f_{\ell}(x)d\mathrm{vol}_g(x).\end{align*} Thus, the leading diffusion contribution exactly matches the second continuum term in \eqref{eq:laplace-beltrami decomposition}. 

Next, the integrand involving the first-order density term $\epsilon_+\langle\nabla_g\rho_j(x), v\rangle$ contains \[\langle\nabla_g^2f_{\ell}(x)\epsilon_+v,\epsilon_+v\rangle_g\langle\nabla_g\rho_j(x), \epsilon_+v\rangle_g\] which is cubic and odd in $v$. Since $K_{\epsilon}^{\mathcal N}$ is radial and even, the integral over $B(0,1)$ vanishes. 

It remains to estimate the quadratic remainder $Q_{\rho,J}$. Substituting the bound \eqref{eq:Q_T_bound} yields \begin{align*}\Bigg|\frac{\vartheta_\varepsilon}{2}\sum_{j=1}^Nw_j^2\int_{\mathcal M_j}&f_{\ell}(x)\rho_j(x)\int_{B(0,1)}K_\epsilon^{\mathcal N}(v)\langle\nabla_g^2f_{\ell}(x)\,\epsilon_+v,\epsilon_+v\rangle_g\,Q_{\rho,J}(x,v;\epsilon_+)\epsilon_+^{d}\,dv\,d\mathrm{vol}_g(x)\Big|\\&\le C_\ell\epsilon_+^{2}\sum_{j=1}^N w_j^2\int_{\mathcal M_j}|f_{\ell}(x)||\rho_j(x)|\|\nabla_g^2f_{\ell}(x)\|d\mathrm{vol}_g(x)\int_{B(0,1)}|K_{\epsilon}^{\mathcal N}(v)|\|v\|^4 dv\\&\le C_\ell\epsilon_+^2 \sum_{j=1}^Nw_j^2\|f_{\ell}\|_{L^2(\mathcal M_j)}\|\nabla_g^2f_{\ell}\|_{L^2(\mathcal M_j)}\leq C_\ell B_{\ell}^2\epsilon_+^2.\end{align*}  Combining the leading-order identification, the cancellation of the first-order term, and the estimate for the quadratic remainder completes the proof.\end{proof}

\begin{lemma}[Taylor remainder]
\label{lem:taylor remainder R}The higher-order Taylor remainder term \eqref{eq:expected geometric correction} satisfies \begin{align*}\Big|\mathbb{E}\big\langle f_\ell, \mathcal{L}^{\mathrm{rem}}f_\ell\big \rangle_{L^2(\mathcal{X}_n)}\Big|\le C_\ell B_{\ell}^2\epsilon_+^2,
\end{align*} where $B_{\ell}$ is defined in \eqref{eq:B_definition} and 
$C_\ell$ depends only on the geometric and density regularity, the kernel profile $\eta$, the index \(\ell\), the fixed annulus ratio $\epsilon_-/\epsilon_+$, and the mixture weights.\end{lemma}
\begin{proof}
We follow the same decomposition strategy as in the proofs of Lemmas~\ref{lem:linear Taylor term} and \ref{lem:quadratic Taylor term}. The key simplification here is that the Taylor remainder is already third order in $\epsilon_+$. It suffices to analyze the expansion \eqref{eq:rho_J_combined} only through the first-order term; $Q_{\rho, J}$, $T_{\rho, J}$, and the remaining higher-order terms automatically generate errors of order $\mathcal O(\epsilon_+^2)$ or smaller. 

Due to the symmetry of the unit ball, the resulting integrand involving the zeroth-order contribution $\rho_j(x)$ vanishes identically since it is odd in \(v\). The first nontrivial contribution arises from the first-order density term $\epsilon_+\langle\nabla_g\rho_j(x),v\rangle_g$. Substituting this term into the expected Taylor remainder gives \begin{align*}
\int_{0}^1\frac{(1-t)^2}{2}\mathcal{II}_tdt=\vartheta_\varepsilon \sum_{j=1}^N&w_j^2\int_{\mathcal M_j}f_{\ell}(x)\rho_j(x)\\&\times\int_{B(0,1)}K_{\epsilon}^{\mathcal N}(v)\mathcal R_f(x,\exp_x(\epsilon_+v))\langle\nabla_g \rho_j(x),\epsilon_+v\rangle_g\epsilon_+^{d}dvd\mathrm{vol}_g(x),
\end{align*} where \[\mathcal {II}_t:=\vartheta_\varepsilon\sum_{j=1}^Nw_j^2\epsilon_+^{d+4}\int_{\mathcal M_j}f_{\ell}(x)\rho_j(x)\int_{B(0,1)} K_{\epsilon}^{\mathcal N}(v)\langle\nabla_g\rho_j(x),v\rangle_g\mathcal T_t(x,v)dm(x,v),\] with \[\mathcal T_t(x,v):=\nabla_g^3f_{\ell}(\exp_x(t\epsilon_+v))[d(\exp_x)_{t\epsilon_+v}(v), d(\exp_x)_{t\epsilon_+v}(v), d(\exp_x)_{t\epsilon_+v}(v)].\] Arguing exactly as in the proof of the estimate for $\mathcal I_t$, we obtain \[|\mathcal {II}_t|\le C\vartheta_\varepsilon \epsilon_+^{d+4}\sum_{j=1}^Nw_j^2\left(
\|f_\ell\|_{L^2(\mathcal M_j)}\|\nabla_g^3f_\ell\|_{L^{\infty}(\mathcal M_j)}+\|\nabla_gf_\ell\|_{L^2(\mathcal M_j)}\|\nabla_g^3 f_\ell\|_{L^2(\mathcal M_j)}\right)\leq CB_{\ell}^2\epsilon_+^2,\] uniformly in $t\in[0,1]$. Consequently, \[\left|\int_{0}^1\frac{(1-t)^2}{2}\mathcal{II}_tdt\right|\leq C_\ell B_{\ell}^2\epsilon_+^2.\] 
Finally, the quadratic term $Q_{\rho,J}$ and cubic term $T_{\rho,J}$ in \eqref{eq:rho_J_combined} contribute strictly higher-order factors and are absorbed into the same $\mathcal O(B_{\ell}^2\epsilon_+^2)$ bound. \end{proof}

\begin{lemma}[Geometric correction] \label{lem:geometric correction R_1} The geometric correction term in \eqref{eq:expected geometric correction} satisfies \begin{equation}\label{eq:geometric_remainder}\left| \mathbb{E}\big\langle f_\ell, \mathcal{L}^{\mathrm{geo}} f_\ell \big\rangle_{L^2(\mathcal{X}_n)} \right| \leq C_\ell B_{\ell}^2 \epsilon_+^2,\end{equation} where $B_{\ell}$ is defined in \eqref{eq:B_definition} and 
$C_\ell$ depends only on the geometric and density regularity, the index \(\ell\), the kernel profile $\eta$, the fixed annulus ratio $\epsilon_-/\epsilon_+$, and the mixture weights.
\end{lemma}
\begin{proof} We begin by expanding the geometric kernel discrepancy \[K_{\epsilon}^E(x,y)-K_{\epsilon}^g(x,y)=\left[\eta_{\epsilon_+}(\|x-y\|)-\eta_{\epsilon_+}(d_g(x-y))\right]-\left[\eta_{\epsilon_-}(\|x-y\|)-\eta_{\epsilon_-}(d_g(x-y))\right].\] By the distance comparison estimate from \cite{Belkin2005}, \[\|x-y\|-d_g(x,y)=\mathcal O(d_g(x,y)^3)=\mathcal O(\epsilon_+^3\|v\|^3).\] By the mean-value theorem, together with the Lipschitz continuity of the kernel profile, \begin{equation}\label{eq:kernel_discrepancy_expansion}
K_{\epsilon}^E(x,\exp_x(\epsilon_+v))-K_{\epsilon}^g(x,\exp_x(\epsilon_+v))=\epsilon_+^2\mathcal S(x,v) + \mathcal O(\epsilon_+^3\|v\|^4),
\end{equation} where $|\mathcal S(x,v)|\leq C\|v\|^3$. Substituting this expansion together with \eqref{eq:rho_J_combined} and the Taylor expansion of $f_{\ell}$ \eqref{eq:f_ell expansion}, the geometric correction is governed by the products of $\epsilon_+^2\mathcal S(x,v)$ with \[[\rho_j(x)+\epsilon_+\langle\nabla_g\rho_j(x),v\rangle_g+Q_{\rho,J}(x,v;\epsilon_+)]\left[\epsilon_+\langle\nabla_gf_{\ell}(x),v\rangle_g+\frac{\epsilon_+^2}{2}\langle\nabla_g^2f_{\ell}(x)v,v\rangle_g+\mathcal O(\epsilon_+^3\|v\|^3).\right]\] Since the kernel discrepancy already contributes a factor of order $\mathcal O(\epsilon_+^2\|v\|^3)$, it suffices to retain only terms up to second order in the above product. We first consider the contribution \[\epsilon_+^2\mathcal S(x,v)\rho_j(x)\epsilon_+\langle\nabla_gf_{\ell}(x),v\rangle_g.\] Since both $K_{\epsilon}^E$ and $K_{\epsilon}^g$ are radial kernels, $\mathcal S(x,v)$ is even in $v$, whereas $\langle\nabla_gf_{\ell}(x),v\rangle_g$ is odd. Thus, the integrand is odd and its integral over $B(0,1)$ vanishes. 

The first nontrivial term arises from \[\epsilon_+^2\bigg(\langle\nabla_gf_{\ell}(x),v\rangle_g\langle\nabla_g\rho_j(x),v\rangle_g+\frac{1}{2}\rho_j(x)\langle\nabla_g^2f_{\ell}(x)v,v\rangle_g\bigg).\] The boundedness of the kernel moment $\int_{B(0,1)}|\mathcal S(x,v)|\|v\|^2dv$ gives \[|\mathbb E\langle f_{\ell},\mathcal L^{\mathrm{geo}}f_{\ell}\rangle_{L^2(\mathcal X_n)}|\leq C\vartheta_\varepsilon\epsilon_+^{d+4}\sum_{j=1}^Nw_j^2\int_{\mathcal M_j}\left(|f_{\ell}(x)||\nabla_gf_{\ell}(x)|+|f_{\ell}(x)||\nabla_g^2f_{\ell}(x)|\right)\rho_jd\mathrm{vol}_g(x).\] Applying Cauchy--Schwarz and the definition of $B_{\ell}$ yields \eqref{eq:geometric_remainder}. All remaining terms contain at least one additional factor of $\epsilon_+$ and therefore contribute only $\mathcal O(B_{\ell}^2\epsilon_+^3)$.\end{proof}

Combining Lemmas~\ref{lem:linear Taylor term} through  \ref{lem:geometric correction R_1} yields the bias estimate of the graph Laplacian.
\begin{proposition}[Bias of the graph Laplacian]
\label{prop:bias of the graph laplacian} The discrepancy between the expected within-manifold graph energy and its continuum counterpart
\begin{align}\label{eq:graph laplacian bias}
\left|\mathbb{E}\langle f_{\ell},(\sigma_{\eta}\Delta_{\mathcal{M}}-\mathcal{L}^{\mathrm{wth}})f_{\ell}\rangle_{L^2(\mathcal{X}_n)}\right|\le C_\ell B_{\ell}^2\epsilon_+^2, 
\end{align} where $B_{\ell}$ is defined in \eqref{eq:B_definition} and 
$C$ depends only on the geometric and density regularity, the kernel profile $\eta$, the index \(\ell\), the fixed annulus ratio $\epsilon_-/\epsilon_+$, and the mixture weights.
\end{proposition}

\begin{proposition}[Concentration of the within-manifold graph Laplacian]
\label{prop:concentration of the graph laplacian}
Let $B_{\ell}$ be defined in \eqref{eq:B_definition}. For some constant $C_\ell$ depending only on the geometric and density regularity, the kernel profile $\eta$, the fixed annulus ratio $\epsilon_-/\epsilon_+$, and the mixture weights,
\begin{equation}\label{eq:concentration_result}
\mathbb{P}\Big[\Big|\langle f_{\ell},(\sigma_{\eta}\Delta_{\mathcal{M}}-\mathcal{L}^{\mathrm{wth}})f_{\ell}\rangle_{L^2(\mathcal{X}_n)}\Big|\ge C_\ell B_{\ell}^2\epsilon_+^2\Big]\leq CN\exp\left(-cn \epsilon_+^{d+4}\right)\end{equation}\end{proposition}
\begin{proof}
Given Proposition~\ref{prop:bias of the graph laplacian}, it suffices to control the fluctuation of the within-manifold graph Laplacian around its expectation and the desired concentration estimate then follows from Bernstein's inequality for $U$-statistic stated in Lemma~\ref{lem:bernstein}. To this end, we consider the decomposition introduced in \eqref{eq:within decomposition} and estimate the second-moment bound
$B_{\Psi}$ and the uniform bound $C_{\Psi}$ in \eqref{eq:bernstein second moment bound} and \eqref{eq:bernstein uniform bound}. We treat each contribution separately and combine the resulting bounds via a union bound.

Recall from Section~\ref{sec:symmetrization argument} that the symmetrized representation of $U_{\mathrm{drift}}$ is \[U_{\mathrm{drift}}=\frac{1}{n^2}\sum_{j=1}^N\sum_{x,y\in\mathcal X_{n,j}}K_{\epsilon}^g(x,y)\left(\frac{\vartheta_\varepsilon}{2}\left[f_{\ell}(x)\langle\nabla_gf_{\ell}(x),\log_x(y)\rangle_g+f_{\ell}(y)\langle\nabla_gf_{\ell}(y),\log_y(x)\rangle_g\right]\right).\] Let $(\bar x, \bar v)$ be the geodesic-flow counterpart of $(x,v)$ and define \[\Psi_{\mathrm{drift}}(x,\epsilon_+v):=\frac{\vartheta_\varepsilon\epsilon_+^{d+1}}{2}  [f_{\ell}(x)\langle\nabla_gf_{\ell}(x),v\rangle_g+f_{\ell}(\bar x)\langle\nabla_gf_{\ell}(\bar x), \bar v\rangle_g].\] The cancellation estimate in Section~\ref{sec:symmetrization argument} has \[|f_{\ell}(x)\langle\nabla_gf_{\ell}(x),v\rangle_g+f_{\ell}(\bar x)\langle\nabla_gf_{\ell}(\bar x), \bar v\rangle_g|\leq CB_{\ell}^2\|v\|^2\epsilon_+.\] Substituting this estimate into the definition of $\Psi_{\mathrm{drift}}$, we obtain \[\Big|\int_{B(0,1)}K_{\epsilon}^{\mathcal N}(v)\Psi_{\mathrm{drift}}(x,\exp_x(\epsilon_+v))dv\Big|\leq C_\ell\vartheta_\varepsilon\epsilon_+^{d+2}B_{\ell}^2\int_{B(0,1)}\|v\|^2dv\leq C_\ell B_{\ell}^2,\] where we used $\vartheta_\varepsilon\asymp\epsilon_+^{-(d+2)}$ and the finiteness of the kernel moment $\int_{B(0,1)}\|v\|^2dv$. Consequently, the second-moment bound satisfies \begin{align*}
B_{\Psi_{\mathrm{drift}}}:=\max_{j\in[N]}\int_{\mathcal M_j}\left(\int_{B(0,1)}K_{\epsilon}^{\mathcal N}(v)\Psi_{\mathrm{drift}}(x, \exp_x(\epsilon_+v))dv\right)^2d\mathrm{vol}_g(x)\leq C_\ell B_{\ell}^4.\end{align*} For the uniform bound, using $|\log_x(y)|\leq\epsilon_+$ and the fixed mixture weights, we obtain \[C_{\Psi_{\mathrm{drift}}}:=\max_{j\in[N]}\sup_{x\in\mathcal M_j}\sup_{\|v\|\leq1}\Big|\Psi_{\mathrm{drift}}(x,\exp_x(\epsilon_+v))\Big|\leq\max_{j\in[N]}\frac{C_\ell}{\epsilon_+}\|f_{\ell}\|_{L^\infty(\mathcal M_j)}\|\nabla_gf_{\ell}\|_{L^\infty(\mathcal M_j)}\leq C_\ell B_{\ell}^2\epsilon_+^{-1}.\] 

We consider the quadratic contribution and the Taylor remainder simultaneously. Define \[[\Psi_{\mathrm{diff}}+\Psi_{\mathrm{rem}}](x,\exp_x(\epsilon_+v)):=\vartheta_\varepsilon\epsilon_+^{d+2}f_{\ell}(x)\left[\frac{1}{2}\langle\nabla_g^2f_{\ell}(x)v,v\rangle_g+\epsilon_+\langle\nabla_g^3f_{\ell}(x)v,v,v\rangle_g\right].\] Using the finiteness of the kernel moment and the definition of $B_{\ell}$, 
\[\Big|\int_{B(0,1)}K_{\epsilon}^{\mathcal N}(v)[\Psi_{\mathrm{diff}}+\Psi_{\mathrm{rem}}](x,\exp_x(\epsilon_+v))dv\Big|\leq C\vartheta_\varepsilon\epsilon_+^{d+2}B_{\ell}^2\int_{B(0,1)}\|v\|^2+\epsilon_+\|v\|^3dv\leq CB_{\ell}^2(1+\epsilon_+).\] Using the inequality $(a+b)^2 \leq 2a^2 + 2b^2$, we obtain the second-moment bound as \[B_{\Psi_{\mathrm{diff}}+\Psi_{\mathrm{rem}}}:=\max_{j\in[N]}\int_{\mathcal M_j}\left(\int_{B(0,1)}K_{\epsilon}^{\mathcal N}(v)[\Psi_{\mathrm{diff}}+\Psi_{\mathrm{rem}}](x, \exp_x(\epsilon_+v))dv\right)^2d\mathrm{vol}_g(x)\leq CB_{\ell}^4(1+\epsilon_+^2).\] Analogous to the estimate of the drift uniform bound, we have \[C_{\Psi_{\mathrm{diff}}+\Psi_{\mathrm{rem}}}\leq\max_{j\in[N]}C\left(\|f_{\ell}\|_{L^\infty(\mathcal M_j)}\|\nabla_g^2f_{\ell}\|_{L^\infty(\mathcal M_j)}+\epsilon_+\|f_{\ell}\|_{L^\infty(\mathcal M_j)}\|\nabla_g^3f_{\ell}\|_{L^\infty(\mathcal M_j)}\right)\leq CB_{\ell}^2(1+\epsilon_+).\]
Finally, for the geometric correction, recall that the kernel discrepancy admits the expansion \eqref{eq:kernel_discrepancy_expansion}, which implicitly contributes a factor of order $\epsilon_+^2$. Consequently, we define \[\Psi_{\mathrm{geo}}(x,\exp_x(\epsilon_+v)):=\vartheta_\varepsilon\epsilon_+^{d+2}f_{\ell}(x)[\epsilon_+\langle\nabla_gf_{\ell}(x),v\rangle_g+\mathcal O(\epsilon_+^2v^2)].\] Retaining the leading order term is sufficient to hit the target order. Proceeding exactly as above, the second-moment and the uniform bound satisfy \[B_{\Psi_{\mathrm{geo}}}\leq CB_{\ell}^4\epsilon_+^2, \qquad C_{\Psi_{\mathrm{geo}}}\leq CB_{\ell}^2\epsilon_+.\] All together and using $\Psi=\Psi_{\mathrm{drift}}+\Psi_{\mathrm{diff}}+\Psi_{\mathrm{rem}}+\Psi_{\mathrm{geo}}$ yields \[B_{\Psi}\leq C_\ell B_{\ell}^4, \quad \text{and} \quad C_{\Psi}\leq C_\ell B_{\ell}^2(\epsilon_+^{-1}+1+\epsilon_+).\] Bernstein's inequality in Lemma~\ref{lem:bernstein} with $t=c\epsilon_+^2B_{\ell}^2$ and $U=\langle f_{\ell}, \mathcal L^{\mathrm{wth}}f_{\ell}\rangle_{L^2(\mathcal M_j)}$ gives \[\mathbb P[|U-\mathbb E[U]|>C_\ell\epsilon_+^2B_{\ell}^2]\leq C\exp\left(-\frac{n\epsilon_+^d(\epsilon_+^2B_{\ell}^2)^2}{B_{\ell}^4+\epsilon_+^2B_{\ell}^2B_{\ell}^2(\epsilon_+^{-1}+1+\epsilon_+)}\right)\leq C\exp\left(-cn\epsilon_+^{d+4}\right).\] Taking a union bound over all $N$ sub-manifolds yields the desired probability bound.\end{proof}

\begin{lemma}[Improved eigenvalue convergence]
\label{lem:imporved_eigenvalue_convergence} Let $B_{\ell}$ be the eigenfunction regularity defined in \eqref{eq:B_definition}. For some constant $C_\ell$ depending on the geometric and density regularity, the index \(\ell\), the eigengap \(\gamma_\ell\), the kernel profile $\eta$, the fixed annulus ratio $\epsilon_-/\epsilon_+$, and the mixture weights, with probability at least \(1-CNn\epsilon_+^{-d}\exp(-cn\epsilon_+^{d+4})\), one has
\begin{equation}\label{eq:improved-eigenvalue-convergence}
|\widehat \lambda_\ell-\sigma_\eta\lambda_\ell|\leq C_\ell B_\ell\epsilon_+^2,
\end{equation}
\end{lemma}
\begin{proof}
The proof proceeds by controlling each contribution in the decomposition in \eqref{eq:improved_decomposition}. By Lemma~\ref{lem:eigenfunction_convergence} and Proposition~\ref{prop:consistency of the within-manifold operator}, with probability at least $1-CNn\epsilon_+^{-d}\exp(-cn\epsilon_+^{d+4})$,
\[\big|\langle\hat f_{\ell}-f_{\ell}, (\mathcal{L}^{\mathrm{wth}}-\sigma_{\eta}\Delta_{\mathcal{M}})f_{\ell}\rangle_{L^2(\mathcal{X}_n)}\big|\leq C_\ell B_\ell^2\epsilon_+^4(\log\epsilon_+^{-1})^2.\]
By Proposition~\ref{prop:concentration of the graph laplacian}, with probability at least \(1-CN\exp(-cn\epsilon_+^{d+4})\),
\[\Big|\langle f_{\ell},(\sigma_{\eta}\Delta_{\mathcal{M}}-\mathcal{L}^{\mathrm{wth}})f_{\ell}\rangle_{L^2(\mathcal{X}_n)}\Big|\le C_\ell B_{\ell}^2\epsilon_+^2.\]
The other two terms are bounded similarly as in the proof of Lemma~\ref{lem:preliminary_eigenvalue_convergence}. Hence, the second and third term dominates the convergence, which concludes the proof.
\end{proof}


\subsection{Proof of Theorem~\ref{thm:spectral convergence upper bound}}
\label{ssec:proof_theorem_7}
The proof follows directly from Lemmas~\ref{lem:eigenfunction_convergence} and \ref{lem:imporved_eigenvalue_convergence}. Specifically, with probability at least \(1-CNn\epsilon_+^{-d}\exp(-cn\epsilon_+^{d+4})\),
\[|\widehat \lambda_\ell-\sigma_\eta\lambda_\ell|\leq C_\ell B_\ell\epsilon_+^2, \qquad \|\widehat f_\ell-f_\ell\|_{H^1(\mathcal X_n)}\leq C_\ell B_\ell
\epsilon_+^2\log\epsilon_+^{-1}.\] This high-probability event is conditioned on Theorems~\ref{thm:intense_within_manifold} and \ref{thm:sparse_between_manifolds}. Under the choice of \(r=\epsilon_+^3\) in our construction, the two theorems together require \(\epsilon_+\gg C(\log n /n)^{\frac{1}{5d-2}}\) so that both theorems hold simultaneously with probability at least \(1-Cn^{-2}\). Under this choice, for every \(d\geq2\),
\(Nn\epsilon_+^{-d}\exp(-cn\epsilon_+^{d+4})=o(n^{-2})\). Substituting this back to the eigenpair estimate yields, for all \(d\geq2\), with probability at least \(1-Cn^{-2}\), 
\[|\widehat \lambda_\ell-\sigma_\eta\lambda_\ell|\leq C_\ell B_\ell  {\Big(\frac{\log n}{n}\Big)}^{\frac{1}{2.5d-1}}, \qquad \|\widehat f_\ell-f_\ell\|_{H^1(\mathcal X_n)}\leq C_\ell B_\ell {\Big(\frac{\log n}{n}\Big)}^{\frac{1}{2.5d-1}}\log n.\]
This completes the proof.


\section{Experiments}
\label{sec:experiments}

To numerically verify Theorems~\ref{thm:spectral convergence lower bound}, \ref{thm:spectral convergence (unrealistic) upper bound}, and \ref{thm:spectral convergence upper bound}, we construct synthetic data on two connected, unit-volume flat tori for each $d \in \{1, 2, 3\}$, parameterized by $\mathbb{T}^d := \mathbb{R}^d / \mathbb{Z}^d$ and identified with the periodic coordinate domain $[0, 1)^d$. Define the unit-speed, unit-length circle 
\(\gamma(t) := \frac{1}{2\pi} \bigl(\cos(2\pi t) - 1,\ \sin(2\pi t)\bigr).\) For $d = 1$, the embeddings into $\mathbb{R}^4$ are given by
\[F_1^{(1)}(u)=(\gamma(u),0_2),\qquad F_2^{(1)}(u)=(0_2,\gamma(u)).\]
For $d \in \{2, 3\}$, let $\Gamma_d(u) := (\gamma(u_2), \dots, \gamma(u_d))$, and define the embeddings into $\mathbb{R}^{4d-2}$ by
\[F_1^{(d)}(u) = \bigl(\gamma(u_1), \Gamma_d(u), 0_{2(d-1)}\bigr), \qquad
F_2^{(d)}(u) = \bigl(\gamma(u_1), 0_{2(d-1)}, \Gamma_d(u)\bigr).\]
Setting $\mathcal{M}_j^{(d)} := F_j^{(d)}(\mathbb{T}^d)$, these embeddings induce the standard flat metric satisfying
\[(DF_j^{(d)})^{\mathsf{T}} DF_j^{(d)} = I_d, \qquad \operatorname{vol}_g(\mathcal{M}_j^{(d)}) = 1.\]
The two components intersect at the origin when $d = 1$, whereas for $d \in \{2, 3\}$, they intersect along the common unit-length circle
\[\mathcal{M}_1^{(d)} \cap \mathcal{M}_2^{(d)} = \bigl\{(\gamma(t), 0_{2(d-1)}, 0_{2(d-1)}) : t \in \mathbb{T}^1\bigr\},\]
with their transverse tangent directions mutually orthogonal.

Our nonuniform periodic sampling density on each component is given by
\[\rho_j^{(d)}\bigl(F_j^{(d)}(u)\bigr) = 1 + \sum_{r=1}^{8} a_{jr} \, \varphi_r\!\left(2\pi \boldsymbol{k}_{jr}^{(d)} \cdot u\right), \qquad j \in \{1, 2\},\]
where $\varphi_r(\theta)$ equals $\cos\theta$ for $r \in \{1, 2, 3, 5, 6, 7\}$ and $\sin\theta$ for $r \in \{4, 8\}$. The amplitudes are shared across all \(d\), while the frequencies are specified in Table~\ref{tab:numerical-density-design}. For $d = 1$, we set
\[\boldsymbol{k}_{jr}^{(1)} = \boldsymbol{k}_{jr,1}^{(2)} + 4\boldsymbol{k}_{jr,2}^{(2)},\]
implying that the one-dimensional chart density is precisely the restriction of the two-dimensional chart density to $u \mapsto (u, 4u \bmod 1)$. All frequencies are nonzero, so the densities are properly normalized. 
Moreover,
\(\sum_{r=1}^{8}|a_{1r}|\approx0.552370\) and
\(\sum_{r=1}^{8}|a_{2r}|\approx0.666109\),
which imply the dimension-independent bounds
\(0.3338\leq\rho_j^{(d)}\leq1.6662.\)
The densities are \(C^\infty\) and hence naturally satisfy the
density smoothness assumptions with \(s=2\). See Figure~\ref{fig:density} for a visualization for this density construction.  

\begin{table}[t]
\centering
\small
\setlength{\tabcolsep}{4pt}
\begin{tabular}{cccccccc}
\toprule
\(r\) & \(\varphi_r\) & \(a_{1r}\) & \(a_{2r}\)
& \(\boldsymbol k_{1r}^{(2)}\)
& \(\boldsymbol k_{2r}^{(2)}\)
& \(\boldsymbol k_{1r}^{(3)}\)
& \(\boldsymbol k_{2r}^{(3)}\)\\
\midrule
1 & \(\cos\) & 0.19429357 & 0.21530560
  & \((2,0)\) & \((2,0)\)
  & \((2,0,0)\) & \((2,0,0)\)\\
2 & \(\cos\) & 0.07146951 & 0.16111551
  & \((0,2)\) & \((0,2)\)
  & \((0,2,0)\) & \((0,2,0)\)\\
3 & \(\cos\) & 0.035 & 0.040
  & \((1,1)\) & \((1,-1)\)
  & \((1,1,0)\) & \((1,-1,0)\)\\
4 & \(\sin\) & 0.025 & 0.020
  & \((2,1)\) & \((1,2)\)
  & \((2,1,0)\) & \((1,2,0)\)\\
5 & \(\cos\) & 0.11797347 & 0.07914921
  & \((4,0)\) & \((4,0)\)
  & \((0,0,2)\) & \((0,0,2)\)\\
6 & \(\cos\) & 0.05728538 & 0.10603882
  & \((0,4)\) & \((0,4)\)
  & \((0,4,0)\) & \((0,4,0)\)\\
7 & \(\cos\) & 0.03634769 & 0.02649998
  & \((2,2)\) & \((2,-2)\)
  & \((1,1,1)\) & \((1,1,1)\)\\
8 & \(\sin\) & 0.015 & 0.018
  & \((4,1)\) & \((1,4)\)
  & \((1,1,2)\) & \((1,1,2)\)\\
\bottomrule
\end{tabular}
\caption{Density amplitudes and frequencies used in the
numerical experiments. Coefficients are displayed to eight
decimal places; the computations use their unrounded values.}
\label{tab:numerical-density-design}
\end{table}

\begin{figure}[htbp]
    \centering
    \includegraphics[width=0.8\textwidth]{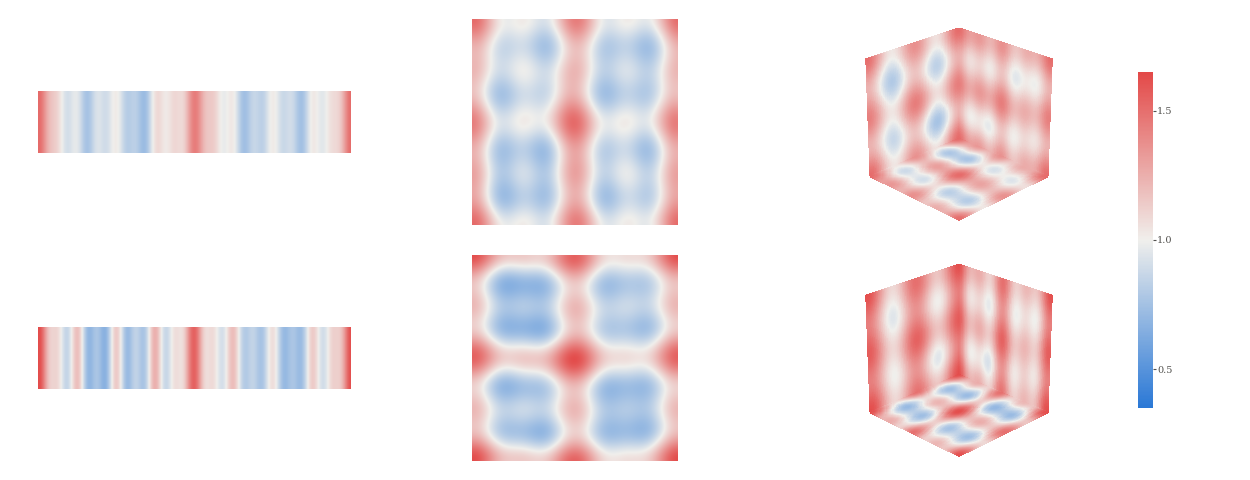}
    \caption{The nonuniform sampling densities $\rho_j^{(d)}$ on the unit-volume flat tori for $d \in \{1,2,3\}$ and $j \in \{1,2\}$. Each column displays the densities for the two components at dimension $d$.}
    \label{fig:density}
\end{figure}

Assigning equal weights to both components, we sample from the mixture distribution
\[
\mu^{(d)} = \frac{1}{2} \mu_1^{(d)} + \frac{1}{2} \mu_2^{(d)}, \qquad d\mu_j^{(d)} = \rho_j^{(d)} \, d\operatorname{vol}_g.
\]
Specifically, for each observation, we independently draw a component index $Z_i \in \{1, 2\}$ with $\mathbb{P}(Z_i = j) = 1/2$, sample the periodic coordinate $U_i$ according to the chart density $\rho_{Z_i}^{(d)} \circ F_{Z_i}^{(d)}$, and set $X_i = F_{Z_i}^{(d)}(U_i)$. 
For the PBC coordinates, samples are generated via rejection sampling from the uniform distribution on $[0, 1)^d$, utilizing the envelope constant $B_j = 1 + \sum_{r=1}^{8} |a_{jr}|$ and the acceptance probability $\rho_j^{(d)}\bigl(F_j^{(d)}(u)\bigr) / B_j$. 
The oracle KDE and PBC experiments share all identical component manifolds, densities, and mixture weights.

We first verify Theorems~\ref{thm:spectral convergence lower bound} and \ref{thm:spectral convergence (unrealistic) upper bound} using oracle KDE with periodic Gaussian kernel
\[\widetilde\rho_{j,h}(x)=\frac1{n_j}\sum_{i:Z_i=j}K_h(x-X_i), \qquad 
K_h(x)=\sum_{z\in\mathbb Z^2}\frac1{2\pi h^2}\exp\!\left(-\frac{|x+z|^2}{2h^2}\right).\]
The component labels and intrinsic coordinates are known a priori. Our experiment uses its equivalent discrete Fourier form with prescribed bandwidth
\[K_{h_n}(x)=\sum_{k\in\mathbb Z^2}e^{-2\pi^2h_{n}^2|\boldsymbol{\xi}|^2}e^{2\pi i \boldsymbol{\xi}\cdot x}, \qquad h_n=0.2\,(n/2)^{-1/{(d+2s)}}.\] Let $N=256$. The samples are first binned onto an
$N\times N$ grid of the torus. Let $C_j$ denote the grid counts of component
$j$ and $n_j$ its sample size. The smoothed density is computed in the Fourier
domain as
\[\widetilde\rho_j=\frac{N^{2}}{n_j}\,\operatorname{IFFT}\!\Bigl[\operatorname{FFT}(C_j)(\boldsymbol{\xi})\,
e^{-2\pi^{2}h_n^{2}|\boldsymbol{\xi}|^{2}}\,\mathbf 1_{\{|\boldsymbol{\xi}|_\infty\le 2m_n\}}\Bigr],\qquad m_n=\Bigl\lceil \frac{2.5}{2\pi h_n}\Bigr\rceil+3 ,\]
where $\boldsymbol{\xi}\in\mathbb Z^{2}$, the factor $e^{-2\pi^{2}h_n^{2}|\boldsymbol{\xi}|^{2}}$ is the Fourier multiplier of a Gaussian kernel of bandwidth $h_n$, and the indicator
truncates the spectrum at $|\boldsymbol{\xi}|_\infty\le 2m_n$.
To keep the estimate positive and normalized on the grid, we set
\[\widehat\rho_{j,ab}=\frac{\max\{\widetilde\rho_{j,ab},\,0.05\}}
{N^{-2}\sum_{c,d}\max\{\widetilde\rho_{j,cd},\,0.05\}},
\qquad a,b=1,\dots,N,\]
so $\widehat\rho_{j}$ has unit mean on the grid. In all experiments,
the lower bound $0.05$ was never active. 
\begin{figure}[H]
    \centering
    \begin{subfigure}[b]{0.33\textwidth}
        \centering
        \includegraphics[width=\textwidth]{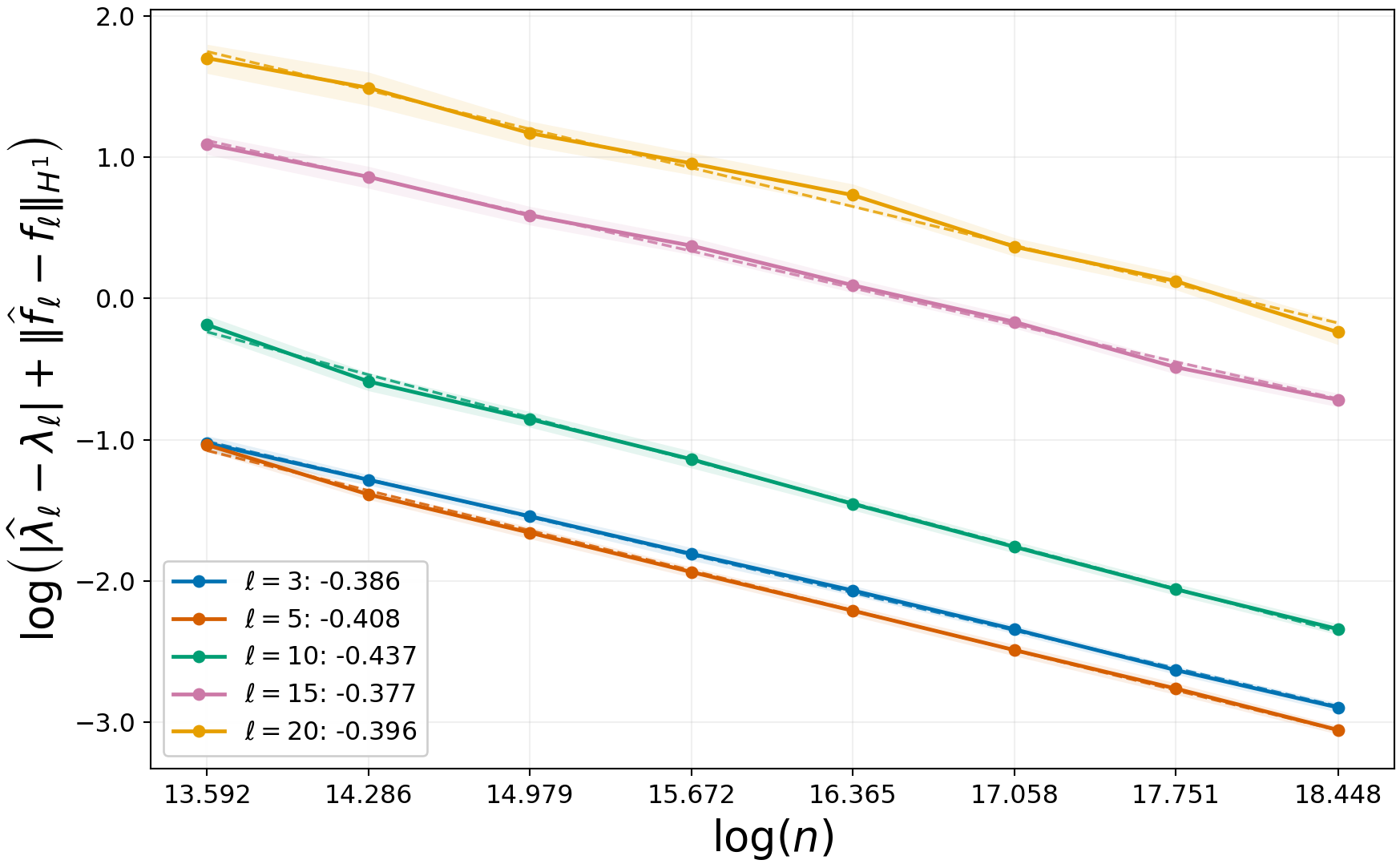}
        \label{fig:1d_KDE}
    \end{subfigure}
    \hfill
    \begin{subfigure}[b]{0.33\textwidth}
        \centering
        \includegraphics[width=\textwidth]{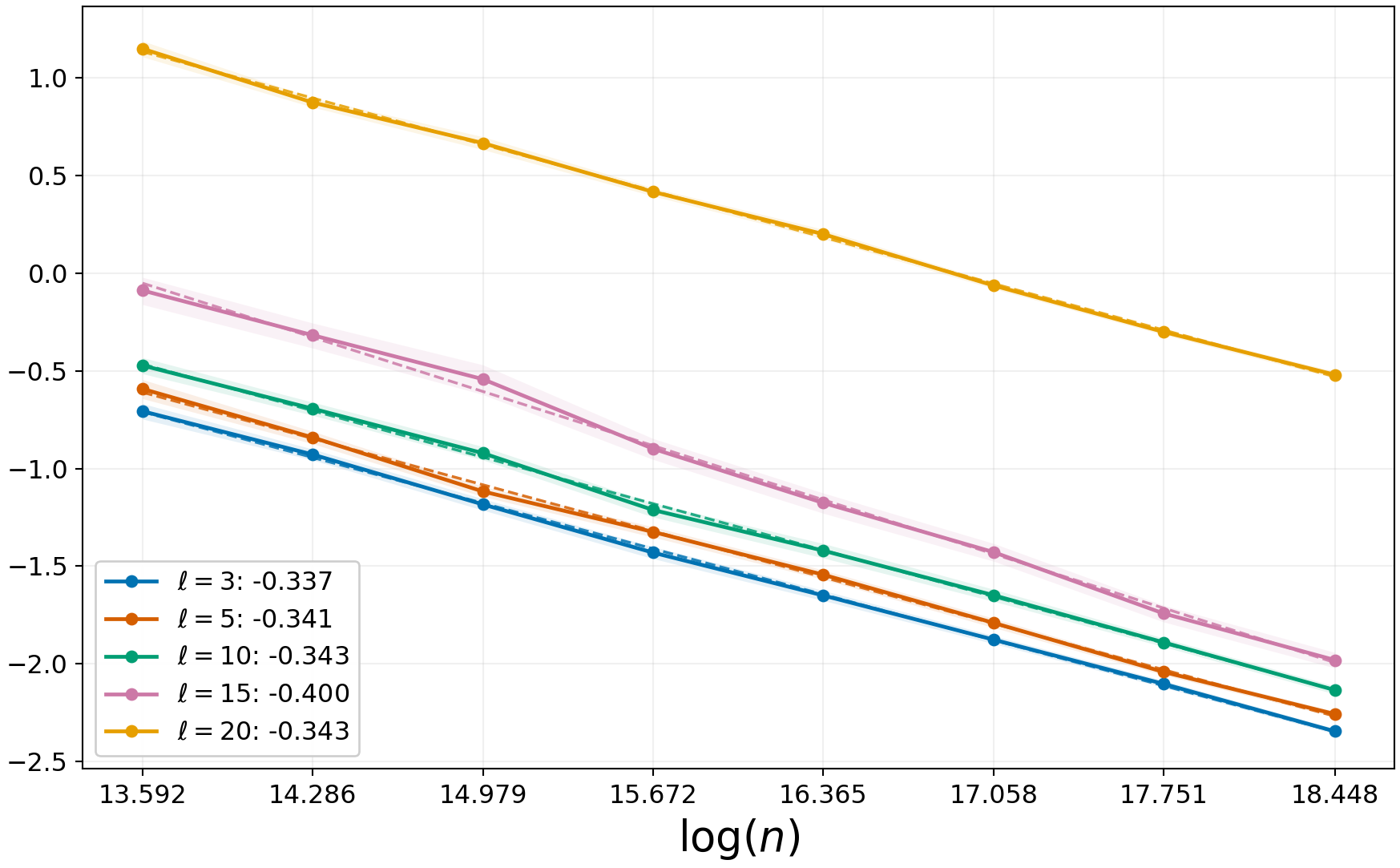}
        \label{fig:2d_KDE}
    \end{subfigure}%
    \hfill
    \begin{subfigure}[b]{0.33\textwidth}
        \centering
        \includegraphics[width=\textwidth]{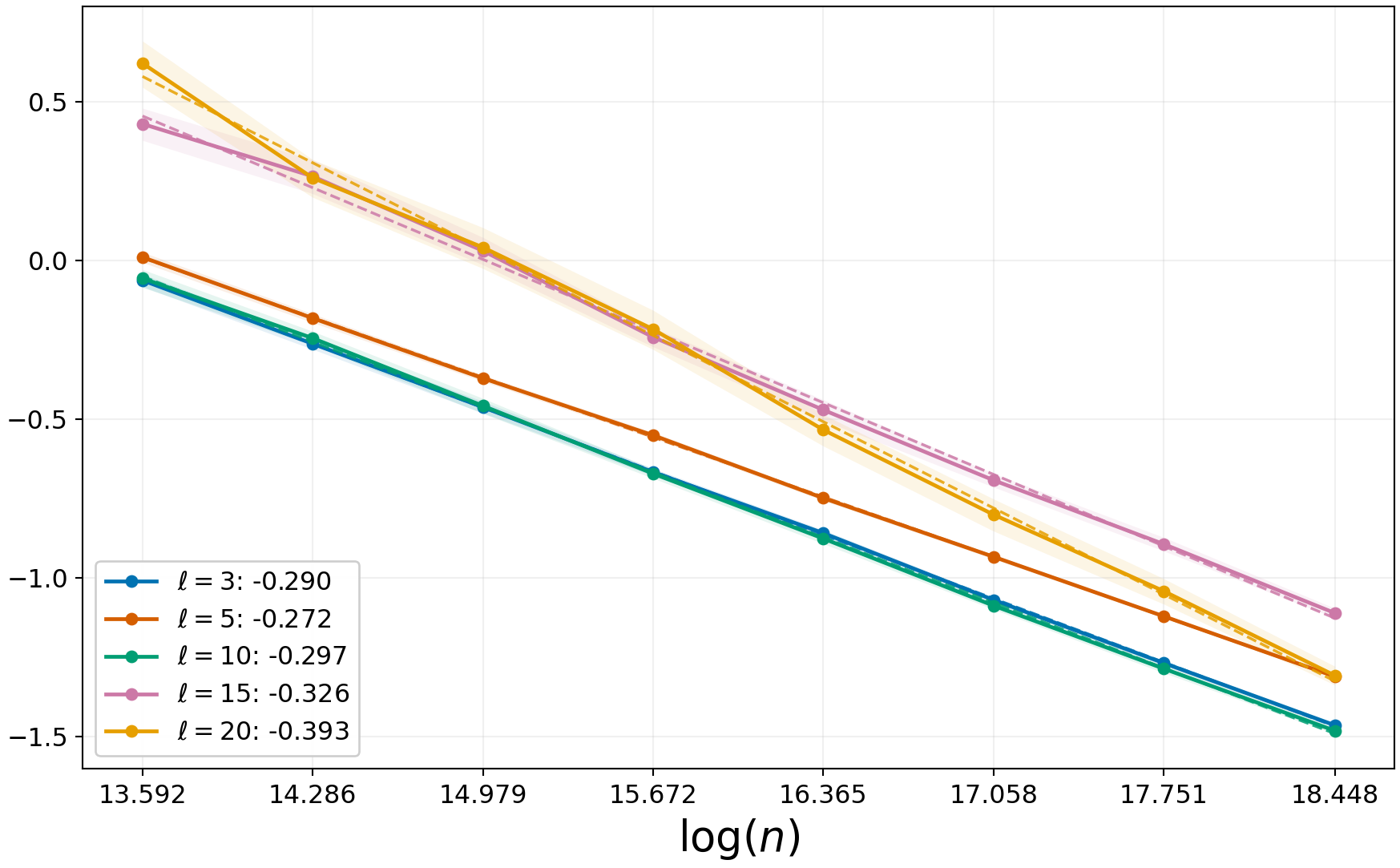}
        \label{fig:3d_KDE}
    \end{subfigure}   
    \caption{Oracle KDE spectral convergence rate for \(d\in\{1,2,3\}\).}
    \label{fig:KDE_convergence}
\end{figure}
The leading eigenpairs are therefore obtained
from the problems
\[-\frac{1}{2\widehat\rho_{j,h}}\operatorname{div}\!\bigl(\widehat\rho_{j,h}^{\,2}
\nabla\widehat f_{\ell,j}\bigr)=\widehat\lambda_\ell\,\widehat f_{\ell,j},\qquad j=1,2,\]
whose weak form is
\[\frac12\int_{\mathbb T^{2}}\widehat\rho_{j,h}^{\,2}\,
\nabla\widehat f_{\ell,j}\cdot\nabla\overline v\,dx
=\widehat\lambda_\ell\int_{\mathbb T^{2}}\widehat\rho_{j,h}\,
\widehat f_{\ell,j}\,\overline v\,dx
\quad\text{for all } v\in H^{1}(\mathbb T^{2}).\]
The problem is solved by a Fourier--Galerkin method in
\(V_{m_n}=\operatorname{span}\bigl\{e^{2\pi i\,\boldsymbol{\xi}\cdot x}:\ |\boldsymbol{\xi}|_\infty\le m_n\bigr\}.\)
The spectra of the two components are then merged and sorted in increasing order
to give $\widehat\lambda_1\le\widehat\lambda_2\le\cdots$. The reference spectrum
$\lambda_1\le\lambda_2\le\cdots$ is computed in the same way from the true
densities $\rho_j$, using a Fourier truncation at $|\boldsymbol{\xi}|_\infty\le 18$. Both the population and the estimated eigenfunctions are normalized using corresponding densities $\rho_j$ and $\widehat\rho_{j,h}$. Experimental results for \(d\in\{1,2,3\}\), \(\ell\in\{3,5,10,15,20\}\) and \(
n=\{0.8,1.6,3.2,6.4,12.8,25.6,51.2,102.8\}\times10^6,
\) are given in the log-log plot in Figure~\ref{fig:KDE_convergence}. The dashed lines are least-squares fits, and the shaded bands show the standard deviation. All the errors are averaged over $50$ independently generated data sets. 

For $s=2$, Theorems~\ref{thm:spectral convergence lower bound} and
\ref{thm:spectral convergence (unrealistic) upper bound} give the rate $n^{-s/(2s+d)}$, that is, $n^{-0.400}$, $n^{-0.333}$ and $n^{-0.286}$ for $d=1,2,3$. The fitted rates are in close agreement with these exponents: $13$ of the $15$ pairs $(d,\ell)$ lie within $0.05$ of
the theoretical exponent, and averaging over $\ell$ gives $0.401$, $0.353$ and
$0.316$ for $d=1,2,3$, against $0.400$, $0.333$ and $0.286$. In particular, the convergence slows down as the dimension grows, as the theory demonstrates.

Two further features are worth noting. First, the largest deviations occur for
large $\ell$, namely $\ell=15$ for $d=2$ ($0.400$) and $\ell=20$ for $d=3$
($0.393$). There the error curves are less regular at small $n$ and the fitted
slopes are steeper than the theoretical one, which suggests that the higher
eigenpairs are resolved only once $n$ is sufficiently large. Second, a slope
slightly steeper than the theoretical exponent does not contradict the lower
bound: the minimax rate describes the worst case over the smoothness class,
whereas the experiments use a single $C^\infty$ density. Finally, the error level
generally increases with $\ell$ while the slopes remain close to each other,
so $\ell$ appears to affect the constant rather than the rate.

\begin{figure}[H]
    \centering
    \begin{subfigure}[b]{0.33\textwidth}
        \centering
        \includegraphics[width=\textwidth]{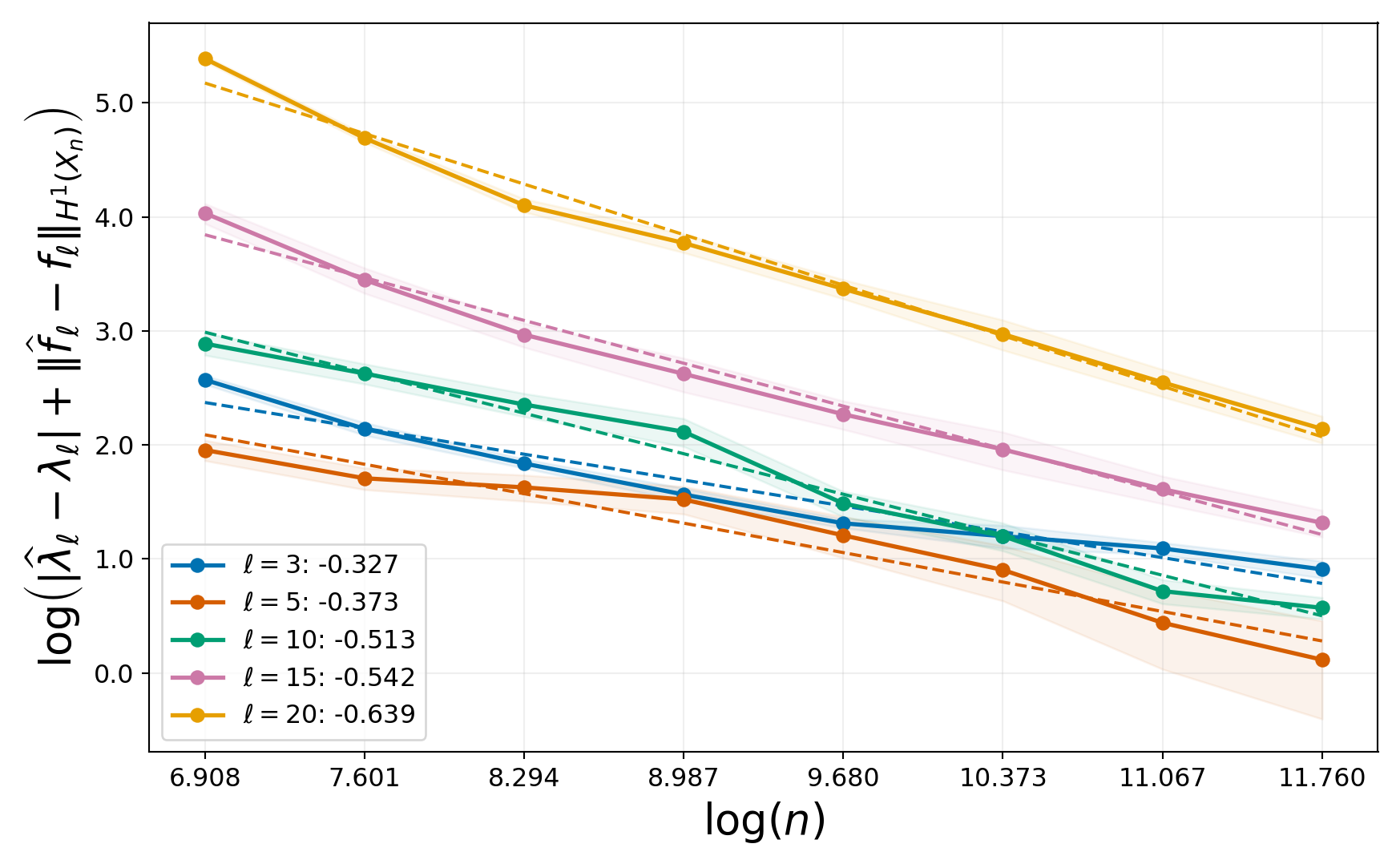}
        \label{fig:1d_pbc}
    \end{subfigure}
    \hfill
    \begin{subfigure}[b]{0.33\textwidth}
        \centering
        \includegraphics[width=\textwidth]{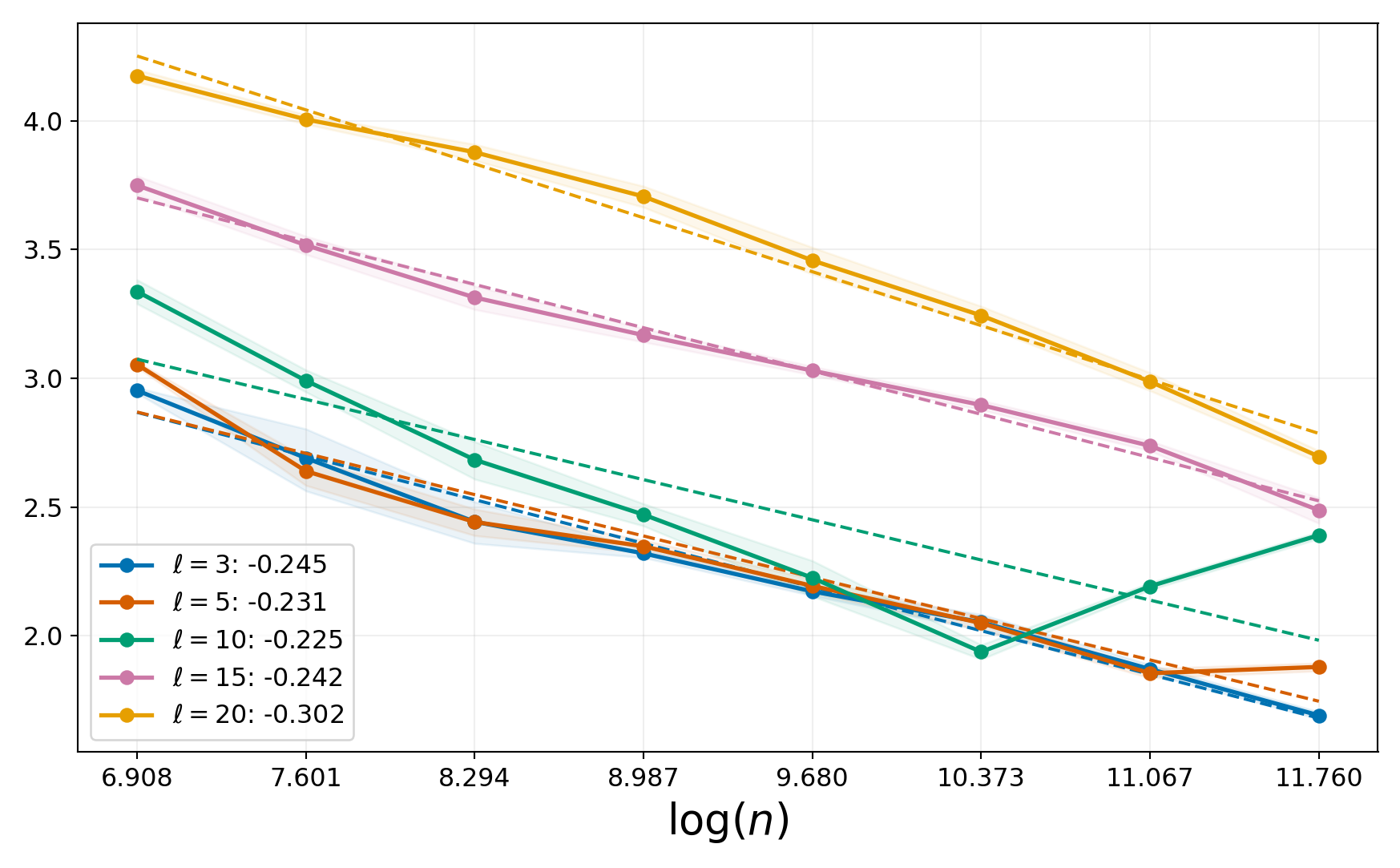}
        \label{fig:2d_pbc}
    \end{subfigure}%
    \hfill
    \begin{subfigure}[b]{0.33\textwidth}
        \centering
        \includegraphics[width=\textwidth]{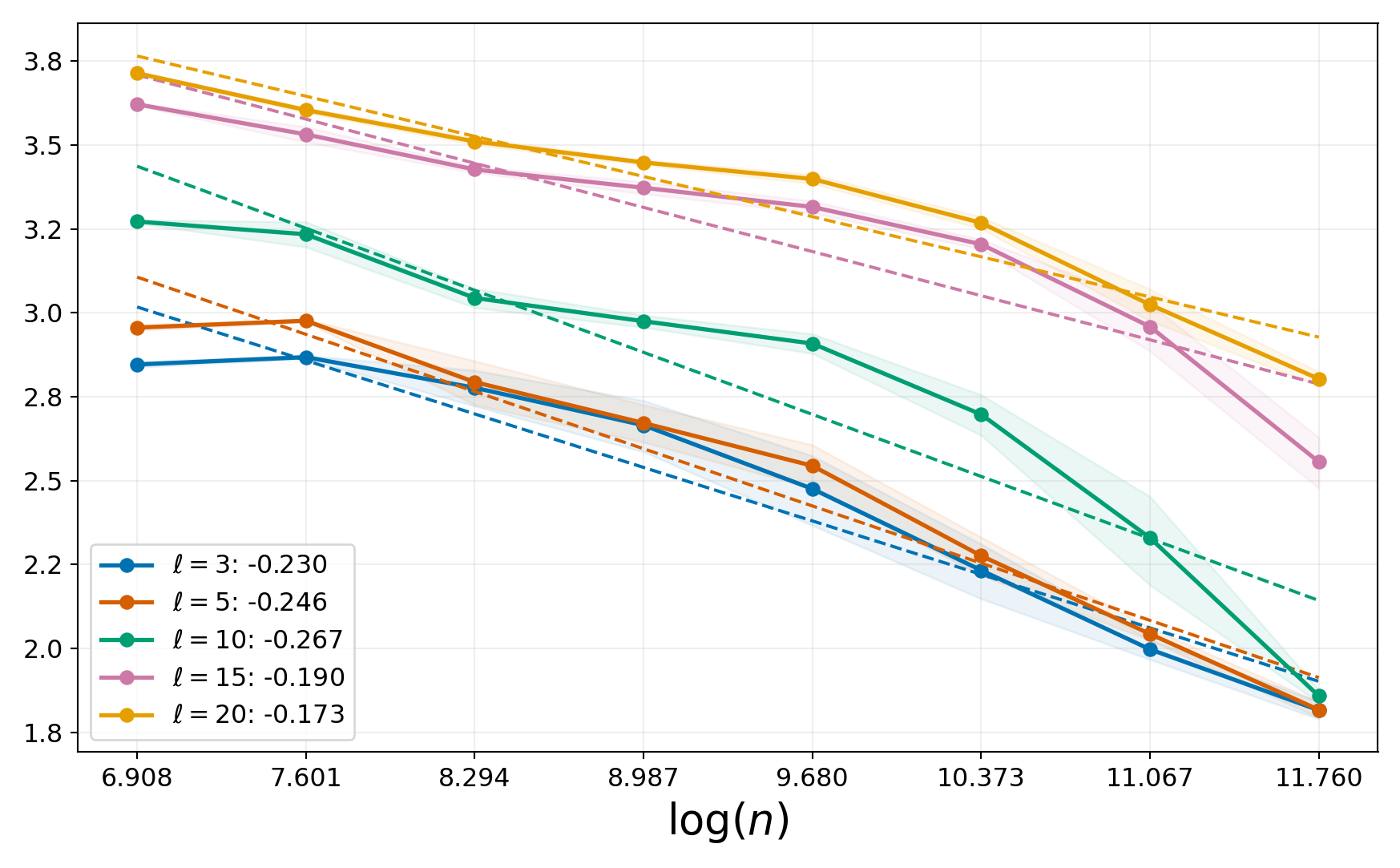}
        \label{fig:3d_pbc}
    \end{subfigure}   
    \caption{Realistic spectral convergence rate for \(d\in\{1,2,3\}\).}
    \label{fig:realistic_convergence}
\end{figure}



Next, we test the convergence rate in Theorem~\ref{thm:spectral convergence upper bound} where a practical graph based method is applied; the latent component labels are used only to generate the samples and to evaluate the results. We restrict to $n\in\{0.1,0.2,0.4,0.8,1.6,3.2,6.4,12.8\}\times10^{4}$. Given the sample size $n$, the graph scales are set by the schedule
\begin{equation}\label{eq:PBC_schedule}
t_n=\left[\frac{(\log n)/n}{(\log n_0)/n_0}\right]^{\beta},\qquad
\epsilon_+=C_\epsilon\,t_n,\qquad
\epsilon_-=c_\epsilon\,\epsilon_+,\qquad
r=C_r\,C_\epsilon\,t_n^{\,p},
\end{equation}
where $n_0$ is a reference sample size, with $n_0=2000$ for $d=1,2$ and $n_0=4000$
for $d=3$. The values used for each $d$ are listed in Table~\ref{tab:PBC_params}.
\begin{table}[t]\centering\small
\begin{tabular}{ccccccc}
\toprule
$d$ & $C_\epsilon$ & $\beta$ & $c_\epsilon$ & $C_r$ & $p$ & $\alpha$\\
\midrule
1 & 0.028 & $1/3$  & 0.6 & 0.50  & 3 & 3.5$\epsilon_+$ \\
2 & 0.11  & $1/8$  & 0.78 & 0.60 & 3 & 3.5$\epsilon_+$ \\
3 & 0.18  & $1/13$ & 0.8 & 0.72 & 3 & 3.8$\epsilon_+$ \\
\bottomrule
\end{tabular}
\caption{Parameters used in the PBC experiments.}
\label{tab:PBC_params}
\end{table}
Figure~\ref{fig:realistic_convergence} shows the experimental results for a practical spectral convergence rate on a log--log scale; as before, all errors are averaged over 50 independently generated data sets. The fitted rates are $n^{-0.479}$ for $d=1$, $n^{-0.249}$ for $d=2$, and $n^{-0.222}$ for $d=3$. Compared with Figure~\ref{fig:KDE_convergence}, the rates for $d=2$ and $d=3$ are slower than those of the oracle KDE, whereas the rate for $d=1$ is faster.
Recall that Theorem~\ref{thm:spectral convergence upper bound} guarantees, for $d\geq 2$, a convergence rate at least as fast as $n^{-1/(2.5d-1)}$, that is, $n^{-1/4}$ for $d=2$ and $n^{-2/13}\approx n^{-0.154}$ for $d=3$. The experimental results are consistent with this bound: the observed exponent for $d=2$ essentially matches the predicted one, and that for $d=3$ slightly exceeds it. This suggests that the rate established in Theorem~\ref{thm:spectral convergence upper bound} may not be sharp and could potentially be improved. 
Moreover, the theory makes no statement for $d=1$, where we observe a rate of $n^{-0.479}$. We view both observations as directions for future work: tightening the upper bound and extending the analysis to $d=1$.

The errors in Figure~\ref{fig:realistic_convergence} fluctuate considerably, which may be due to two factors. First, because of the limited scalability of the implementation, the sample sizes are relatively small compared with those used for the oracle KDE. Second, there is no general guidance for tuning the parameters. We fix $\beta$ and $p$ according to the theory so that the parameters have the scale prescribed by the graph construction, but the results still depend on the prefactors $C_\epsilon$, $c_\epsilon$, $C_r$ and \(\alpha\). A more reliable way to identify the critical local scale of the approximation graph would likely yield more accurate results. Overall, practical methods still trail the theory, and principled parameter-selection rules are needed.

\section{Conclusion \& Future Direction}
\label{sec:conclusion}
In this work, we establish the minimax lower bound for the spectral estimation of the tensorized Laplace--Beltrami operator in multi-manifold clustering. We further show that this lower bound is sharp by constructing a matching oracle estimator when the manifold components are known. Beyond this information-theoretic characterization, we establish an improved spectral convergence guarantee for the graph Laplacian to the tensorized Laplace--Beltrami operator in the more practical setting where manifold components are unknown. Our analysis yields sharper convergence rates for both eigenvalues and eigenfunctions than previously established results, advancing the theoretical understanding of spectral estimation in multi-manifold learning.

Several important questions remain open. A central challenge is to close the gap between the minimax rate achieved by the oracle estimator and the convergence rate established for the practical graph Laplacian. Whether this gap can be eliminated through sharper analysis or improved graph constructions, or instead reflects an intrinsic statistical cost of unknown manifold membership, remains to be investigated. Another important direction is to extend the minimax analysis beyond simple eigenpairs supported on individual manifold components to more general eigenspaces spanning multiple components, particularly when the manifolds intersect. Such extensions would provide a more complete characterization of the statistical limits of multi-manifold spectral estimation and further clarify the relationship between information-theoretic optimality and achievable algorithmic guarantees.




\clearpage
\bibliography{refs}
\clearpage
\appendix
\section{Auxiliary Statistical Tools}
\label{app:statistical tools}

\subsection{Fano's Method}
\label{sec:fano's method}

We briefly review Fano's method, adapting it to our specific notation and problem setting; see \cite{Wainwright2019} for a comprehensive treatment. 
\begin{proposition}[Fano's method; cf.\ \cite{Wainwright2019}]
\label{prop:fano}
Let $\mathcal P_M=\{\rho^{(1)},\dots,\rho^{(M)}\}\subset\mathcal P^{(s)}$, and let $\mu^{(a)}$ be the sampling measure induced by $\rho^{(a)}$. Suppose that, for some $\alpha,\beta>0$,
\begin{align}
d_{\mathrm{eig}}(\theta(\rho^{(a)}),\theta(\rho^{(b)})) &\ge 2\alpha \qquad \text{for all } a\neq b,
\label{eq:fano_sep}\\
\mathrm{KL}(\mu^{(a)}||\mu^{(b)}) &\le \beta.
\qquad \text{for all } a\neq b.
\label{eq:fano_kl}\end{align}
Then the minimax risk satisfies
\[\mathfrak M_n\bigl(\theta(\mathcal P_M); d_{\mathrm{eig}}\bigr) \ge\
\alpha\left(1-\frac{n\beta+\log 2}{\log M}\right).\]
\end{proposition}
Here $\mathrm{KL}(\cdot\,\|\,\cdot)$ denotes the Kullback--Leibler (KL) divergence between probability measures 
\[\mathrm{KL}(\mu^{(a)}||\mu^{(b)})=
\int_{\mathcal M} \log\!\left(\frac{d\mu^{(a)}}{d\mu^{(b)}}(x)\right)\,d\mu^{(a)}(x),\] provided $\mu^{(a)}$ is absolutely continuous with respect to $\mu^{(b)}$.
For $n$ i.i.d.\ samples,
\[\mathrm{KL}((\mu^{(a)})^{\otimes n}\|(\mu^{(b)})^{\otimes n})
= n\,\mathrm{KL}(\mu^{(a)}\|\mu^{(b)}).\]

\subsection{Kernel Density Estimation on Manifolds}
\label{sec:KDE on manifolds}
Kernel Density Estimation (KDE) is a non-parametric statistical method that applies kernel smoothing for probability density estimation. For a fixed point $x$, classical KDE in a $d$-dimensional Euclidean space has its kernel density estimator \[\hat\rho_h(x)=\frac{1}{nh^d}\sum_{i=1}^n K\left(\frac{x-x_i}{h}\right),\] with $K(\cdot)$ the kernel function and $h>0$ the smoothing bandwidth. The following result is a higher-order extension of
\citet[Theorem~3.1]{pelletier2005}, which establishes the case $s=2$.

\begin{lemma}[KDE upper bound on a known manifold]
\label{lem:KDE manifold}
Let $\M$ be a known compact $d$-dimensional Riemannian manifold without boundary satisfying Assumption~\ref{assump:geometric assumptions}, and
let $X_1,\dots,X_n$ be independently sampled from
$\rho\,d\mathrm{vol}_g$, where $\rho\in\mathcal P_{\M}^{(s)}$ for an integer $s\ge2$. Then, for all sufficiently small $h>0$, there exists a KDE estimator $\widehat\rho_h\in\mathcal P_{\M}^{(s)}$ such that
\[\mathbb E\|\widehat\rho_h-\rho\|_{L^2(\M)}^2\le
C\left(\frac{1}{nh^d}+h^{2s}\right).\]
Consequently, choosing $h\asymp n^{-1/(2s+d)}$ gives
\[\mathbb E\|\widehat\rho_h-\rho\|_{L^2(\M)}\le
C n^{-s/(2s+d)}.\]
Here $C$ depends only on the geometric and density-class parameters, and the kernel.\end{lemma}

\begin{proof}
Choose a compactly supported radial order-$s$ kernel
$\mathsf K_s:\mathbb R^d\to\mathbb R$ satisfying
\[\int_{\mathbb R^d}\mathsf K_s(z)\,dz=1, \qquad
\int_{\mathbb R^d}z^\alpha\mathsf K_s(z)\,dz=0,
\quad 1\le|\alpha|\le s-1,\]
and
\[\int_{\mathbb R^d}\|z\|^s|\mathsf K_s(z)|\,dz<\infty,
\qquad\int_{\mathbb R^d}\mathsf K_s(z)^2\,dz<\infty.\]
The radiality makes the kernel independent of the choice of orthonormal coordinates on each tangent space. For $x\in\M$, let $J_x(u)$ denote the normal-coordinate volume density \[d\mathrm{vol}_g(y)=d\mathrm{vol}_g(\exp_xu)=J_x(u)\,du,\qquad u\in T_x\M.\]
Choose $h_0>0$ below the injectivity radius under the geometric assumptions so that the exponential map
is a diffeomorphism on $B(0,h_0)\subset T_x\M$, uniformly in $x$, \[0<c_J\le J_x(u)\le C_J\qquad\text{for }x\in\M,\quad \|u\|\le h_0.\]
For $0<h\le h_0$, define the volume-corrected kernel
\[\mathsf K_{s,h}(x,y):=\begin{cases}\displaystyle
\frac{1}{h^dJ_x(\log_x y)}\mathsf K_s\!\left(\frac{\log_x y}{h}\right),& d_g(x,y)<h,\\[8pt]0,&d_g(x,y)\ge h,\end{cases}\]
and the preliminary estimator
\(\widetilde\rho_h(x):=\frac1n\sum_{i=1}^n\mathsf K_{s,h}(x,X_i).\)
We proceed by making the bias--variance decomposition \[\mathbb E\|\widetilde\rho_h-\rho\|_{L^2(\mathcal M)}^2=\int_{\mathcal M}
\left(\mathbb E\widetilde\rho_h(x)-\rho(x)\right)^2\,d\mathrm{vol}_g(x)+\int_{\mathcal M}\operatorname{Var}\bigl(\widetilde \rho_h(x)\bigr)\,d\mathrm{vol}_g(x).\]
We first control its bias. Change of variables with $u=\log_x(y)=hz$ for $z\in B(0,1)$ implies 
\[\mathbb E\widetilde\rho_h(x)=\int_{B_g(x,h)}\frac{\mathsf K_s\!\left(h^{-1}\log_x y\right)}{h^dJ_x(\log_x y)}
\rho(y)\,d\mathrm{vol}_g(y).\]
Using the normal-coordinate change of variables, the volume correction cancels the Jacobian:
\[\mathbb E\widetilde\rho_h(x)=\int_{B(0,h)}
\frac1{h^d}\mathsf K_s\!\left(\frac{u}{h}\right)\rho(\exp_xu)\,du=\int_{B(0,1)}
\mathsf K_s(z)\rho(\exp_x(hz))\,dz.\]
The covariant Taylor expansion of $\rho$ along \(t\mapsto\exp_x(thz)\) gives
\[\rho(\exp_x(hz))=\rho(x)+\sum_{\nu=1}^{s-1}\frac{h^\nu}{\nu!}(\nabla_g^\nu\rho)_x[z^{\otimes\nu}]+R_{s,x}(hz),\]
where the density-class assumption implies
\(|R_{s,x}(hz)|\le C h^s\|z\|^s\) uniformly in $x$.
The order-$s$ moment conditions eliminate all terms
of order one through $s-1$. Therefore,
\[\mathbb E\widetilde\rho_h(x)-\rho(x)=\int_{B(0,1)}\mathsf K_s(z)R_{s,x}(hz)\,dz,\] and hence
$\sup_{x\in\M}\left|\mathbb E\widetilde\rho_h(x)-\rho(x)
\right|\le Ch^s$. Since $\M$ has uniformly bounded volume,
\begin{equation}\label{eq:KDE-integrated-bias}
\int_{\M}\left|\mathbb E\widetilde\rho_h(x)-\rho(x)
\right|^2d\mathrm{vol}_g(x)\le Ch^{2s}.\end{equation}

We next estimate the variance. Independence gives
\[\operatorname{Var}\bigl(\widetilde\rho_h(x)\bigr)=
\frac1n\operatorname{Var}\bigl(\mathsf K_{s,h}(x,X_1)\bigr)
\le\frac1n\mathbb E\bigl[\mathsf K_{s,h}(x,X_1)^2\bigr].\]
Applying the same normal-coordinate change of variables,
\begin{align*}
\operatorname{Var}\bigl(\widetilde\rho_h(x)\bigr)
\le\frac{1}{nh^{2d}}\int_{B(0,h)}\mathsf K_s\!\left(\frac{u}{h}\right)^2\frac{\rho(\exp_xu)}{J_x(u)}\,du=\frac{1}{nh^d}
\int_{B(0,1)}\mathsf K_s(z)^2\frac{\rho(\exp_x(hz))}{J_x(hz)}\,dz\le\frac{C}{nh^d},\end{align*}
where the last inequality uses the uniform bounds on $\rho$, $J_x$, $\mathsf K_s$. Integrating over $\M$ yields
\begin{equation}\label{eq:KDE-integrated-variance}
\int_{\M}\operatorname{Var}\bigl(\widetilde\rho_h(x)\bigr)
\,d\mathrm{vol}_g(x)\le\frac{C}{nh^d}.\end{equation}

Combining \eqref{eq:KDE-integrated-bias} and
\eqref{eq:KDE-integrated-variance} with the bias--variance
decomposition gives
\[\mathbb E\|\widetilde\rho_h-\rho\|_{L^2(\M)}^2\le
C\left(h^{2s}+\frac{1}{nh^d}\right).\]
For $s>2$, the order-$s$ kernel may be signed, so
$\widetilde\rho_h$ need not itself belong to the admissible density class. Choose an approximate $L^2$-projection
$\widehat\rho_h\in\mathcal P_{\M}^{(s)}$ satisfying
\[\|\widehat\rho_h-\widetilde\rho_h\|_{L^2(\M)}\le
\inf_{\varrho\in\mathcal P_{\M}^{(s)}}\|\varrho-\widetilde\rho_h\|_{L^2(\M)}+n^{-1}.\]
Since $\rho\in\mathcal P_{\M}^{(s)}$,
\(\|\widehat\rho_h-\rho\|_{L^2(\M)}\le2\|\widetilde\rho_h-\rho\|_{L^2(\M)}+n^{-1}.\)
It follows that
\[\mathbb E\|\widehat\rho_h-\rho\|_{L^2(\M)}^2
\le C\left(h^{2s}+\frac{1}{nh^d}\right).\]
Balancing the two terms gives
$h\asymp n^{-1/(2s+d)}$, which leads to the first claimed identity. Jensen's inequality then yields the second.
\end{proof}

\subsection{Bernstein Concentration Inequality for the U-statistic}\label{sec:bernstein u-statistic}
We first recall the classic Bernstein's inequality for i.i.d. real-valued random variables. 
\begin{lemma}[Bernstein's inequality]
Let $X_1,X_2,\dots,X_n$ be independent, identically distributed (i.i.d.) real-valued random variables for which $|X_i|\leq M$ and $\var(X_i)\leq\sigma^2$. Then 
\[\mathbb{P}\left[\left|\frac{1}{n}\sum_{i=1}^n X_i-\mathbb{E}\left[\frac{1}{n}\sum_{i=1}^nX_i\right]\right|\geq t\right]\leq2\exp\left(-\frac{cnt^2}{\sigma^2+Mt}\right),\quad \forall t> 0. \]
\end{lemma}
There is a Bernstein-type concentration inequality for the U (unbiased)-statistic 
\[U=\frac{1}{n}\sum_{j=1}^Nw_j\sum_{y\in\mathcal{M}_j}\eta_{\epsilon_+}(d(x,y))\mathcal{K}(x,y)\]
for a kernel function $\mathcal{K}:\mathcal{M}_j\times\mathcal{M}_j\rightarrow\mathbb{R}$, which we state next lemma.  
\begin{lemma}[Bernstein's inequality for the U-statistic]
\label{lem:bernstein}
Given a function $\Psi:\mathcal{M}_j\times\mathcal{M}_j\rightarrow\mathbb{R}$ which satisfies
\begin{equation}\label{eq:bernstein second moment bound}
\max_{j\in[N]}\int_{\mathcal{M}_j}\left(\int_{B_1(0)\subseteq T_x\mathcal{M}_j}\eta(\|v\|)\Psi(x,\exp_x(\epsilon_+v))\mathrm{d}v\right)^2d\mathrm{vol}_g(x)\leq B_{\Psi}<\infty,
\end{equation} and \begin{equation}\label{eq:bernstein uniform bound}\max_{j\in[N]}\sup_{x\in\mathcal{M}_j}\sup_{v\in B_1(0)\subseteq T_x\mathcal{M}_j}|\Psi(x,\exp_x(\epsilon_+v))|\leq C_{\Psi}<\infty.
\end{equation} Then the corresponding U-statistic has the Bernstein's concentration bound
\[\mathbb{P}\left[|U-\mathbb{E}(U)|>t\right]\leq2\exp\left(-\frac{cn\epsilon_+^dt^2}{B_{\Psi}+tC_{\Psi}}\right).\]
\end{lemma}


\section{Technical Proofs of Section \ref{sec:minimax lower bound}} 
\subsection{Construction of Perturbation Templates}\label{sec:construction of perturbation templates}

We verify the support, normalization, and derivative properties stated in \eqref{eq:a_i properties}. Since
$\phi_i\in C_c^\infty(B(0,1/4))$, the rescaled profile
$\phi_i(m\,\exp_{p_i}^{-1}(x))$ vanishes near
$\partial B_g(p_i,r/4)$. Its extension by zero is therefore regular up to the order permitted by $\rho_j$. Moreover,
\[\operatorname{supp}(\widetilde a_i)\subset B_g(p_i,r/4)
\subset B_g(p_i,r/2)\subset\Omega_i.\]
Because the cells $\{\Omega_i\}_{i=1}^K$ are pairwise disjoint up to their boundaries, the raw templates $\{\widetilde a_i\}_{i=1}^K$ have pairwise disjoint supports. To establish the derivative scaling, the chain rule in normal coordinates gives
\begin{equation}\label{eq:raw-template-derivative-bound}
\|\nabla_g^\nu\widetilde a_i\|_{L^\infty(\M_j)}
\le Cm^\nu,\qquad\nu=0,\dots,s,\end{equation}
where the factor $\rho_j^{-1}$ from differentiation is uniformly bounded by Assumption~\ref{assump:density class}.

It remains to impose the exact weighted mean zero condition. For \(u=\exp_{p_i}^{-1}(x)\in B(0,r/4)\), write
$J_i(u)=\sqrt{\det g_i(u)}$ for the normal-coordinate volume density. Then
\[\int_{\M_j}\widetilde a_i\,d\mu_j=\int_{B(0,r/4)} \phi_i(mu)J_i(u)\,du=\int_{B(0,r/4)}\phi_i(mu)\bigl(J_i(u)-1\bigr)\,du.\]
The second equality follows from $\int_{\mathbb R^{d_j}}\phi_i=0$. Since $|J_i(u)-1|\le C|u|^2$ uniformly in $i$, we can bound the error by applying the change of variables $v = mu$. This yields $du = m^{-d_j}dv$ and $|u|^2 = m^{-2}|v|^2$, which gives
\begin{equation}
\label{eq:tilde_ai_mean_estimate}
\left|\int_{\M_j}\widetilde a_i\,d\mu_j\right|\le C\int_{\mathbb{R}^{d_j}}|\phi_i(mu)|\,|u|^2\,du = Cm^{-(d_j+2)}\int_{B(0,1/4)} |\phi_i(v)|\,|v|^2\,dv \le C'm^{-(d_j+2)}.
\end{equation}
Choose a nonnegative function $\chi_i\in C_c^\infty\!\left(B_g(p_i,r/2)\setminus\overline{B_g(p_i,r/4)}
\right)$ such that
\[\int_{\M_j}\chi_i\,d\mu_j\asymp m^{-d_j},\qquad
\|\nabla_g^\nu\chi_i\|_{L^\infty(\M_j)}\le Cm^\nu, \quad
\nu=0,\dots,s.\] Such functions are obtained by rescaling a fixed annular bump in normal coordinates. Define
\[b_i:=\frac{\int_{\M_j}\widetilde a_i\,d\mu_j}{  \int_{\M_j}\chi_i\,d\mu_j},\quad a_i:=\widetilde a_i-b_i\chi_i \quad \implies \quad  \int_{\M_j}a_i\,d\mu_j=0,\qquad|b_i|\le Cm^{-2}.\] Consequently,
\(\|\nabla_g^\nu(b_i\chi_i)\|_{L^\infty(\M_j)}\le Cm^{\nu-2}\) for \(\nu=0,\dots,s\).
Combining with
\eqref{eq:raw-template-derivative-bound} proves
\[\operatorname{supp}(a_i)\subset B_g(p_i,r/2)\subset
\Omega_i,\qquad\|\nabla_g^\nu a_i\|_{L^\infty(\M_j)}
\le Cm^\nu,\quad\nu=0,\dots,s.\]

\subsection{Eigenpair stability under small density 
perturbation} 
We introduce ODE-like quantities analogous to the static ones in the main text for analysis. Let $\rho(t)$ be a $C^1$ curve in the admissible density class on a compact manifold $\M$, and define
\[d\mu_t=\rho(t)\,d\mathrm{vol}_g,\qquad\Delta_{\rho(t)}=-\frac{1}{\rho(t)}\operatorname{div}_g\!\left(
\rho(t)^2\nabla_g(\cdot)\right).\]
Suppose that $\lambda_\ell(t)$ is a simple and isolated eigenvalue with eigengap
\[\gamma_\ell:=\inf_{|t|<t_*}\min_{k\neq\ell}|\lambda_k(t)-\lambda_\ell(t)|>0.\]
By standard perturbation theory for self-adjoint operators with compact resolvent \cite{Kato1995}, the eigenpair admits a $C^1$ continuation, normalized by
\(\|f_\ell(t)\|_{L^2(\mu_t)}=1.\) We write
\[\rho=\rho(0),\quad \dot\rho=\partial_t\rho(t)|_{t=0},\quad
\lambda_\ell=\lambda_\ell(0),\quad
\dot\lambda_\ell=\partial_t\lambda_\ell(t)|_{t=0},
\quad f_\ell=f_\ell(0),\quad\dot f_\ell=\partial_tf_\ell(t)|_{t=0}.\]

\begin{proposition}[Infinitesimal stability of a simple eigenpair]\label{prop:eigenpair under perturbation} There exists a constant $C_\ell<\infty$, depending only on the fixed
geometric and density-class parameters, the eigenvalue index $\ell$, and the eigengap $\gamma_\ell$, such that \begin{equation}\label{eq:lambda_dot_bound}
|\dot \lambda_{\ell}|\leq C_{\ell}\, \lambda_{\ell}\,\|\dot\rho\|_{L^2(\mathcal M)}, \end{equation} and \begin{equation}\label{eq:fdot_H1_bound}
\|\dot f_{\ell}\|_{H^1(\mathcal M)}= \|\dot f_{\ell}\|_{L^2(\mathcal M)}+\|\nabla_g\dot f_{\ell}\|_{L^2(\mathcal M)}
\ \le\ C_{\ell} \frac{\lambda_{\ell}}{\gamma_\ell}\,\|\dot\rho\|_{L^2(\mathcal M)}.
\end{equation}\end{proposition}

\begin{proof} 
Introduce the parameter-dependent bilinear forms
\[a_t(u,v):=\int_{\M}\rho(t)^2 \langle\nabla_gu, \nabla_gv\rangle_g\,d\mathrm{vol}_g,\qquad b_t(u,v):=
\int_{\M}\rho(t)uv\,d\mathrm{vol}_g.\]
The eigenvalue equation and normalization condition then take the weak form
\begin{equation}\label{eq:weak-eigenvalue-equation}a_t(f_\ell(t),v)=\lambda_\ell(t)b_t(f_\ell(t),v),
\qquad b_t(f_\ell(t),f_\ell(t))=1.
\end{equation}
Differentiating the first identity with respect to $t$ at $t=0$ and testing with $v=f_\ell$ yields
\[\dot\lambda_\ell=2\int_{\M}\rho\dot\rho \|\nabla_gf_\ell\|_g^2\,d\mathrm{vol}_g-
\lambda_\ell\int_{\M}\dot\rho f_\ell^2\,d\mathrm{vol}_g.\]
where the terms involving $\dot f_\ell$ cancel out due to the unperturbed eigenvalue equation $a_0(\dot f_\ell, f_\ell) = \lambda_\ell b_0(\dot f_\ell, f_\ell)$. Appplying the eigenfunction bounds from~\eqref{eq:L-infty-regularity}, Hölder's inequality, and the finite volume of $\M$, we obtain
\[|\dot \lambda_\ell|\le C_\ell\Bigl(\|\nabla_g f_\ell\|_{L^\infty(\mathcal M)}^2
+\lambda_\ell\|f_\ell\|_{L^\infty(\mathcal M)}^2\Bigr)\|\dot\rho\|_{L^2(\mathcal M)}\le C_\ell\lambda_\ell\|\dot\rho\|_{L^2(\mathcal M)}.\]

Now set the unperturbed bilinear form
\(B(u,v):=a_0(u,v)-\lambda_\ell b_0(u,v).\)
Differentiating the weak eigenvalue equation \eqref{eq:weak-eigenvalue-equation} at $t=0$ for an arbitrary $v\in H^1(\M)$ yields the variational equation $B(\dot f_\ell, v) = F(v)$, where the forcing functional $F\in H^{-1}(\M)$ is defined by
\[F(v):=-2\int_{\M}\rho\dot\rho
\langle\nabla_gf_\ell,\nabla_gv\rangle_g\,d\mathrm{vol}_g
+\dot\lambda_\ell b_0(f_\ell,v)+\lambda_\ell
\int_{\M}\dot\rho f_\ell v\,d\mathrm{vol}_g.
\label{eq:linearized-weak-eigenproblem}\]
Applying Hölder's inequality alongside the uniform bounds for $f_\ell$ and the estimate \eqref{eq:lambda_dot_bound}, we readily bound the operator norm of $F$:
\begin{equation}\label{eq:F_Hminus1}
\|F\|_{H^{-1}(\M)} = \sup_{\|v\|_{H^1}=1} |F(v)| \le C_\ell\lambda_\ell\|\dot\rho\|_{L^2(\M)}.
\end{equation}

To estimate $\dot f_\ell$, we introduce the $b_0$-orthogonal decomposition
\[\dot f_\ell=\alpha f_\ell+\dot f_{\ell,\perp}, \qquad \text{where}\quad b_0(\dot f_{\ell,\perp}, f_\ell)=0.\]
Differentiating the normalization identity \eqref{eq:weak-eigenvalue-equation} at $t=0$ fixes the projection coefficient $\alpha$:
\[2b_0(\dot f_\ell, f_\ell)+\int_{\M}\dot\rho f_\ell^2\,d\mathrm{vol}_g=0 \implies 2\alpha+\int_{\M}\dot\rho f_\ell^2\,d\mathrm{vol}_g=0,\]
which immediately implies $|\alpha|\le C_\ell\|\dot\rho\|_{L^2(\M)}$. Since $f_\ell$ lies in the kernel of $B$, we have $B(\dot f_{\ell,\perp}, v) = B(\dot f_\ell, v) - \alpha B(f_\ell, v) = F(v)$ for all $v$. We now restrict this problem to the invariant subspace $V = \{u\in H^1(\M): b_0(u,f_\ell)=0\}$. Since $\lambda_\ell$ is simple, the spectral gap condition, combined with standard elliptic regularity, ensures that $B$ satisfies the resolvent stability estimate on $V$. Consequently, 
\[\|\dot f_{\ell,\perp}\|_{H^1(\M)}\le C_\ell\gamma_\ell^{-1}\|F\|_{H^{-1}(\M)}\le C_\ell\frac{\lambda_\ell}{\gamma_\ell}\|\dot\rho\|_{L^2(\M)}.\]
Finally, summing the orthogonal components yields \eqref{eq:fdot_H1_bound}:
\[\|\dot f_\ell\|_{H^1(\M)}\le |\alpha|\|f_\ell\|_{H^1(\M)}+\|\dot f_{\ell,\perp}\|_{H^1(\M)}\le C_\ell\frac{\lambda_\ell}{\gamma_\ell}\|\dot\rho\|_{L^2(\M)}.\]

\end{proof}

\subsection{Local Coercivity and Geometric Estimates}
\begin{proposition}[Uniform Jacobi-field comparison]\label{eq:jacobi-field-comparison}
Let $e_i(x) \in T_x\mathcal M_j$ denote the parallel transport of a unit vector $e_i(p_i) \in T_{p_i}\mathcal M_j$ along the radial geodesic $\gamma(t) := \exp_{p_i}(tu)$.
Then there exist constants $C_{\exp}>0$, independent of the patch index $i$, such that for \(u\in T_{p_i}\mathcal M_j\) with \(|u|\leq r/4\), the differential of the exponential map satisfies  
\[\left\| d(\exp_{p_i})_u \bigl(e_i(p_i)\bigr) - e_i(x) \right\|_g \leq C_{\exp}|u|^2.\]
\end{proposition}
\begin{proof}
Consider the variation through radial geodesics and its associated variation vector field along \(\gamma\),
\[\Gamma(s,t):=\exp_{p_i}\bigl(t(u+s e_i(p_i))\bigr), \qquad \text{with} \qquad J(t):=\left.\frac{\partial}{\partial s}\Gamma(s,t)\right|_{s=0}.\] By definition, \(J(t)\) is a Jacobi field satisfying the initial conditions
\(J(0)=0\), \(\nabla_{\dot\gamma(0)}J(0)=e_i(p_i)\). Moreover, evaluating at \(t=1
\) yields \(J(1) = d(\exp_{p_i})_u \bigl(e_i(p_i)\bigr)\).
Along the geodesic \(\gamma(t)\), the Jacobi field satisfies the second-order linear ODE:
\[\nabla_{\dot\gamma(t)}^2 J(t) + R(J(t),\dot\gamma(t))\dot\gamma(t) = 0.\] The uniform curvature bound implies that there exists a constant $C_0$ such that 
\[\left\|R\bigl(J(t),\dot\gamma(t)\bigr)\dot\gamma(t) \right\|_g \leq C_0 \|J(t)\|_g |u|^2.\]
Let $E(t)$ be the parallel vector field along $\gamma$ determined by $E(0)=e_i(p_i)$, so that $E(1)=e_i(x)$. The flat comparison field $tE(t)$ shares the same initial conditions as $J(t)$, namely \[tE(t)\big|_{t=0}=0, \qquad \nabla_{\dot\gamma(0)}(tE(t)) = e_i(p_i).\]
Define the error field \(K(t):=J(t)-tE(t)\). Since $E(t)$ is parallel ($\nabla_{\dot\gamma}E \equiv 0$), we have $\nabla_{\dot\gamma(t)}^2 K(t) = \nabla_{\dot\gamma(t)}^2 J(t)$. The Jacobi equation then yields
\[\nabla_{\dot\gamma(t)}^2K(t) = -R\bigl(J(t),\dot\gamma(t)\bigr)\dot\gamma(t),\]
with initial conditions $K(0)=0$ and $\nabla_{\dot\gamma(0)}K(0)=0$. Integrating this second derivative twice with respect to $t \in [0,1]$ yields
\[K(1) = \int_0^1 \int_0^\tau \nabla_{\dot\gamma(s)}^2 K(s) \, ds \, d\tau.\]
By standard Jacobi field estimates, $\|J(t)\|_g$ is uniformly bounded for $t\in[0,1]$. Taking the norm of the double integral gives, for some uniform constant $C_{\exp}$,
\[\|K(1)\|_g \leq \int_0^1 \int_0^\tau C_0 \|J(s)\|_g |u|^2 \, ds \, d\tau \leq C_{\exp} |u|^2.\]
Since \(K(1) = J(1) - E(1) = d(\exp_{p_i})_u \bigl(e_i(p_i)\bigr) - e_i(x)\), the claim is proved.
\end{proof}

\begin{lemma}[Local $L^1$ coercivity]
\label{lem:local L1 coercivity}
Fix an active patch $i\in\mathcal D$ and let
$h\in C^2(\mathcal N_i^\star)$,
$\mathfrak r\in L^1(\mathcal N_i^\star)$, and $A>0$. Suppose that
\begin{equation}
\label{eq:patch_PDE_no_source}
\operatorname{div}_g(\rho_j^2\nabla_gh)+V(x)h+\mathfrak r\ge A\,m^{1-s}\qquad\text{on }\mathcal N_i^\star, \end{equation}
where \(0\le V(x)\le\Lambda\) on \(\mathcal N_i^\star. \)
Then there exist constants $c,C>0$, depending only on the density-class and geometric bounds and on $\tau_0$, such that
\begin{equation}\label{eq:patch_conclusion_no_source}
\frac{\Lambda}{m}\int_{\mathcal N_i^\star}|h|\,d\mathrm{vol}_g
+\int_{\mathcal N_i^\star} \|\nabla_gh\|_g\,d\mathrm{vol}_g
\ge cA\,m^{-(d_j+s)}-\frac{C}{m}
\int_{\mathcal N_i^\star}|\mathfrak r|\,d\mathrm{vol}_g.
\end{equation}\end{lemma}
\begin{proof}
Let \(z_i^*\in\{\pm u_0/m\}\) denote the center of \(\mathcal N_i^*\). Choose a fixed cut-off function $\chi_0\in C_c^\infty(B(0,\tau_0))$ satisfying
\(0\le\chi_0\le1\) and \(\chi_0\ge\frac12\) on \(B(0,\tau_0/2)\).
Define $\chi \in C_c^\infty(\mathcal N_i^\star)$ by
\[\chi(x):=\chi_0\!\left(m\bigl(\psi_i^{-1}(x)-z_i^\star\bigr)\right), \qquad x\in\mathcal N_i^\star.\]
By construction, $\chi$ satisfies $0\le\chi\le1$, $\|\nabla_g\chi\|_g\le Cm$, and is supported in $\mathcal N_i^\star$ with $\chi \ge \frac12$ on $\psi_i\left(B\left(z_i^\star, \frac{\tau_0}{2m}\right)\right)$. The uniform bi-Lipschitz property of normal coordinates at scale $m^{-1}$ implies
\[\int_{\mathcal N_i^\star}\chi\,d\mathrm{vol}_g \ge \frac{1}{2}\operatorname{vol}_g\left(\psi_i\left(B\left(z_i^\star, \tfrac{\tau_0}{2m}\right)\right)\right) \ge c_0 m^{-d_j}.\]
Multiplying \eqref{eq:patch_PDE_no_source} by $\chi$ and integrating over $\mathcal N_i^\star$ yields
\begin{align}\label{eq:local-coercivity-lower}
\int_{\mathcal N_i^\star}\chi\left[
\operatorname{div}_g(\rho_j^2\nabla_gh)+Vh+\mathfrak r
\right]\,d\mathrm{vol}_g &\ge A\,m^{1-s}
\int_{\mathcal N_i^\star}\chi\,d\mathrm{vol}_g\ge
c_0A\,m^{-(d_j+s-1)}.\end{align}
On the other hand, since $0\le V\le\Lambda$ and $0\le\chi\le1$,
{\small\begin{align}\label{eq:local-coercivity-upper}
\int_{\mathcal N_i^\star}\chi
\left[\operatorname{div}_g(\rho_j^2\nabla_gh)
+Vh+\mathfrak r\right]\,d\mathrm{vol}_g\le
\int_{\mathcal N_i^\star} \chi\operatorname{div}_g(\rho_j^2\nabla_gh)
\,d\mathrm{vol}_g+\Lambda
\int_{\mathcal N_i^\star}|h|\,d\mathrm{vol}_g
+\int_{\mathcal N_i^\star}|\mathfrak r|\,d\mathrm{vol}_g.
\end{align}}
Because $\chi$ is compactly supported in $\mathcal N_i^\star$, integration by parts gives
\[\int_{\mathcal N_i^\star} \chi\operatorname{div}_g(\rho_j^2\nabla_gh)\,d\mathrm{vol}_g=-\int_{\mathcal N_i^\star}\rho_j^2
\langle\nabla_g\chi,\nabla_gh\rangle_g\,d\mathrm{vol}_g.
\]
Using $\rho_j\le\rho_{\max}$ and $\|\nabla_g\chi\|_g\le Cm$, we obtain
\[\left|\int_{\mathcal N_i^\star} \chi\operatorname{div}_g(\rho_j^2\nabla_gh)
\,d\mathrm{vol}_g\right|\le Cm \int_{\mathcal N_i^\star}
\|\nabla_gh\|_g\,d\mathrm{vol}_g.\]
Combining \eqref{eq:local-coercivity-lower} and
\eqref{eq:local-coercivity-upper} therefore gives
\begin{align*}
c_0A\,m^{-(d_j+s-1)}\le Cm\int_{\mathcal N_i^\star}
\|\nabla_gh\|_g\,d\mathrm{vol}_g+\Lambda
\int_{\mathcal N_i^\star}|h|\,d\mathrm{vol}_g
+\int_{\mathcal N_i^\star}|\mathfrak r|\,d\mathrm{vol}_g.
\end{align*}Dividing by $m$ and adjusting the constants proves \eqref{eq:patch_conclusion_no_source}.\end{proof}

\section{Technical Proofs of Section \ref{sec:spectral upper bound}} 

\subsection{Annular proximity graph with angle constraints}
\label{ssec:annular_proximity_graph}
To make the paper self-contained, we give a high-level summary of the annular proximity graph with angle constraints introduced in \cite{Li2022} and the related theorems. 
\begin{definition}\label{def:annular_proximity_graph}
Let \(\alpha\in(0,\pi/2)\) and \(r>0\). We say \(x_{i_1}\) and \(x_{i_m}\) is \((\alpha, r)\)-connected if there exists a path between the two points such that for all \(j=1,2\dots, m-1\), 
\[\angle(x_{i_m}-x_{i_1}, x_{i_{j+1}}-x_{i_j})<\alpha, \qquad \|x_{i_{j+1}}-x_{i_j}\|<r.\] Given \(0<\epsilon_-<C\epsilon_+\), construct a graph with weights \(w(x,y)=1\) only if \(\epsilon_-<\|x-y\|<\epsilon_+\) and \(x,y\) is \((\alpha,r)\)-connected; otherwise \(w(x,y)=0\). Here \(C>0\) is a constant strictly less than \(1\).  The resulting graph is called the annular proximity graph with angle constraints.\end{definition}
Several theorems are established following this graph construction. The first concerns strong within-manifold connectivity, while the second justifies the sparsity of between-manifold connections. Together, these yield the eigenpair convergence claimed in that work. Finally, a supplementary lemma characterizes the exact conditions for pairwise path connectivity. 
\begin{theorem}[Theorem 14; \cite{Li2022}]\label{thm:intense_within_manifold} Let \(\alpha\in[0,\pi/4)\) and \(r\leq C\epsilon_+\). Then for \(j\in[N]\), there exists \(C_j>0\), depending on the  geometric parameters of \(\mathcal M_j\), such that any points \(x,y\in\mathcal X_{n,j}\) satisfying \(\|x-y\|<\epsilon_+\) is \((\alpha,r)\)-connected with probability at least 
\[1-C_{\mathrm{wth}}(n), \qquad \text{with} \qquad C_{\mathrm{wth}}(n):=\frac{C_jn^2\epsilon_+}{r}\exp\left(-C_jn_jr^{d_j}(\tan\alpha)^{{d_j}-1}\right).\]
\end{theorem}
\begin{theorem}[Theorem 15; \cite{Li2022}]\label{thm:sparse_between_manifolds} Suppose \(\alpha<\arcsin(C_{j,k}r/\epsilon_+)\). Choose \(\epsilon_-\geq c\epsilon_+\) for some \(0<c<1\). Then with probability at least 
\[1-C_{j,k}n\exp\left(-C_{j,k}n\max\{q^{d_j-d_{jk}},q^{d_k-d_{jk}},\epsilon_+^{d_j},\epsilon_+^{d_k}\}\right),\] the number of connections \(\mathcal N_{jk}\) between \(\mathcal M_j\) and \(\mathcal M_k\) satisfies
\[\mathcal N_{jk}\leq C_{j,k}n^2\max\{q^{d_j-d_{jk}}\epsilon_+^{d_k},q^{d_k-d_{jk}}\epsilon_+^{d_j}\}, \qquad q:=\left(\frac{r^2}{\epsilon_+}+r\epsilon_+\right).\]  
\end{theorem}
\begin{theorem}[Theorem 10; \cite{Li2022}]\label{thm:early_eigenpair_convergence} Suppose the conditions for Theorems~\ref{thm:intense_within_manifold} and \ref{thm:sparse_between_manifolds} hold. Then with probability at least \(1-Cn\exp(-cn\epsilon_+^d)-C_{\mathrm{wth}}(n)\),
\[|\widehat\lambda_\ell-\sigma_\eta\lambda_\ell|\leq C_\ell\lambda_\ell, \qquad \|\widehat f_\ell-f_\ell\|_{L^2(\mathcal X_n)}\leq C_\ell.\]\end{theorem}
\begin{lemma}[Lemma 20; \cite{Li2022}]\label{lem:four_regions}
Suppose \(x\in\mathcal X_{n,j}\) and \(y\in\mathcal X_{n,k}\) satisfy \(\epsilon_-<\|x-y\|<\epsilon_+\) with \(\epsilon_-\geq c\epsilon_+\) for some \(0<c<1\). Let \(\alpha>0\) satisfy 
\[\sin(\alpha)< Cr/\epsilon_+.\] Assume, in addition, that \(\mathrm{dist}(x,\mathcal M_k)>Cq\) and \(\mathrm{dist}(y,\mathcal M_j)>r\). Then there does not exist an \((\alpha,r)\)-constrained path connecting \(x\) and \(y\).     
\end{lemma}
Our construction throughout the section chooses \(r=\epsilon_+^3\) so that \(q\asymp\epsilon_+^4\).

\subsection{Within-manifold and Between-manifold Laplacian Estimates}\label{ssec:within_between_estimates}
The analysis in this section relies on the multiscale Poincaré
inequality of \cite{trillos2025}. For completeness, we record the relevant quantities in our notation.

For each $j$, fix a cover of $\mathcal M_j$ by $N_{\nu,j}$
radius-$2\nu$ balls and choose the multiscale depth $m_j$ so that $3^{-m_j}\asymp\epsilon_+$. For $p\in[m_j]$, define the level-$p$ cell size by $s_{p,j}:=3^{p-m_j}$. Thus
$s_{1,j}\asymp\epsilon_+$, $s_{m_j,j}=1$, and
$m_j\asymp\log(\epsilon_+^{-1})$. For each chart $a$ and level $p$, let $\Upsilon_{p,a,j}^{m_j}$ denote the cell-index set,
$\square_{p,a,j,\upsilon}^{m_j}\subset\mathcal M_j$ the geometric cell indexed by $\upsilon$, and $Q_{p,a,j,\upsilon}^{m_j}:=\square_{p,a,j,\upsilon}^{m_j}\cap\mathcal X_{n,j}$ its sample set. On the relevant high-probability occupancy event, uniformly over $a,p,\upsilon$,
\[|\Upsilon_{p,a,j}^{m_j}|\asymp s_{p,j}^{-d_j}, \qquad 
|Q_{p,a,j,\upsilon}^{m_j}|\asymp n_j s_{p,j}^{d_j}.\]
For $\xi:\mathcal X_{n,j}\to\mathbb R$, define its root-mean-square
cell average at chart $a$ and level $p$ by
\begin{equation}\label{eq:multiscale_average}
T_{p,a,j}(\xi)
:=
\left(
\frac{1}{|\Upsilon_{p,a,j}^{m_j}|}
\sum_{\upsilon\in\Upsilon_{p,a,j}^{m_j}}
\left|
\frac{1}{|Q_{p,a,j,\upsilon}^{m_j}|}
\sum_{x\in Q_{p,a,j,\upsilon}^{m_j}}\xi(x)
\right|^2
\right)^{1/2}.
\end{equation}
See \cite{trillos2025} for the precise multiscale construction.

Applying Proposition~3.6 of \cite{trillos2025} componentwise and
weighting the resulting estimates, the global negative Sobolev seminorm satisfies
\begin{equation}\label{eq:between-multiscale-bound}
\|h\|_{H^{-1}(\mathcal X_n)}
\leq C\sum_{j=1}^Nw_j\left[\epsilon_+\|h_j\|_{L^2(\mathcal X_{n,j})}+\sum_{a=1}^{N_{\nu,j}}\sum_{p=1}^{m_j}
s_{p,j}T_{p,a,j}(h_j)\right]
\end{equation} with probability at least
$1-C\epsilon_+^{-d}\exp(-cn\epsilon_+^d)$. Our first result in this subsection controls the within-manifold consistency error in this weighted negative Sobolev seminorm.

\begin{proposition}[Consistency of the within-manifold operator]
\label{prop:consistency of the within-manifold operator}
There exists some constant \(C_\ell>0\), depending only on the index \(\ell\), the geometric and density class parameters, and the kernel profile \(\eta\), such that with probability at least $1-CNn\epsilon_+^{-d}\exp(-cn\epsilon_+^{d+4})$, \[\|(\mathcal{L}^{\mathrm{wth}}-\sigma_{\eta}\Delta_{\mathcal{M}})f_{\ell}\|_{H^{-1}(\mathcal{X}_n)}\leq C_\ell(B_{\ell}+1)\epsilon_+^2\log\epsilon_+^{-1}.\] Here $B_{\ell}$ denotes the regularity constant defined in \eqref{eq:B_definition}.
\end{proposition} 
\begin{proof} 
Set $h:=(\mathcal L^{\mathrm{wth}}-\sigma_{\eta}\Delta_{\mathcal M})f_{\ell}$. It follows that \(h_j=h|_{\mathcal X_{n,j}}=(\mathcal L_j^{\mathrm{wth}}-\sigma_{\eta}\Delta_{\mathcal M_j})f_{\ell,j}\) on \(\mathcal X_{n,j}.\) Applying Lemma 8(3) of \cite{li2025} componentwise with the rescaling $t\mapsto t/\sqrt{\epsilon_+}$ gives, with probability at least
\(1-C\epsilon_+^{-d_j}
\exp(-cn_j\epsilon_+^{d_j}t^2)\), \[\epsilon_+\|h_j\|_{L^2(\mathcal X_{n,j})}\leq C_{\ell,j} (\epsilon_+^2+t),\] for some constant \(C_{\ell,j}\) depending only on the index \(\ell\) and the geometric and density-class parameters. Choosing \(t\asymp\epsilon_+^2\), weighting the resulting estimates, and taking a union bound over the \(N\) fixed number of components yields, with probability at least
\(1-CN\epsilon_+^{-d}\exp(-cn\epsilon_+^{d+4})\),
\[\sum_{j=1}^Nw_j\epsilon_+\|h_j\|_{L^2(\mathcal X_{n,j})}
\leq C_\ell\epsilon_+^2.\]
For the multiscale cell-average, fix
$j,a,p$, $\upsilon\in\Upsilon_{p,a,j}^{m_j}$, and define the normalized cell indicator
\[\psi_{p,a,j,\upsilon}(x):=\frac{\mathbf 1_{\square_{p,a,j,\upsilon}^{m_j}}(x)}
{\mu_j(\square_{p,a,j,\upsilon}^{m_j})}.\]
The corresponding componentwise form of Proposition~3.10 of
\cite{trillos2025} gives
\[\Bigg|\frac{1}{n_j}\sum_{x\in\mathcal X_{n,j}}
\psi_{p,a,j,\upsilon}(x)h_j(x)\Bigg|\leq
C_{\ell,j}(B_\ell+1)\epsilon_+^2s_{p,j}^{-1}\]
for each fixed cell, except on an event of probability at most
\(Cn_j\exp(-cn_j\epsilon_+^{d_j+4})\). Here \(C_{\ell,j}\) depends on the geometric and density class parameters, and the kernel profile \(\eta\).
Moreover,
\[\frac{1}{|Q_{p,a,j,\upsilon}^{m_j}|}
\sum_{x\in Q_{p,a,j,\upsilon}^{m_j}}h_j(x)={}
\frac{n_j\mu_j(\square_{p,a,j,\upsilon}^{m_j})}
{|Q_{p,a,j,\upsilon}^{m_j}|}\times\frac{1}{n_j}
\sum_{x\in\mathcal X_{n,j}}\psi_{p,a,j,\upsilon}(x)h_j(x).\]
 On the occupancy event, \(|Q_{p,a,j,\upsilon}^{m_j}|\asymp n_j\mu_j (\square_{p,a,j,\upsilon}^{m_j}),\) and hence
\[\Bigg|\frac{1}{|Q_{p,a,j,\upsilon}^{m_j}|}
\sum_{x\in Q_{p,a,j,\upsilon}^{m_j}}h_j(x)\Bigg|
\leq C_{\ell,j}(B_\ell+1)\epsilon_+^2s_{p,j}^{-1}.\] It follows from \eqref{eq:multiscale_average} that
\[s_{p,j}T_{p,a,j}(h_j)\leq s_{p,j}\Bigg(\frac{1}{|\Upsilon_{p,a,j}^{m_j}|}\sum_{\upsilon\in\Upsilon_{p,a,j}^{m_j}}C_{\ell,j}(B_\ell+1)^2\epsilon_+^4s_{p,j}^{-2}
\Bigg)^{1/2}\leq C_{\ell,j}(B_\ell+1)\epsilon_+^2.\]
Since \(m_j\asymp\log(\epsilon_+^{-1})\) and the number of charts is uniformly bounded, summing over all components, charts, and scales gives, with probability at \(1-Cn\exp(-cn\epsilon_+^{d+4})\),
\[\sum_{j=1}^Nw_j\sum_{a=1}^{N_{\nu,j}}\sum_{p=1}^{m_j}
s_{p,j}T_{p,a,j}(h_j)\leq C_\ell(B_\ell+1)\epsilon_+^2
\log\epsilon_+^{-1}.\]
Combining the preceding two estimates with the multiscale Poincaré inequality proves the asserted seminorm bound on the simultaneous high-probability event.

\end{proof}

We next control the contribution from the between-manifold edges. Fix two intersecting components
$\mathcal M_j$ and $\mathcal M_k$, both of maximal intrinsic dimension $d_j=d_k=d$. 
Let \[d_{jk}:=\dim(\mathcal M_j\cap\mathcal M_k)\leq d-1, \qquad c_{jk}:=d-d_{jk}\geq1,\]
so that $c_{jk}$ is the codimension of the intersection in either component. The argument uses the same tubular-neighborhood counting twice: globally for the $L^2$ estimate and cellwise in the multiscale Poincaré inequality for the \(H^1\)-estimate.

\begin{proposition}[Between-manifold Laplacian bound]
\label{prop:between connections counts} There exists some constant \(C_\ell>0\), depending only on the geometric, the density-class, and the intersection parameters such that for all \(d\geq1\), with probability at least
$1-Cn\epsilon_+^{-d}\exp(-cn\epsilon_+^{d+4})$, 
\begin{equation}\label{eq:between-L2-norm}
\|\mathcal L^{\mathrm{btw}}f_\ell\|_{L^2(\mathcal X_n)}
\leq C_{\ell}B_\ell\epsilon_+
\end{equation}
and
\begin{equation}\label{eq:between-H-1-norm}
\|\mathcal L^{\mathrm{btw}}f_\ell\|_
{H^{-1}(\mathcal X_n)}
\leq C_{\ell}B_\ell\epsilon_+^2.
\end{equation}
Here $B_{\ell}$ denotes the regularity constant defined in \eqref{eq:B_definition}.
\end{proposition}
\begin{proof} 
Since the number of components is fixed, it suffices to work with one ordered pair $j\neq k$ and then sum the resulting estimates. 
For tube widths $\alpha,\beta\leq\epsilon_+$, define
\[\mathsf N_{jk}(\alpha):=\sum_{x\in\mathcal X_{n,j}}\mathbf 1_{\{\operatorname{dist}(x,\mathcal M_k)\leq\alpha\}}, \qquad \mathsf D_{jk}(x;\beta):=\sum_{y\in\mathcal X_{n,k}}
\mathbf 1_{\{K_\epsilon^E(x,y)\neq0,\,\operatorname{dist}(y,\mathcal M_j)\leq\beta\}}.\]
Thus, $\mathsf N_{jk}(\alpha)$ counts candidate
$j$-endpoints in the $\alpha$-tube of $\mathcal M_{jk}:=\mathcal M_j\cap\mathcal M_k$, whereas $\mathsf D_{jk}(x;\beta)$ counts the admissible $k$-neighbors lying within the ambient $\epsilon_+$-neighborhood of a fixed $x$ and the $\beta$-tube of $\mathcal M_{jk}$. Tubular coordinates around
$\mathcal M_{jk}$ give
\[\operatorname{vol}_{\mathcal M_j}\{x\in\mathcal X_{n,j}:\operatorname{dist}(x,\mathcal M_k)\leq\alpha\}\leq C_{jk}\alpha^{c_{jk}},\] and \[\operatorname{vol}_{\mathcal M_k}
\{y\in\mathcal X_{n,k}: \|x-y\|\leq\epsilon_+, \;\operatorname{dist}(y,\mathcal M_j)\leq\beta\}
\leq C_{jk}\epsilon_+^{d-c_{jk}}\beta^{c_{jk}}.\]
Furthermore, for a level-$p$ cell of side length $s:=s_{p,j}$, the corresponding local estimate is
\[\operatorname{vol}_{\mathcal M_j}\bigl(
\square_{p,a,j,\upsilon}^{m_j}\cap\{x\in\mathcal X_{n,j}:\operatorname{dist}(x,\mathcal M_k)\leq\alpha\}\bigr)\leq C_{jk}s^{d-c_{jk}}\alpha^{c_{jk}}.\]
Since $\alpha,\beta\leq\epsilon_+$,
\[\alpha^{c_{jk}}\leq\epsilon_+^{c_{jk}-1}\alpha,
\qquad
\epsilon_+^{d-c_{jk}}\beta^{c_{jk}}\leq\epsilon_+^{d-1}\beta.\]
Bernstein's inequality and a union bound over the sample points and the
$O(\epsilon_+^{-d})$ multiscale cells therefore give an event
$\mathcal E_{\mathrm{btw}}$ satisfying
\[\mathbb P(\mathcal E_{\mathrm{btw}})
\geq1-Cn\epsilon_+^{-d}\exp(-cn\epsilon_+^{d+3})
\geq1-Cn\epsilon_+^{-d}\exp(-cn\epsilon_+^{d+4}),\]
on which, simultaneously for the relevant choices
$\alpha,\beta\in\{q,r,\epsilon_+\}$,
\begin{equation}\label{eq:uniform-between-counts}
\mathsf N_{jk}(\alpha)\leq C_{jk}n\epsilon_+^{c_{jk}-1}\alpha,
\qquad\mathsf D_{jk}(x;\beta)\leq C_{jk}n\epsilon_+^{d-1}\beta.
\end{equation}
We work on this event throughout the remainder of the proof.

We first prove \eqref{eq:between-L2-norm}. By the kernel support and global $C_\ell B_\ell$-Lipschitz
continuity of \(f_\ell\),
\begin{equation}\label{eq:compatibility}
|f_{\ell,j}(x)-f_{\ell,k}(y)|\leq C_\ell B_\ell\epsilon_+
\qquad\text{whenever }K_\epsilon^E(x,y)\neq0.
\end{equation}
Using $w_j/n_j=1/n$, the boundedness of the kernel profile, and
Cauchy-Schwarz inequality,
\begin{equation}\label{eq:between-L2-count-reduction}
\|\mathcal L^{\mathrm{btw}}f_\ell\|_{L^2(\mathcal X_n)}^2
\leq C_\ell\frac{\vartheta_\varepsilon^2B_\ell^2\epsilon_+^2}{n^3}
\sum_{j=1}^N\sum_{k\neq j} \sum_{x\in\mathcal X_{n,j}}
\mathsf D_{jk}(x;\epsilon_+)^2.\end{equation}
For upper tube widths $(\alpha,\beta)$, define the pairwise
squared-neighbor count by
\[\mathfrak C_{jk}^{(2)}(\alpha,\beta):=\sum_{x\in\mathcal X_{n,j},\,\operatorname{dist}(x,\mathcal M_k)\leq\alpha}
\mathsf D_{jk}(x;\beta)^2\leq\mathsf N_{jk}(\alpha)
\sup_x\mathsf D_{jk}(x;\beta)^2.\]
Lemma~\ref{lem:four_regions} restricts every retained $j$--$k$ edge to one of the three configurations summarized in Table~\ref{tab:global-between-squared-count}. 
\begin{table}[htbp]
\centering
\small
\begin{tabular}{c|cc|cc|c}
\toprule
Case
& $\operatorname{dist}(x,\mathcal M_k)$
& $\operatorname{dist}(y,\mathcal M_j)$
& $\mathsf N_{jk}(\alpha)$
& $\sup_x\mathsf D_{jk}(x;\beta)$
& $\mathfrak C_{jk}^{(2)}(\alpha,\beta)$ \\
\midrule
(1)
& $(r,\epsilon_+]$
& $(0,q]$
& $C_{jk}n\epsilon_+^{c_{jk}}$
& $C_{jk}n\epsilon_+^{d+3}$
& $C_{jk}n^3\epsilon_+^{2d+c+6}\leq C_{jk}n^3\epsilon_+^{2d+7}$ \\
(2)
& $(q,r]$
& $(0,r]$
& $C_{jk}n\epsilon_+^{c_{jk}+2}$
& $C_{jk}n\epsilon_+^{d+2}$
& $C_{jk}n^3\epsilon_+^{2d+c+6}\leq C_{jk}n^3\epsilon_+^{2d+7}$ \\
(3)
& $(0,q]$
& $(0,\epsilon_+]$
& $C_{jk}n\epsilon_+^{c_{jk}+3}$
& $C_{jk}n\epsilon_+^d$
& $C_{jk}n^3\epsilon_+^{2d+c_{jk}+3}\leq C_{jk}n^3\epsilon_+^{2d+4}$ \\
\bottomrule
\end{tabular}
\caption{Global squared-neighbor counts for an ordered pair $(j,k)$. The last column counts $j$-endpoints weighted by the square of their admissible $k$-neighbor counts.}
\label{tab:global-between-squared-count}
\end{table}
Under the choice \(r=\epsilon_+^3\), Case~(3) dominates and so 
\eqref{eq:between-L2-count-reduction} yields \eqref{eq:between-L2-norm}.

We next prove \eqref{eq:between-H-1-norm}. Set
$h:=\mathcal L^{\mathrm{btw}}f_\ell$ and define the ordered-pair contribution by
\[h_{jk}(x):=\frac{\vartheta_\varepsilon}{n}\sum_{k\neq j,\,y\in\mathcal X_{n,k}}K_\epsilon^E(x,y)\bigl(f_{\ell,j}(x)-f_{\ell,k}(y)\bigr),\qquad h_j=h|_{\mathcal X_{n,j}}=\sum_{k\neq j}h_{jk}.\] Recall the componentwise multiscale Poincaré estimate \eqref{eq:between-multiscale-bound}.
By Cauchy--Schwarz and \eqref{eq:between-L2-norm},
\begin{equation}\label{eq:multiscale_first_term}
\sum_{j=1}^Nw_j\epsilon_+\|h_j\|_{L^2(\mathcal X_{n,j})}
\leq\epsilon_+\|h\|_{L^2(\mathcal X_n)}\leq C_\ell B_\ell\epsilon_+^2.
\end{equation}
It remains to control the multiscale term. Fix $a,p,\upsilon$, and abbreviate the cell side length by $s:=s_{p,j}$ and 
$Q:=Q_{p,a,j,\upsilon}^{m_j}$. Define the local endpoint count by
\[\mathsf N_{jk,Q}(\alpha):=\sum_{x\in Q}\mathbf 1_{\{\operatorname{dist}(x,\mathcal M_k)\leq\alpha\}}.\] The local tubular estimate and the same concentration event give
\[\mathsf N_{jk,Q}(\alpha)\leq C_{jk}ns^{d-c_{jk}}\epsilon_+^{c_{jk}-1}\alpha.
\]
In parallel with the global squared count, define the cellwise ordinary connection count by
\[\mathfrak C_{jk,Q}^{(1)}(\alpha,\beta):=\sum_{\substack{x\in Q,\,\operatorname{dist}(x,\mathcal M_k)\leq\alpha}}\mathsf D_{jk}(x;\beta).\]
The resulting estimates are summarized in
Table~\ref{tab:cell-between-edge-count}.
\begin{table}[htbp]
\centering
\small
\begin{tabular}{c|cc|cc|c}
\toprule
Case
& $\operatorname{dist}(x,\mathcal M_k)$
& $\operatorname{dist}(y,\mathcal M_j)$
& $\mathsf N_{jk,Q}(\alpha)$
& $\sup_x\mathsf D_{jk}(x;\beta)$
& $\mathfrak C_{jk,Q}^{(1)}(\alpha,\beta)$ \\
\midrule
(1)
& $(r,\epsilon_+]$
& $(0,q]$
& $C_{jk}ns^{d-c_{jk}}\epsilon_+^{c_{jk}}$
& $C_{jk}n\epsilon_+^{d+3}$
& $C_{jk}n^2s^{d-{c_{jk}}}\epsilon_+^{d+{c_{jk}}+3}$ \\
(2)
& $(q,r]$
& $(0,r]$
& $C_{jk}ns^{d-{c_{jk}}}\epsilon_+^{{c_{jk}}+2}$
& $C_{jk}n\epsilon_+^{d+2}$
& $C_{jk}n^2s^{d-{c_{jk}}}\epsilon_+^{d+{c_{jk}}+4}$ \\
(3)
& $(0,q]$
& $(0,\epsilon_+]$
& $C_{jk}ns^{d-{c_{jk}}}\epsilon_+^{{c_{jk}}+3}$
& $C_{jk}n\epsilon_+^d$
& $C_{jk}n^2s^{d-{c_{jk}}}\epsilon_+^{d+{c_{jk}}+3}$ \\
\bottomrule
\end{tabular}
\caption{Cellwise ordinary connection counts for an ordered pair
$(j,k)$. The last column counts ordered retained edges $(x,y)$ with $x\in Q$.}
\label{tab:cell-between-edge-count}
\end{table}
Cases~(1) and~(3) dominate, so the number of retained $j$--$k$ edges with their $j$-endpoint in $Q$ is at most \(C_{jk}n^2s^{d-{c_{jk}}}\epsilon_+^{d+{c_{jk}}+3}\). Since $|Q|\asymp n_js^d\asymp ns^d$, the Lipschitz continuity \eqref{eq:compatibility} and
$\vartheta_\varepsilon\asymp\epsilon_+^{-d-2}$ imply
\begin{equation}\label{eq:between-cell-average}
\begin{aligned}\left|\frac{1}{|Q|}\sum_{x\in Q}h_{jk}(x)
\right|&\leq\frac{\vartheta_\varepsilon}{n|Q|}\sum_{x\in Q}\sum_{y\in\mathcal X_{n,k}}\left|K_\epsilon^E(x,y)\right|
\left|f_{\ell,j}(x)-f_{\ell,k}(y)\right|\leq
\frac{C_\ell B_\ell\vartheta_\varepsilon\epsilon_+}{n|Q|}
\mathfrak C_{jk,Q}^{(1)}\\&\leq\frac{C_\ell B_\ell\epsilon_+^{-d-2} \epsilon_+}{n(ns^d)}\left(C_{jk}n^2s^{d-c_{jk}}\epsilon_+^{d+c_{jk}+3}\right)=
C_\ell B_\ell\epsilon_+^{c_{jk}+2}s^{-c_{jk}}.\end{aligned}\end{equation}
Define the active-cell index set for $(j,k)$ at level $(p,a)$ by
\[\mathcal A_{p,a,jk}:=\left\{\upsilon\in\Upsilon_{p,a,j}^{m_j}:
Q_{p,a,j,\upsilon}^{m_j}\text{ contains the $\mathcal M_j$-endpoint of a retained $j$--$k$ edge}\right\}.\]
Only cells intersecting the $O(\epsilon_+)$-tube around
$\mathcal M_{jk}$ are active. Since $s\gtrsim\epsilon_+$,
\[|\mathcal A_{p,a,jk}|\leq C_{jk}s^{-(d-c_{jk})},\qquad
\frac{|\mathcal A_{p,a,jk}|}{|\Upsilon_{p,a,j}^{m_j}|}
\leq C_{jk}s^{c_{jk}}.\]
Combining this active-cell proportion with
\eqref{eq:between-cell-average} and the triangle inequality gives
\[T_{p,a,j}(h_j)\leq\sum_{k\neq j}T_{p,a,j}(h_{jk})\leq C_\ell B_\ell\epsilon_+^{c_{jk}+2}s^{-c_{jk}/2}.\]
Since $s=3^{p-m_j}$ and $3^{-m_j}\asymp\epsilon_+$, a geometric series gives the codimension-$c_{jk}$ scale factor
\[\sum_{p=1}^{m_j}s^{1-c_{jk}/2}\leq C\begin{cases}1, & c_{jk}=1,\\
\log(\epsilon_+^{-1}), & c_{jk}=2,\\\epsilon_+^{1-c_{jk}/2}, & c_{jk}>2.
\end{cases}\]
It follows that the multiscale contribution of each ordered pair is bounded by
\[\sum_{p=1}^{m_j}
sT_{p,a,j}(h_j)=C_\ell B_\ell\epsilon_+^{c_{jk}+2}\sum_{p=1}^{m_j}s^{1-c_{jk}/2}\leq
C_\ell B_\ell\begin{cases}\epsilon_+^3, & c_{jk}=1,\\
\epsilon_+^4\log(\epsilon_+^{-1}), & c_{jk}=2,\\
\epsilon_+^{c_{jk}/2+3}, & c_{jk}>2.
\end{cases}\]
Because the numbers of charts and components are fixed,
summing over $a,j,k$ yields
\[\sum_{j=1}^Nw_j\sum_{a=1}^{N_{\nu,j}}\sum_{p=1}^{m_j}
sT_{p,a,j}(h_j)\leq C_\ell B_\ell\epsilon_+^3,\]
for all \(d\geq1\). Combining this estimate with \eqref{eq:multiscale_first_term}
proves \eqref{eq:between-H-1-norm}.
\end{proof}

\begin{proposition}[Between-manifold Dirichlet energy bound]
\label{prop:between-manifold-energy-bound} Choose \(\epsilon_+\) such that $n\epsilon_+^{5d-2}\to\infty$. Then there exists some constant \(C_\ell>0\), depending only on the geometric, the density-class, and the intersection parameters such that for all \(d\geq1\), 
\[\big|\langle f_l,\mathcal{L}^\mathrm{btw}f_l\rangle_{L^2(\mathcal{X}_n)}\big|\leq CB_{\ell}^2 \epsilon_+^4,\] 
with probability at least \(1-Cn\exp\bigl(-cn\epsilon_+^{d+4}\bigr).\) $B_{\ell}$ denotes the regularity constant from \eqref{eq:B_definition}.
 
\end{proposition}
\begin{proof}
By the definition of the between-manifold Dirichlet energy in \eqref{eq:between dirichlet energy} \[\langle f_l,\mathcal{L}^\mathrm{btw}f_l\rangle_{L^2(\mathcal{X}_n)}=\frac{\vartheta_\varepsilon}{2n^2}\sum_{j=1}^N\sum_{x\in\mathcal X_{n,j}}\sum_{k\neq j,\, y\in\mathcal X_{n,k}}K_{\epsilon}^E(x,y)(f_{\ell}(x)-f_{\ell}(y))^2. \]
By the $C_\ell B_\ell$-Lipschitz continuity of $f_\ell$, \(\bigl(f_\ell(x)-f_\ell(y)\bigr)^2\leq C_\ell B_\ell^2\epsilon_+^2\) whenever \(K_\epsilon^E(x,y)\neq0\). 
It follows that
\begin{equation}\label{eq:between-energy-edge-count}
\big|\langle f_l, \mathcal{L}^\mathrm{btw}f_l\rangle_{L^2(\mathcal{X}_n)}\big|\leq C\frac{\vartheta_\varepsilon B_{\ell}^2\epsilon_+^2}{n^2}\sum_{j=1}^N\sum_{x\in\mathcal X_{n,j}}\sum_{k\neq j,\, y\in\mathcal X_{n,k}}1_{\{\|x-y\|\leq\epsilon_+\}},
\end{equation}
By Theorem~\eqref{thm:sparse_between_manifolds}, with probability at least
\(1-Cn\exp(-cn\epsilon_+^{d+4}),\)
the number of retained edges between $\mathcal M_j$ and
$\mathcal M_k$ is at most 
\(Cn^2\epsilon_+^{d+4}\). 
Substituting into
\eqref{eq:between-energy-edge-count}  yields
\[\big|\langle f_l, \mathcal{L}^\mathrm{btw}f_l\rangle_{L^2(\mathcal{X}_n)}\big|\leq
CB_\ell^2\epsilon_+^{-d-2}\epsilon_+^2\epsilon_+^{d+4}=
CB_\ell^2\epsilon_+^4.\]
Finally, the scale condition required by Theorem~\eqref{thm:sparse_between_manifolds} becomes
\(r^{2d-1}/\epsilon_+^{d-1}\asymp\epsilon_+^{5d-2}\).
Thus, it is enough to require $n\epsilon_+^{5d-2}\to\infty$, completing the proof.
\end{proof}

\subsection{Geodesic Symmetrization for the Density Remainder}\label{sec:symmetrization argument}

Fix \(t\in[0,1]\). We define the geodesic-flow symmetrization map $\Phi_t: T\mathcal{M}_j \to T\mathcal{M}_j$ by \[\Phi_t(x,v) = (\bar x,\bar v) = \bigl(\exp_x(2t\epsilon_+ v),\; -\,d(\exp_x)_{2t\epsilon_+ v}(v)\bigr).\]
Geometrically, $\Phi_t$ pairs each tangent vector \((x,v)\) with its time-reversed counterpart along the same geodesic. Since $\Phi_t$ is a reparametrization of the geodesic flow, it preserves the canonical volume measure on the tangent bundle. Moreover, $\exp_{\bar{x}}(t\epsilon_+\bar{v})=\exp_x(t\epsilon_+v)$ so the Hessian factor is invariant under the pairing \(\mathcal D_t(\bar x,\bar v)=\mathcal D_t(x,v)\): \[\mathcal D_t(x,v)=\nabla_g^2\rho_j (\exp_x(t\epsilon_+v))\left[d(\exp_x)_ {t\epsilon_+v}(v),d(\exp_x)_{t\epsilon_+v}(v)\right], \qquad |\mathcal D_t(x,v)|\le C_{\rho}^{(2)}\|v\|^2.\]  The measure-preserving property of \(\Phi_t\) symmetrizes the density remainder term as {\small\[\mathcal I_t=\frac{\vartheta_\varepsilon
\epsilon_+^{d+3}}{2}\sum_{j=1}^Nw_j^2\int_{T\mathcal M_j}
K_\epsilon^{\mathcal N}(v)\mathcal D_t(x,v)\Big[f_\ell(x) \rho_j(x)\langle\nabla_g f_\ell(x),v\rangle_g+
f_\ell(\bar x)\rho_j(\bar x)\langle\nabla_g f_\ell(\bar x),\bar v\rangle_g\Big]\,dm(x,v),\]} where \(dm\) denotes the canonical measure on the tangent bundle, restricted to \(\|v\|\le1\). The advantage of this representation is that the Hessian factor has become common to both terms, allowing the cancellation to occur entirely inside the bracket.

We now estimate the bracket. Adding and subtracting $f_\ell(\bar x)\rho_j(\bar x)\langle\nabla_g f_\ell(x), v\rangle_g$ yields \begin{align*}&|f_{\ell}(x)\rho_j(x)\langle\nabla_gf_{\ell}(x),v\rangle_g+f_{\ell}(\bar x)\rho_j(\bar x)\langle\nabla_gf_{\ell}(\bar x),\bar v\rangle_g|\\\le\;\;& |f_{\ell}(\bar x)\rho_j(\bar x)||\langle\nabla_gf_{\ell}(x),v\rangle_g+\langle\nabla_gf_{\ell}(\bar x),\bar v\rangle_g|+|\langle\nabla_gf_{\ell}(x),v\rangle_g||f_{\ell}(x)\rho_j(x)-f_{\ell}(\bar x)\rho_j(\bar x)|. \end{align*}
The first term corresponds to the cancellation of the directional derivatives. Since $(x,v)$ and $(\bar x, \bar v)$ lie on the same geodesic and $\bar v$ is the transported negative of $v$, we invoke the cancellation estimate (to be established below) \[|\langle\nabla_gf_{\ell}(x),v\rangle_g+\langle\nabla_gf_{\ell}(\bar x),\bar v\rangle_g|\le C\epsilon_+\|v\|^2\|\nabla_g^2f_{\ell}\|_{L^\infty(B_g(x,2\epsilon_+))}.\] The second term measures the variation of the product \(f_\ell\rho_j\) between \(x\) and \(\bar x\). Since $d_g(x,\bar x)\leq C\epsilon_+\|v\|$, the mean value theorem implies \[|f_{\ell}(x)\rho_j(x)-f_\ell(\bar x)\rho_j(\bar x)|
\le C\epsilon_+\|v\|\left(\rho_{\max}\|\nabla_g f_\ell\|+
C_\rho^{(1)}\|f_\ell\|\right),\] where the derivatives are evaluated along the geodesic segment from $x$ to $\bar x$. 

Combining these estimates and applying Cauchy--Schwarz, we obtain \begin{align*}
|\mathcal I_t|&\le C\sum_{j=1}^Nw_j^2 \vartheta_\varepsilon\epsilon_+^{d+4}\int_{\mathcal M_j}\left(|f_{\ell}(x)|\|\nabla_g^2f_{\ell}\|_{L^\infty{(B_g(x,2\epsilon_+))}}+|\nabla_gf_{\ell}(x)|^2+|f_{\ell}(x)||\nabla_gf_{\ell}(x)|\right)d\mu_j\\&\le C\epsilon_+^{2}\sum_{j=1}^Nw_j^2 \left(\|f_{\ell}\|_{L^2(\mathcal M_j)}\|\nabla_g^2f_{\ell}\|_{L^\infty{(\mathcal M_j)}}+\|\nabla_gf_{\ell}\|_{L^2(\mathcal M_j)}^2+\|f_{\ell}\|_{L^2(\mathcal M_j)}\|\nabla_gf_{\ell}\|_{L^2(\mathcal M_j)}\right).\end{align*}
The eigenfunction's regularity allows the bracket to be absorbed into the constant. Since the number of manifold components and the mixture weights are fixed, \(|\mathcal I_t|\leq C_\ell B_{\ell}^2\epsilon_+^{2}\).

It remains to quantify the cancellation in the symmetrized integrand, we compare the directional derivatives at the paired configuration $(x,v)$ and $(\bar{x},\bar{v})$, where $\bar{x} \coloneqq \exp_x(\epsilon_+ v)$ and $\bar{v}:= \log_{\bar{x}}(x)/\epsilon_+$. Let $\gamma:[0,1]\rightarrow\mathcal{M}_j$ be the constant-speed minimizing geodesic from $x$ to $\bar{x}$, so that $\dot\gamma(0)=\log_x(\bar{x})=\epsilon_+ v$. Since the tangent vector at its endpoint points opposite to the log map back to its start, $\dot\gamma(1)=-\log_{\bar{x}}(x)=-\epsilon_+ \bar{v}$. Define $\tilde f(s):=f_{\ell}(\gamma(s))$. Then 
\[
\tilde f'(s)=\langle\nabla_g f_{\ell}(\gamma(s)),\dot\gamma(s)\rangle_g, \qquad \tilde f''(s)=\langle \nabla_g^2 f_{\ell}(\gamma(s))[\dot \gamma(s)],\dot \gamma(s)\rangle_g.
\] 
Applying the fundamental theorem of calculus $\tilde f'(1) - \tilde f'(0) = \int_0^1\tilde f''(s)\mathrm{d}s$ yields
\[
-\epsilon_+\left( \langle\nabla_g f_{\ell}(\bar{x}),\bar{v}\rangle_g + \langle\nabla_g f_{\ell}(x),v\rangle_g \right) = \int_0^1\langle\nabla_g^2 f_{\ell}(\gamma(s))[\dot\gamma(s)],\dot\gamma(s) \rangle_g\mathrm{d}s.
\]
Since $\|\dot\gamma(s)\|_g = d_g(x,\bar{x}) \leq C\epsilon_+\|v\|$ for all $s\in[0,1]$, taking the absolute value and bounding the integral gives
\[
\epsilon_+\left| \langle\nabla_g f_{\ell}(\bar{x}),\bar{v}\rangle_g + \langle\nabla_g f_{\ell}(x),v\rangle_g \right| \leq C\epsilon_+^2\|v\|^2\|\nabla_g^2 f_{\ell}\|_{L^{\infty}(B_g(x,\epsilon_+))}.\]
Dividing both sides by $\epsilon_+$ yields the desired estimate.

\end{document}